%% file: master.tex
\pdfoutput=1
\documentclass[11pt]{article}

\usepackage{amsmath,amssymb,amsthm}
\usepackage{graphicx}
\usepackage{tikz}
\usetikzlibrary{positioning,arrows.meta}
\usepackage{forest}
\usepackage[authoryear,round]{natbib}

\usepackage{newtxtext}
\usepackage[subscriptcorrection]{newtxmath}

\usepackage[margin=1in]{geometry}

\usepackage[plain,noend]{algorithm2e}

\usepackage{amsthm}

\usepackage{amsthm}

\theoremstyle{plain}
\newtheorem{theorem}{Theorem}
\newtheorem{proposition}{Proposition}
\newtheorem{lemma}{Lemma}
\newtheorem{corollary}{Corollary}

\theoremstyle{definition}
\newtheorem{definition}{Definition}
\newtheorem{assumption}{Assumption}

\theoremstyle{remark}
\newtheorem{remark}{Remark}

\makeatletter
\renewcommand{\algocf@captiontext}[2]{#1\algocf@typo. \AlCapFnt{}#2}

\def\@algocf@capt@plain{top}
\renewcommand{\algocf@makecaption}[2]{%
	\addtolength{\hsize}{\algomargin}%
	\sbox\@tempboxa{\algocf@captiontext{#1}{#2}}%
	\ifdim\wd\@tempboxa >\hsize%
	\hskip .5\algomargin%
	\parbox[t]{\hsize}{\algocf@captiontext{#1}{#2}}%
	\else%
	\global\@minipagefalse%
	\hbox to\hsize{\box\@tempboxa}%
	\fi%
	\addtolength{\hsize}{-\algomargin}%
}
\makeatother

\graphicspath{{./art/}}

\DeclareMathOperator{\Var}{Var}
\DeclareMathOperator{\Cov}{Cov}

\usepackage[hidelinks]{hyperref}

\title{Random Forests as Statistical Procedures: Design, Variance, and Dependence}
\author{Nathaniel S. O'Connell\\
	Department of Biostatistics and Data Science, Wake Forest University School of Medicine\\
	Winston-Salem, NC, United States\\
	\texttt{nathaniel.oconnell@wfusm.edu}}
\date{} 

\begin{document}
	
	\maketitle
	
\begin{abstract}
	We develop a finite-sample, design-based theory for random forests in which each tree is a randomized conditional predictor acting on fixed covariates and the forest is their Monte Carlo average. An exact variance identity separates Monte Carlo error from a covariance floor that persists under infinite aggregation. The floor arises through two mechanisms: observation reuse, where the same training outcomes receive weight across multiple trees, and partition alignment, where independently generated trees discover similar conditional prediction rules. We prove the floor is strictly positive under minimal conditions and show that alignment persists even when sample splitting eliminates observation overlap entirely. We introduce procedure-aligned synthetic resampling (PASR) to estimate the covariance floor, decomposing the total prediction uncertainty of a deployed forest into interpretable components. For continuous outcomes, resulting prediction intervals achieve nominal coverage with a theoretically guaranteed conservative bias direction. For classification forests, the PASR estimator is asymptotically unbiased, providing the first pointwise confidence intervals for predicted conditional probabilities from a deployed forest. Nominal coverage is maintained across a range of design configurations for both outcome types, including high-dimensional settings. The underlying theory extends to any tree-based ensemble with an exchangeable tree-generating mechanism.
\end{abstract}
	
	
	\input{main_v2}

	\bibliographystyle{plainnat}
	\bibliography{rf_design}
	
	\clearpage
	\appendix
	\input{supplementary_material}

\end{document}

%% file: main_v2.tex
\section{Introduction}
\label{sec:intro}
\subsection{Existing theory and its limitations}

Since their introduction by \citet{breiman2001}, Random Forests have been among the most widely used modeling frameworks for prediction, adopted across a wide range of scientific domains, and widely regarded for providing strong predictive performance, stability, and robustness to tuning \citep{hastie2009}. Despite their ubiquity in applied statistical settings, Random Forests are typically presented as algorithms---ensembles of randomized decision trees motivated empirically or through large-sample arguments---rather than as finite-sample statistical procedures.

This algorithmic framing has shaped much of the theoretical literature. Early work studied consistency properties \citep{breiman2004}, and subsequent results established consistency and rates of convergence under structural assumptions on splitting rules and feature distributions \citep{biau2008,biau2012,scornet2015,biau2016}. Complementary work showed that random forest predictors can be represented as adaptive nearest-neighbor or weighted regression estimators, with predictions expressed as weighted averages of observed responses and weights induced by recursive partitioning \citep{lin2006,biau2010,meinshausen2006,athey2019}. More recent work has treated random forests as statistical estimators and developed asymptotic distributional theory using U-statistic, V-statistic, and influence-function representations \citep{mentch2016,zhou2021v,wager2018}. Related work on variance estimation \citep{xu2024variance} conditions on the forest construction mechanism and yields asymptotically valid estimators under repeated sampling.

These approaches characterize sampling variability---how the forest predictor would change under repeated draws of the training data from the population---but they do not analyze the procedural variability induced by the randomization used to construct the forest at a fixed observed dataset. More generally, the variability of a random forest predictor reflects two distinct sources of randomness: sampling variability arising from repeated draws of the dataset, and procedural variability induced by the forest's own randomized construction. Existing theory primarily addresses the former.

It has also long been recognized that correlation between ensemble members limits the extent to which aggregation can reduce predictive variability \citep{breiman1996,dietterich2000,buhlmann2002,kuncheva2003,brown2005}. More recent work emphasizes that additional sources of randomization can improve predictive performance by weakening dependence between trees, even when individual trees become more variable \citep{mentch2020,mentch2022,liu2025}, and other analyses study algorithmic variance and convergence of randomized ensembles conditional on the training data \citep{lopes2020,lopes2019b}. These contributions document the presence and consequences of dependence, but do not yield an explicit finite-sample decomposition of predictive variability or isolate the structural sources of dependence induced by distinct design choices.

As a practical consequence, no existing method provides pointwise variance estimates for a deployed random forest conditional on the data used to construct it. Existing asymptotic approaches---including the infinitesimal jackknife \citep{wager2014}, U-statistic \citep{mentch2016}, and V-statistic \citep{zhou2021v} frameworks---are developed for regression-type targets and quantify sampling variability of the infinite-aggregation target across repeated datasets, answering how the forest's target would change if new training data were collected. Conformal prediction methods \citep{lei2018,romano2019} provide distribution-free coverage guarantees but do not decompose the sources of predictive variability or characterize the forest's own design-induced uncertainty. For a practitioner with a fitted forest in hand, none of these approaches quantify the total uncertainty of the delivered prediction at a given covariate value. For classification tasks that yield predicted probabilities, the asymptotic regression machinery can be applied with binary outcomes encoded as $Y \in \{0,1\}$; however, existing theory is formulated for regression-type targets and does not provide pointwise uncertainty quantification for the estimated conditional probability $\hat{p}(x)$ from a deployed forest.

\subsection{A design-based perspective}

Recent work has synthesized Random Forests as adaptive smoothing procedures \citep{curth2024}. We adopt a complementary perspective: we treat a random forest as a finite-sample statistical procedure generated by an explicit randomized design acting on a realized covariate configuration. Specifically, we fix the observed covariate points $X$ and analyze predictive variability under the joint randomness of the outcome realization $(Y \mid X)$ and the tree-generating mechanism $\theta$.

This conditioning choice is deliberate. Fixing $X$ but not $Y$ preserves both sources of variability that affect the forest predictor at the moment it is constructed: the intrinsic dispersion of outcomes at realized covariates $X$, and algorithmic randomness introduced by the forest procedure. A practitioner observes a fixed set of covariate profiles determined by the study design; what remains uncertain is both what outcomes those profiles produce and how the tree-building algorithm partitions them. Conditioning on $X$ captures this joint uncertainty.

Within this framework, each tree arises from randomized operations---observation selection \citep{breiman1996}, random subspace selection for candidate splits \citep{ho1998}, and random split choice---that determine how outcomes are locally averaged to form predictions. The tree-generating mechanism induces a probability distribution over terminal regions and local averaging rules conditional on $X$, and all algorithmic variability is defined with respect to this design. Outcome randomness enters through the conventional fixed-design model $Y \mid X$. All results are finite-sample and hold at fixed $n$.

When the dataset itself is viewed as random under repeated sampling, the unconditional variability of the forest predictor decomposes into a sampling component and a design component. The sampling component reflects instability of the induced infinite-forest target across datasets and, under standard stability conditions, decreases at the $O(n^{-1})$ rate. The design component captures variability introduced by the forest's randomized construction acting on the realized covariate configuration and depends explicitly on aggregation, subsampling, and partitioning choices. Section~\ref{sec:rf_procedures} formalizes this decomposition and motivates our focus on the design-induced variability that distinguishes random forests as statistical procedures.

We study the random forest predictor itself, not resampling-based performance estimates such as out-of-bag evaluation. Out-of-bag evaluation is a form of internal cross-validation targeting predictive risk; our analysis concerns the behavior of the fitted predictor at a fixed prediction point. This parallels classical regression analysis, in which one studies properties of a fitted model independently of how cross-validation is used to assess its predictive performance.

We show that each tree induces a data-adaptive conditional regression, which can be written as as a weighted average of observed responses, and that the forest predictor is an average of these randomized regression functions. This formulation yields an exact finite-sample variance identity that separates aggregation variability from a structural dependence component that persists under infinite aggregation, and reveals how common forest hyperparameters govern resolution, single-tree variability, and structural dependence.

\section{Random Forests as Statistical Procedures}
\label{sec:rf_procedures}

\subsection{Variance decomposition and scope}
\label{subsec:variance_decomp}

We distinguish two sources of variability in random forest prediction: procedural variability arising from the forest-generating randomization acting on a realized covariate configuration, and sampling variability arising from repeated draws of the dataset from an underlying population.

Let $\mathcal D_n^\star=\{(X_i,Y_i)\}_{i=1}^n$ denote a generic sample of size $n$ from the population, and let $\hat f_B(x;\mathcal D,\theta)$ denote the $B$-tree forest predictor constructed from dataset $\mathcal D$ using design randomization $\theta$. To isolate variability conditional on the realized covariate configuration, we apply the law of total variance with respect to $X$:
\[
\Var\!\bigl(\hat f_B(x;\mathcal D_n^\star,\theta)\bigr)
=
\mathbb E\!\left[
\Var\!\bigl(\hat f_B(x;\mathcal D_n^\star,\theta)\mid X\bigr)
\right]
+
\Var\!\left(
\mathbb E\!\bigl[\hat f_B(x;\mathcal D_n^\star,\theta)\mid X\bigr]
\right).
\]

The first term represents the finite-sample variability of a forest at a realized covariate configuration, where randomness arises jointly from outcome variation under the fixed-design model $Y \mid X$ and from the tree-generating mechanism $\theta$. The second term reflects sampling variability of the covariate configuration itself; it captures how the induced infinite-forest target changes across repeated draws of $X$ from the population. Under standard stability conditions, this sampling component is generically $O(n^{-1})$; see Appendix~\ref{app:sampling_rate} and related results in \citet{mentch2016,wager2018}.

Accordingly, for the remainder of this paper we fix the realized covariate configuration $X$ and analyze the finite-sample variability of $\hat f_B(x)$ under the joint randomness of $(Y \mid X)$ and the forest-generating mechanism $\theta$. All probabilistic quantities are therefore evaluated conditional on $X$.

\subsection{Tree-level randomized regression functions}
\label{subsec:tree_level}

We formalize a single tree as a randomized conditional predictor acting on a realized covariate configuration $X=\{X_i\}_{i=1}^n$, with $X_i\in\mathcal X\subset\mathbb R^p$. Outcomes satisfy a fixed-design model $Y\mid X$, and randomness in the tree arises jointly from the outcome realization and from the algorithmic construction. A particular realization of this randomization is indexed by $\theta$, and the resulting predictor is denoted by $T_\theta$. At a fixed realization $(Y,\theta)$ and covariate configuration $X$, the tree and its induced conditional predictor are deterministic functions.

For a prediction point $x\in\mathcal X$, the tree-generating mechanism $\theta$ induces a random averaging set
\[
A_\theta(x)
:=
\{\, i\in\{1,\ldots,n\} : X_i \text{ is routed with } x \text{ under tree } \theta \,\},
\]
consisting of the training observations grouped with $x$ in the terminal node reached by $x$. The set $A_\theta(x)$ is the primitive object governing the tree-level prediction at $x$ and depends on all sources of algorithmic randomness. Define the terminal node membership indicator
\[
M_{i,\theta}(x)
:=
\mathbf 1\{ i\in A_\theta(x) \},
\]
which records whether observation $i$ contributes to the prediction at $x$ in tree $\theta$. Under squared-error loss, the tree prediction at $x$ is the CART terminal-node estimate \citep{breiman1984},
\[
T_\theta(x)
=
\frac{1}{|A_\theta(x)|}\sum_{i\in A_\theta(x)} Y_i
=
\sum_{i=1}^n W_i(x;\theta)\,Y_i,
\qquad
W_i(x;\theta)
:=
\frac{M_{i,\theta}(x)}{|A_\theta(x)|}.
\]
The weights $W_i(x;\theta)$ are normalized membership indicators derived entirely from the averaging set $A_\theta(x)$. As $\theta$ varies under the tree-generating mechanism, the prediction $T_\theta(x)$ varies solely through the induced averaging set, equivalently through the collection of indicators $\{M_{i,\theta}(x)\}_{i=1}^n$. Although the covariate configuration $X$ is fixed, the averaging set $A_\theta(x)$ depends on the realized outcomes $Y$ because split selection optimizes impurity criteria computed from the training responses available to tree $\theta$.

Representations of random forests as adaptive weighted averages of the training responses are well established \citep{lin2006,biau2010,scornet2015} and underlie extensions such as quantile regression forests \citep{meinshausen2006}. Our contribution is not this algebraic form, but the probabilistic framing: we treat the random weight vector $W(x;\theta)$ as the primary random variable of interest. The infinite-aggregation predictor is the expectation of this random weight vector under the tree-generating mechanism conditional on $X$, and subsequent variance properties are determined by its distribution.

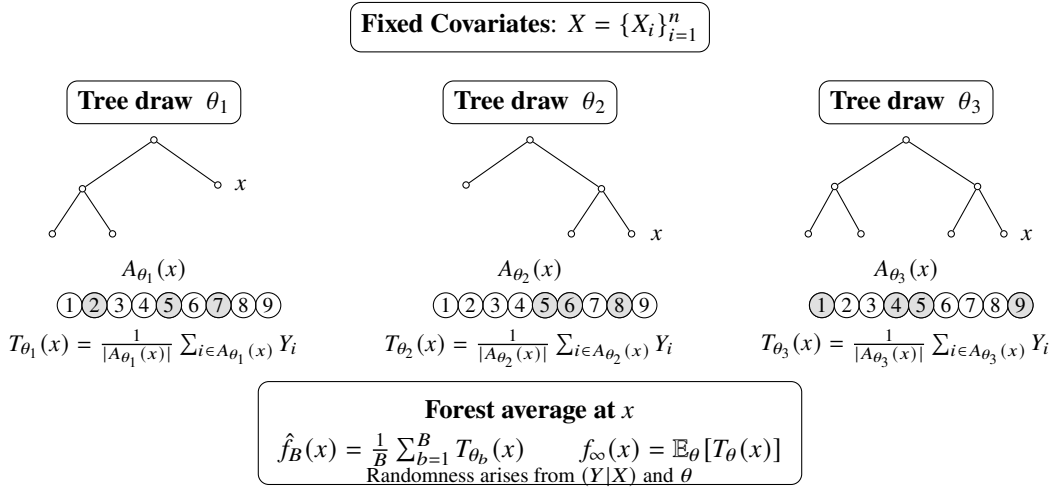
\begin{figure}[t]
	\centering
	\begin{tikzpicture}[font=\small, >=Latex]
		
		\tikzset{
			box/.style={draw, rounded corners, inner sep=4pt, align=center},
			arrow/.style={->, line width=0.6pt},
			treenode/.style={draw, circle, inner sep=0.85pt},
			idx/.style={circle, draw, inner sep=0.95pt},
			picked/.style={circle, draw, fill=black!12, inner sep=0.95pt},
		}
		
		\def\xL{-5.0}
		\def\xC{0.0}
		\def\xR{5.0}
		
		\def\yTop{3.05}
		\def\yLab{2.05}
		\def\yTree{1.55}
		\def\yStripLab{-0.20}
		\def\yStrip{-.65}
		\def\yFormula{-1.22}
		\def\yForest{-2.00}  
		\def\yNote{-2.56}    
		
		\node[box] (data) at (0,\yTop) {%
			\textbf{Fixed Covariates}:
			$X = \{X_i\}_{i=1}^n$%
		};
		
		\node[box] (lab1) at (\xL,\yLab) {\textbf{Tree draw } $\theta_1$};
		\node[box] (lab2) at (\xC,\yLab) {\textbf{Tree draw } $\theta_2$};
		\node[box] (lab3) at (\xR,\yLab) {\textbf{Tree draw } $\theta_3$};
		
		
		\node[treenode] (t1a) at (\xL,\yTree) {};
		\node[treenode] (t1b) at (\xL-0.95,\yTree-0.65) {};
		\node[treenode] (t1c) at (\xL+0.85,\yTree-0.60) {};
		\node[treenode] (t1d) at (\xL-1.35,\yTree-1.25) {};
		\node[treenode] (t1e) at (\xL-0.55,\yTree-1.25) {};
		\draw (t1a) -- (t1b);
		\draw (t1a) -- (t1c);
		\draw (t1b) -- (t1d);
		\draw (t1b) -- (t1e);
		\node[anchor=west, font=\footnotesize] at (\xL+0.95,\yTree-0.60) {$x$};
		
		\node[treenode] (t2a) at (\xC,\yTree) {};
		\node[treenode] (t2b) at (\xC-0.85,\yTree-0.60) {};
		\node[treenode] (t2c) at (\xC+0.95,\yTree-0.65) {};
		\node[treenode] (t2d) at (\xC+0.55,\yTree-1.25) {};
		\node[treenode] (t2e) at (\xC+1.35,\yTree-1.25) {};
		\draw (t2a) -- (t2b);
		\draw (t2a) -- (t2c);
		\draw (t2c) -- (t2d);
		\draw (t2c) -- (t2e);
		\node[anchor=west, font=\footnotesize] at (\xC+1.48,\yTree-1.26) {$x$};
		
		\node[treenode] (t3a) at (\xR,\yTree) {};
		\node[treenode] (t3b) at (\xR-0.95,\yTree-0.62) {};
		\node[treenode] (t3c) at (\xR+0.95,\yTree-0.62) {};
		\node[treenode] (t3d) at (\xR-1.30,\yTree-1.25) {};
		\node[treenode] (t3e) at (\xR-0.60,\yTree-1.25) {};
		\node[treenode] (t3f) at (\xR+0.60,\yTree-1.25) {};
		\node[treenode] (t3g) at (\xR+1.30,\yTree-1.25) {};
		\draw (t3a) -- (t3b);
		\draw (t3a) -- (t3c);
		\draw (t3b) -- (t3d);
		\draw (t3b) -- (t3e);
		\draw (t3c) -- (t3f);
		\draw (t3c) -- (t3g);
		\node[anchor=west, font=\footnotesize] at (\xR+1.40,\yTree-1.26) {$x$};
		
		\node[font=\footnotesize] at (\xL,\yStripLab) {$A_{\theta_1}(x)$};
		\node[font=\footnotesize] at (\xC,\yStripLab) {$A_{\theta_2}(x)$};
		\node[font=\footnotesize] at (\xR,\yStripLab) {$A_{\theta_3}(x)$};
		
		\def\dx{0.33}
		\def\xstart{-1.45}
		
		\foreach \k in {1,...,9} { \node[idx] (s1-\k) at (\xL+\xstart+\dx*\k,\yStrip) {\scriptsize \k}; }
		\foreach \k in {2,5,7} { \node[picked] at (s1-\k.center) {\scriptsize \k}; }
		
		\foreach \k in {1,...,9} { \node[idx] (s2-\k) at (\xC+\xstart+\dx*\k,\yStrip) {\scriptsize \k}; }
		\foreach \k in {5,6,8} { \node[picked] at (s2-\k.center) {\scriptsize \k}; }
		
		\foreach \k in {1,...,9} { \node[idx] (s3-\k) at (\xR+\xstart+\dx*\k,\yStrip) {\scriptsize \k}; }
		\foreach \k in {1,4,5,9} { \node[picked] at (s3-\k.center) {\scriptsize \k}; }
		
		\node[align=center, font=\footnotesize] at (\xL,\yFormula)
		{$T_{\theta_1}(x)=\frac{1}{|A_{\theta_1}(x)|}\sum_{i\in A_{\theta_1}(x)} Y_i$};
		
		\node[align=center, font=\footnotesize] at (\xC,\yFormula)
		{$T_{\theta_2}(x)=\frac{1}{|A_{\theta_2}(x)|}\sum_{i\in A_{\theta_2}(x)} Y_i$};
		
		\node[align=center, font=\footnotesize] at (\xR,\yFormula)
		{$T_{\theta_3}(x)=\frac{1}{|A_{\theta_3}(x)|}\sum_{i\in A_{\theta_3}(x)} Y_i$};
		
		\node[box, inner sep=7pt] (forest) at (0,\yForest-.35) {%
			\textbf{Forest average at $x$}\\[1.0 mm]
			$\hat f_B(x)=\frac{1}{B}\sum_{b=1}^B T_{\theta_b}(x)$\qquad
			$f_\infty(x)=\mathbb E_\theta[T_\theta(x)]$%
		};
		
		\node[align=center, font=\scriptsize] at (0,\yNote-.37)
		{Randomness arises from $(Y|X)$ and $\theta$};
		
	\end{tikzpicture}
	\caption{\textbf{Random forests as randomized local averaging on fixed outcomes.}
		Independent draws of the tree-generating design $\theta_1,\theta_2,\theta_3$ induce realized tree structures and tree specific terminal-node membership sets $A_{\theta_b}(x)$ for prediction point $x$. Each tree prediction is the average over indexed outcomes in its membership set, and averaging over independent draws yields $\hat f_B(x)$ and $f_\infty(x)$.}
	\label{fig:rf_design_schematic}
\end{figure}

\subsection{The forest predictor and its design-based target}
\label{subsec:forest_estimator}

A random forest predictor is a Monte Carlo average of randomized tree-level conditional predictors. Let $\{T_{\theta_b}\}_{b=1}^B$ denote $B$ independent realizations of the tree-generating mechanism described in Section~\ref{subsec:tree_level}. The finite forest predictor is
\[
\hat f_B(x)
=
\frac{1}{B}\sum_{b=1}^B T_{\theta_b}(x)
=
\sum_{i=1}^n \hat W_i^{(B)}(x)\,Y_i,
\qquad
\hat W_i^{(B)}(x)
:=
\frac{1}{B}\sum_{b=1}^B W_i(x;\theta_b).
\]

The aggregation level $B$ controls only the amount of Monte Carlo averaging and does not alter the distribution of the underlying tree-level predictor. Let $\theta$ denote a generic draw from the tree-generating mechanism and define the infinite-aggregation target as the $\theta$-average conditional on the realized covariate configuration,
\[
f_\infty(x)
:=
\mathbb E\!\left[T_\theta(x)\mid X\right].
\]
Because $X$ is fixed but the training outcomes remain random under the fixed-design model, $f_\infty(x)$ is random through $Y$. The finite forest predictor $\hat f_B(x)$ is a Monte Carlo approximation to $f_\infty(x)$ and satisfies $\mathbb E[\hat f_B(x)\mid X]=f_\infty(x)$. Finite forests fluctuate around this target due to limited aggregation, while structural dependence between tree-level conditional predictors induces a non-zero covariance component as $B\to\infty$. The variance analysis that follows makes this separation explicit.

\section{Variance of Random Forest Predictors}
\label{sec:rf_variance}

\subsection{Finite-sample variance identity}
\label{subsec:rf_variance_identity}

We now characterize the finite-sample variance of the random forest predictor $\hat f_B(x)$ conditional on the realized covariate configuration $X$. Define
\[
\sigma_T^2(x)
=
\Var\!\bigl(T_\theta(x)\mid X\bigr),
\qquad
C_T(x)
=
\Cov\!\bigl(T_\theta(x),T_{\theta'}(x)\mid X\bigr),
\]
where $\theta$ and $\theta'$ denote independent draws from the tree-generating mechanism. Although trees are generated independently, the resulting predictions are not independent, because each tree averages outcomes from the same realized dataset and the same observations can contribute to predictions across multiple trees. This dependence is captured by the covariance term $C_T(x)$.

\begin{theorem}[Finite-sample variance identity for random forests]
	\label{thm:finite_var}
	For any $B \ge 1$,
	\[
	\Var\!\left(\hat f_B(x)\mid X\right)
	=
	\frac{1}{B}\,\sigma_T^2(x)
	+
	\frac{B-1}{B}\,C_T(x).
	\]
\end{theorem}

\begin{proof}
	Write $\hat f_B(x)=B^{-1}\sum_{b=1}^B T_{\theta_b}(x)$. Then
	\[
	\Var\!\bigl(\hat f_B(x)\mid X\bigr)
	=
	\frac{1}{B^2}\Var\!\left(\sum_{b=1}^B T_{\theta_b}(x)\mid X\right)
	=
	\frac{1}{B^2}\sum_{b=1}^B\sum_{b'=1}^B
	\Cov\!\bigl(T_{\theta_b}(x),T_{\theta_{b'}}(x)\mid X\bigr).
	\]
	
	By exchangeability of the tree-generating mechanism,
	$\Var(T_{\theta_b}(x)\mid X)=\sigma_T^2(x)$ for all $b$
	and $\Cov(T_{\theta_b}(x),T_{\theta_{b'}}(x)\mid X)=C_T(x)$ for all $b\neq b'$.
	Therefore
	\[
	\Var\!\bigl(\hat f_B(x)\mid X\bigr)
	=
	\frac{1}{B^2}\Bigl(B\,\sigma_T^2(x)+B(B-1)\,C_T(x)\Bigr)
	=
	\frac{1}{B}\sigma_T^2(x)+\frac{B-1}{B}C_T(x).
	\]
\end{proof}

This identity separates variability due to finite aggregation, $\sigma_T^2(x)/B$, from structural dependence induced by the design, $C_T(x)$.

\subsection{Decomposing the single-tree variance}
\label{subsec:single_tree_variance_decomposition}
We now decompose the single-tree variance term $\sigma_T^2(x)=\Var(T_\theta(x)\mid X)$ appearing in Theorem~\ref{thm:finite_var}. Let $\mathcal I_\theta=(I_{\theta,1},\dots,I_{\theta,n})$ denote the resampling indicators for the tree indexed by $\theta$, where $I_{\theta,i}=1$ if observation $i$ is exposed to the tree construction and $I_{\theta,i}=0$ otherwise. The resampling indicators determine which observations are eligible to appear in the random averaging set $A_\theta(x)$, while the remaining algorithmic randomness---feature subsampling and split selection---determines how $x$ is grouped with those observations. We condition on $\mathcal I_\theta$ in addition to $X$ to separate these two sources of variability; this follows the operational structure of the algorithm, in which the resampling step determines which observations are available to each tree before any splitting decisions are made. Applying the law of total variance yields

\begin{equation}
	\label{eq:lotv_sigma}
	\Var\!\bigl(T_\theta(x)\mid X\bigr)
	=
	\underbrace{\mathbb E\!\left[
		\Var\!\bigl(T_\theta(x)\mid \mathcal I_\theta, X\bigr)
		\mid X
		\right]}_{V_{\mathrm{in}}(x)}
	+
	\underbrace{\Var\!\left(
		\mathbb E\!\bigl[T_\theta(x)\mid \mathcal I_\theta, X\bigr]
		\mid X
		\right)}_{V_{\mathrm{out}}(x)}.
\end{equation}

The first term captures variability arising from random splitting decisions after the resampling pattern is fixed; the second captures variability induced by changes in the resampling realization itself. Both components are evaluated under the joint randomness of $(Y\mid X)$ and the tree-generating mechanism $\theta$. 

\subsection{Decomposing the within-resample variance term}
\label{subsec:within_resample_variance}

We refine $V_{\mathrm{in}}(x)$ by conditioning further on the averaging set $A_\theta(x)$, which fixes the local grouping used to form the prediction at $x$. Applying the law of total variance with $A_\theta(x)$ as the intermediate conditioning object yields
\begin{equation}
	\label{eq:lotv_within_resample}
	V_{\mathrm{in}}(x)
	=
	\mathbb E\!\left[
	\Var\!\left(
	T_\theta(x)\mid \mathcal I_\theta, A_\theta(x), X
	\right)
	\;\middle|\;
	X
	\right]
	+
	\mathbb E\!\left[
	\Var\!\left(
	\mathbb E\!\left[T_\theta(x)\mid \mathcal I_\theta, A_\theta(x), X\right]
	\;\middle|\;
	\mathcal I_\theta, X
	\right)
	\;\middle|\;
	X
	\right].
\end{equation}

\noindent\textbf{First term: outcome noise through a fixed averaging rule.} Once both the resampling pattern and $A_\theta(x)$ are fixed, the weights $W_i(x;\theta)$ are deterministic and the tree prediction is a fixed linear combination of outcomes. Under conditional independence of outcomes across observational units,
\[
\Var\!\left(T_\theta(x)\mid \mathcal I_\theta, A_\theta(x), X\right)
=
\sum_{i=1}^n W_i(x;\theta)^2\,\Var(Y_i\mid A_\theta(x), X).
\]
The resampling indicators $\mathcal I_\theta$ are drawn independently of the outcome-generating process and carry no information about $Y_i$ beyond $X$, so they can be removed from the conditioning. The averaging set $A_\theta(x)$, by contrast, depends on the realized outcomes through data-adaptive split selection and thus remains. Taking expectations over the design yields
\[
\mathbb E\!\left[
\Var\!\left(T_\theta(x)\mid \mathcal I_\theta, A_\theta(x), X\right)
\;\middle|\;
X
\right]
=
\mathbb E\!\left[
\sum_{i=1}^n W_i(x;\theta)^2\,\Var(Y_i\mid A_\theta(x), X)
\;\middle|\;
X
\right].
\]
This term measures how much outcome noise at the contributing observations influences the tree prediction through the conditional prediction rule. The conditional variances $\Var(Y_i\mid A_\theta(x), X)$ are properties of the outcome model at the realized design points; the design controls the weights. 

\noindent\textbf{Second term: instability of the grouping rule.} Even with the resampling pattern fixed, different realizations of the remaining tree-level randomization can route $x$ into different terminal nodes, inducing different averaging sets $A_\theta(x)$. Because $\mathbb E[T_\theta(x)\mid \mathcal I_\theta, A_\theta(x), X] = \sum_{i=1}^n W_i(x;\theta)\,Y_i$ with deterministic weights, this component measures variability in the prediction at $x$ induced by randomness in which terminal node $x$ is routed to:
\[
\mathbb E\!\left[
\Var\!\left(
\sum_{i=1}^n W_i(x;\theta)\,Y_i
\;\middle|\;
\mathcal I_\theta, X
\right)
\;\middle|\;
X
\right].
\]
This term is zero if and only if the averaging set $A_\theta(x)$ is almost surely constant given $(\mathcal I_\theta, X)$, and reflects instability in which observations the tree-growing procedure groups with $x$.

\subsection{Decomposing the resampling component}
\label{subsec:resampling_component}

We next interpret $V_{\mathrm{out}}(x)$, which captures variability induced by the resampling design across trees. Conditioning on the resampling indicators and using the weighted representation,
\[
\mathbb E\!\left[T_\theta(x)\mid \mathcal I_\theta, X\right]
=
\sum_{i=1}^n \bar W_i(x;\mathcal I_\theta)\,Y_i,
\qquad
\bar W_i(x;\mathcal I_\theta)
:=
\mathbb E\!\left[W_i(x;\theta)\mid \mathcal I_\theta, X\right],
\]
where $\bar W_i(x;\mathcal I_\theta)$ is the expected weight assigned to observation $i$ at $x$ under the remaining tree-level randomization, given the resampling realization. Therefore
\[
V_{\mathrm{out}}(x)
=
\Var\!\left(
\sum_{i=1}^n \bar W_i(x;\mathcal I_\theta)\,Y_i
\;\middle|\;
X
\right).
\]
Different resampling realizations $\mathcal I_\theta$ change which observations are exposed to the tree, altering the distribution of $A_\theta(x)$ and shifting the expected local prediction rule $\bar W(x;\mathcal I_\theta)$. Unlike $V_{\mathrm{in}}(x)$, which reflects instability in how a given resampled dataset is locally partitioned, $V_{\mathrm{out}}(x)$ reflects variability from changes in which observations are eligible to contribute to the prediction at $x$.

\section{Decomposing the Covariance Structure}
\label{sec:covariance_decomposition}

\subsection{Law of Total Covariance Decomposition}
\label{subsec:lotc_decomposition}

We now decompose the covariance term $C_T(x)$ appearing in Theorem~\ref{thm:finite_var}. From the finite-sample variance identity, $C_T(x)$ is the limiting contribution to predictive variability as $B\to\infty$, representing dependence between tree-level conditional predictors that persists under infinite aggregation. Although trees indexed by $\theta$ and $\theta'$ are generated independently, their predictions at $x$ are generally dependent because both act on the same realized outcomes under the fixed covariate configuration $X$. To isolate the sources of this dependence, we again condition on the resampling indicators $\mathcal I_\theta$ and $\mathcal I_{\theta'}$. Applying the law of total covariance conditional on $X$ with respect to the resampling indicators yields
\begin{equation}
	\label{eq:lotc}
	C_T(x)
	=
	\mathbb E\!\left[
	\Cov\!\left(
	T_\theta(x), T_{\theta'}(x)
	\mid
	\mathcal I_\theta, \mathcal I_{\theta'}, X
	\right)
	\;\middle|\;
	X
	\right]
	+
	\Cov\!\left(
	\mathbb E\!\left[T_\theta(x)\mid \mathcal I_\theta, X\right],
	\mathbb E\!\left[T_{\theta'}(x)\mid \mathcal I_{\theta'}, X\right]
	\;\middle|\;
	X
	\right).
\end{equation}

The two terms in \eqref{eq:lotc} correspond to distinct design mechanisms. The first captures dependence arising from joint inclusion of observations in the averaging sets $A_\theta(x)$ and $A_{\theta'}(x)$, whereby the same outcomes receive positive weight in both trees. The second captures dependence arising from alignment of the induced conditional predictors at $x$: even when two trees are trained on disjoint subsets of observations, they may group $x$ with observations from the same part of the covariate space, because the underlying signal structure leads both trees to make similar partitioning decisions in the region containing $x$. Resampling schemes govern the first term through joint inclusion probabilities, while feature-level randomization and split selection govern the second through how frequently trees arrive at similar groupings of $x$.

\subsection{Covariance Induced by Shared Training Observations}
\label{subsec:obs_overlap}

We first analyze the leading term in \eqref{eq:lotc}. This component is nonzero whenever the same observation receives positive weight in both trees at $x$. Joint inclusion in the resampling realizations is necessary but not sufficient; the observation must also lie in the terminal region containing $x$ in both trees. Using the weighted representation $T_\theta(x)=\sum_{i=1}^n W_i(x;\theta)\,Y_i$ and conditioning on $(\mathcal I_\theta,\mathcal I_{\theta'},X)$, the covariance expands as
\begin{align*}
	\Cov\!\left(
	T_\theta(x),T_{\theta'}(x)
	\;\middle|\;
	\mathcal I_\theta,\mathcal I_{\theta'}, X
	\right)
	&=
	\sum_{i=1}^n\sum_{j=1}^n
	W_i(x;\theta)\,W_j(x;\theta')\,
	\Cov(Y_i,Y_j\mid \mathcal I_\theta,\mathcal I_{\theta'}, X).
\end{align*}

Because the resampling indicators are drawn independently of the outcome-generating process, conditioning on $(\mathcal I_\theta,\mathcal I_{\theta'})$ does not impact the joint distribution of $(Y_i, Y_j)$ given $X$. Under conditional independence of outcomes across observational units, $\Cov(Y_i,Y_j\mid X)=0$ for $i\neq j$, so the double sum reduces to
\[
\Cov\!\left(
T_\theta(x),T_{\theta'}(x)
\;\middle|\;
\mathcal I_\theta,\mathcal I_{\theta'}, X
\right)
=
\sum_{i=1}^n
W_i(x;\theta)\,W_i(x;\theta')\,\sigma_i^2,
\]
where $\sigma_i^2 := \Var(Y_i \mid X)$ is the conditional outcome variance at the realized design point $X_i$. Taking expectations over the tree-generating randomization conditional on $X$ yields
\[
\mathbb E\!\left[
\Cov\!\left(
T_\theta(x),T_{\theta'}(x)
\mid
\mathcal I_\theta,\mathcal I_{\theta'}, X
\right)
\;\middle|\;
X
\right]
=
\sum_{i=1}^n
\mathbb E\!\left[
W_i(x;\theta)\,W_i(x;\theta')
\;\middle|\;
X
\right]
\sigma_i^2.
\]

The joint weight factor decomposes as
\[
\mathbb E\!\left[
W_i(x;\theta)\,W_i(x;\theta')
\;\middle|\;
X
\right]
=
\mathbb P\!\left(
I_{\theta,i}=1,\;I_{\theta',i}=1
\right)
\,
\mathbb E\!\left[
W_i(x;\theta)\,W_i(x;\theta')
\mid
I_{\theta,i}=I_{\theta',i}=1, X
\right],
\]
separating joint inclusion under the resampling scheme from the conditional allocation of weight induced by the local partitions. This contribution reflects dependence induced purely by repeated reuse of the same observations in the predictions at $x$, and is identically zero only when the resampling design assigns zero probability to joint inclusion of any observation across trees.

\subsection{Covariance Induced by Partition Alignment}
\label{subsec:partition_alignment}

We now analyze the second term in \eqref{eq:lotc}, which captures dependence arising from alignment of the induced conditional predictors across independently generated trees. This mechanism does not require any observation to receive positive weight in both trees at $x$. Instead, it arises when distinct trees, potentially trained on disjoint subsets, lead to a similar terminal region containing $x$, therefore yielding similar conditional prediction rules. Two trees trained on different observations may split on the same variables at similar thresholds along the path to $x$, routing $x$ into structurally similar terminal regions; each tree then averages outcomes from the same covariate-defined subpopulation even if the contributing observations are disjoint.

To formalize this, let $R_\theta(x)\subset\mathcal X$ denote the terminal region containing $x$ under tree draw $\theta$. The averaging set can be written equivalently as
\[
A_\theta(x)=\{\, i : I_{\theta,i}=1,\; X_i \in R_\theta(x)\,\},
\qquad
W_i(x;\theta)=\frac{\mathbf 1\{I_{\theta,i}=1\}\,\mathbf 1\{X_i\in R_\theta(x)\}}{|A_\theta(x)|}.
\]
When independent trees induce similar regions $R_\theta(x)$ and $R_{\theta'}(x)$, they average different observations drawn from essentially the same local subpopulation, producing dependence through aligned partition structure rather than shared outcomes. Define the resampling-conditional prediction
\[
Z_\theta(x)
:=
\mathbb E[T_\theta(x)\mid \mathcal I_\theta, X]
=
\sum_{i=1}^n \bar W_i(x;\mathcal I_\theta)\,Y_i,
\]
where $\bar W_i(x;\mathcal I_\theta) := \mathbb E[W_i(x;\theta)\mid \mathcal I_\theta, X]$ averages over the distribution of terminal regions induced by split-level randomization, assigning observation $i$ mass proportional to how frequently it is grouped with $x$ across possible tree realizations given the resampling pattern. This is an intermediate quantity between the realized tree-level weight $W_i(x;\theta)$ and the forest-level expected weight $\mathbb{E}[W_i(x;\theta)\mid X]$, and its introduction makes the alignment mechanism formally identifiable as a distinct component of the covariance structure.

The alignment component is therefore
\[
C_{\mathrm{align}}(x)
:=
\Cov\!\left(
Z_\theta(x), Z_{\theta'}(x)
\;\middle|\;
X
\right).
\]
Under conditional independence of outcomes across observational units given $X$,
\begin{equation}
	\label{eq:calign_explicit}
	C_{\mathrm{align}}(x)
	=
	\sum_{i=1}^n
	\sigma_i^2\;
	\mathbb E\!\left[
	\bar W_i(x;\mathcal I_\theta)\,
	\bar W_i(x;\mathcal I_{\theta'})
	\;\middle|\;
	X
	\right],
\end{equation}
where $\sigma_i^2=\Var(Y_i\mid X)$. The joint expectation $\mathbb E[\bar W_i(x;\mathcal I_\theta)\,\bar W_i(x;\mathcal I_{\theta'})\mid X]$ measures whether independent trees tend to assign predictive weight at $x$ to the same observations after averaging over split-level randomization. When routing of $x$ is stable across tree realizations---so that similar split conditions recur and $x$ repeatedly falls into similar terminal regions---the observations receiving weight, though potentially disjoint across trees, are drawn from the same conditional subpopulation defined by the shared partition structure. Both tree predictions therefore estimate the same conditional mean, yielding positive covariance even without shared training observations.

Alignment reflects a more fundamental source of dependence than observation overlap. Two forests trained on independent samples from the same data-generating process $(Y\mid X)$ will discover similar partitions at $x$ whenever the signal structure produces stable split decisions, because the data-generating process---not the specific training observations---determines which splits are favorable. Observation overlap amplifies this dependence by additionally correlating the noise component: when the same $Y_i$ contributes to both predictions at $x$, outcome-level noise does not cancel across trees. Alignment is the primary structural mechanism, present whenever trees are trained on data from the same population; observation overlap provides a secondary reinforcement that requires shared training observations.

\begin{figure}[h]
	\centering
	\begin{tikzpicture}[font=\small, >=Latex]
		
		\tikzset{
			panel/.style={draw, rounded corners, inner sep=6pt},
			box/.style={draw, rounded corners, inner sep=4pt, align=center},
			treenode/.style={draw, circle, inner sep=0.8pt},
			leafbox/.style={draw, rounded corners, inner sep=2.5pt, font=\tiny, align=center},
			idx/.style={circle, draw, inner sep=0.9pt},
			picked/.style={circle, draw, fill=black!18, inner sep=0.9pt},
			lab/.style={font=\footnotesize, align=center},
			splitlab/.style={font=\tiny, fill=white, inner sep=1pt},
			routearrow/.style={
				->, red!70!black, dashed,
				line width=0.6pt,
				shorten <=1.6pt, shorten >=1.6pt
			},
			connector/.style={
				<->, red!70!black,
				line width=0.45pt,
				shorten <=1.2pt, shorten >=1.2pt
			}
		}
		
		\newlength{\panelW}
		\newlength{\panelHalfW}
		\setlength{\panelW}{0.46\textwidth}
		\setlength{\panelHalfW}{0.5\panelW}
		
		\begin{scope}[xshift=-\panelHalfW]
			
			\node[panel, minimum width=\panelW, minimum height=6.2cm, xshift=4pt] (A) {};
			\node[lab, anchor=north] at ([yshift=-1pt]A.north) {\textbf{(A) Observation overlap}};
			
			\node[box] at (-2.0,2.25) {Tree $\theta$};
			\node[box] at ( 2.0,2.25) {Tree $\theta'$};
			
			\node[treenode] (a1) at (-2.0,1.62) {};
			\node[treenode] at (-2.7,1.06) {};
			\node[treenode] at (-1.3,1.06) {};
			\draw (a1)--(-2.7,1.06);
			\draw (a1)--(-1.3,1.06);
			\node[font=\footnotesize] at (-1.1,1.06) {$x$};
			
			\node[treenode] (a2) at (2.0,1.62) {};
			\node[treenode] at (1.3,1.06) {};
			\node[treenode] at (2.7,1.06) {};
			\draw (a2)--(1.3,1.06);
			\draw (a2)--(2.7,1.06);
			\node[font=\footnotesize] at (2.9,1.06) {$x$};
			
			\node[lab] at (-2.0,0.55) {$A_\theta(x)$};
			\node[lab] at ( 2.0,0.55) {$A_{\theta'}(x)$};
			
			\foreach \k in {1,...,9} {
				\node[idx] (A1-\k) at (-2.0-1.2+0.3*\k,0.2) {\scriptsize\k};
				\node[idx] (A2-\k) at ( 2.0-1.2+0.3*\k,0.2) {\scriptsize\k};
			}
			
			\foreach \k in {3,5,7} { \node[picked] at (A1-\k.center) {\scriptsize\k}; }
			\foreach \k in {2,5,8} { \node[picked] at (A2-\k.center) {\scriptsize\k}; }
			
			\draw[thick, red!70!black, <->, shorten <=2pt, shorten >=2pt]
			(-2.0+0.3*5-1.2, -0.1) -- (2.0+0.3*5-1.2, -0.1);
			\node[font=\tiny, red!70!black] at (0,-0.35) {observation 5 weighted in both trees};
			
			\node[lab] at (0,-0.85) {\textit{Shared outcome} $\Rightarrow$ \textit{reuse-induced}};
			\node[lab] at (0,-1.15) {\textit{dependence}};
			
		\end{scope}
		
		\begin{scope}[xshift=\panelHalfW]
			
			\node[panel, minimum width=\panelW, minimum height=6.2cm] (B) {};
			\node[lab, anchor=north] at ([yshift=-1pt]B.north) {\textbf{(B) Alignment without overlap}};
			
			\node[box] at (-2.0,2.18) {Tree $\theta$};
			\node[box] at ( 2.0,2.18) {Tree $\theta'$};
		\node[font=\tiny] at (-2.0,1.76) {trained on $\{1,2,3,4\}$};
		\node[font=\tiny] at ( 2.0,1.76) {trained on $\{6,7,8,9\}$};

			\def\yRoot{1.38}
			\def\yLevA{0.74}
			\def\yLevB{0.06}
			\def\xShiftR{-0.28}
			
			\node[treenode] (L0)  at (-2.00,\yRoot) {};
			\node[splitlab, anchor=south] at (L0.north) {$X_1\!<\!c$};
			
			\node[treenode] (L1L) at (-2.85,\yLevA) {};
			\node[treenode] (L1R) at (-1.15,\yLevA) {};
			\draw (L0)--(L1L);
			\draw (L0)--(L1R);
			
			\node[splitlab, anchor=south] at (L1R.north) {$X_3\!<\!d$};
			
			\node[treenode] (L2RL) at (-1.65,\yLevB) {};
			\node[treenode] (L2RR) at (-0.65,\yLevB) {};
			\draw (L1R)--(L2RL);
			\draw (L1R)--(L2RR);
			
			\node[leafbox] (LeafL1) at (-2.85,\yLevA-0.36) {$\{1\}$};
			\node[leafbox] (LeafL2) at (-1.65,\yLevB-0.36) {$\{3\}$};
			\node[leafbox, text=red!70!black] (LeafLx) at (-0.65,\yLevB-0.36) {$\{2,4\}$};
			
			\draw[routearrow] ($(L0)!0.52!(L1R)+(-0.07,0)$) -- ($(L1R)!0.52!(L2RR)+(-0.07,0)$);
			\draw[routearrow] ($(L1R)!0.52!(L2RR)+(-0.07,0)$) -- ($(L2RR)!0.52!(LeafLx)+(-0.07,0)$);
			\node[font=\footnotesize, red!70!black, anchor=west] at ($(LeafLx.east)+(0.16,0.00)$) {$x$};
			
			\node[treenode] (R0)  at ( 2.00+\xShiftR,\yRoot) {};
			\node[splitlab, anchor=south] at (R0.north) {$X_1\!<\!c$};
			
			\node[treenode] (R1L) at ( 1.15+\xShiftR,\yLevA) {};
			\node[treenode] (R1R) at ( 2.85+\xShiftR,\yLevA) {};
			\draw (R0)--(R1L);
			\draw (R0)--(R1R);
			
			\node[splitlab, anchor=south] at (R1R.north) {$X_3\!<\!d$};
			
			\node[treenode] (R2RL) at ( 2.35+\xShiftR,\yLevB) {};
			\node[treenode] (R2RR) at ( 3.35+\xShiftR,\yLevB) {};
			\draw (R1R)--(R2RL);
			\draw (R1R)--(R2RR);
			
			\node[leafbox] (LeafR1) at (1.15+\xShiftR,\yLevA-0.36) {$\{6\}$};
			\node[leafbox] (LeafR2) at (2.35+\xShiftR,\yLevB-0.36) {$\{9\}$};
			\node[leafbox, text=red!70!black] (LeafRx) at (3.35+\xShiftR,\yLevB-0.36) {$\{7,8\}$};
			
			\draw[routearrow] ($(R0)!0.52!(R1R)+(-0.07,0)$) -- ($(R1R)!0.52!(R2RR)+(-0.07,0)$);
			\draw[routearrow] ($(R1R)!0.52!(R2RR)+(-0.07,0)$) -- ($(R2RR)!0.52!(LeafRx)+(-0.07,0)$);
			\node[font=\footnotesize, red!70!black, anchor=west] at ($(LeafRx.east)+(0.06,0.00)$) {$x$};
			
			\node[draw, dashed, red!70!black, line width=0.55pt, rounded corners=2pt, inner sep=2.0pt, fit=(LeafLx)] (FitL) {};
			\node[draw, dashed, red!70!black, line width=0.55pt, rounded corners=2pt, inner sep=2.0pt, fit=(LeafRx)] (FitR) {};
			
			\draw[connector]
			([yshift=-10pt]FitL.east) --
			node[below, font=\tiny, red!70!black, align=center, text width=2.10cm]
			{same splits\\same routing\\same region}
			([yshift=-10pt]FitR.west);
			
			\node[lab] at (-2.0,-1.6) {$T_\theta(x)=\tfrac{Y_2+Y_4}{2}$};
			\node[lab] at ( 2.0,-1.6) {$T_{\theta'}(x)=\tfrac{Y_7+Y_8}{2}$};

			\node[lab, text width=0.90\panelW] at (0,-2.30)
			{\textit{No shared outcomes, but identical partition}\\$\Rightarrow$ \textit{aligned averaging rules} $\Rightarrow$ \textit{dependence}};
			
		\end{scope}
		
	\end{tikzpicture}
	
\caption{\textbf{Two distinct design-induced dependence mechanisms at a fixed prediction point $x$.}
	(A)~\emph{Observation overlap:} independently generated trees reuse the same outcome (here, $Y_5$) in their terminal-node averages at $x$, inducing dependence through shared weighted outcomes.
	(B)~\emph{Partition alignment without overlap:} trees are grown on disjoint training sets ($\{1,2,3,4\}$ and $\{6,7,8,9\}$), yet the covariate structure at $x$ drives both trees to discover the same splits ($X_1<c$, then $X_3<d$). The prediction point $x$ routes identically through both trees (dashed red paths), landing in structurally equivalent terminal regions. The resulting predictions $T_\theta(x)$ and $T_{\theta'}(x)$ average different observations from the same covariate-defined subpopulation at $x$, producing dependence through aligned conditional prediction rules rather than shared outcomes.}
\label{fig:overlap_vs_alignment}
\end{figure}

\subsection{The covariance floor and its strict positivity}
\label{subsec:covariance_floor}

\begin{figure}[ht!]
	\centering
	\includegraphics[width=\textwidth]{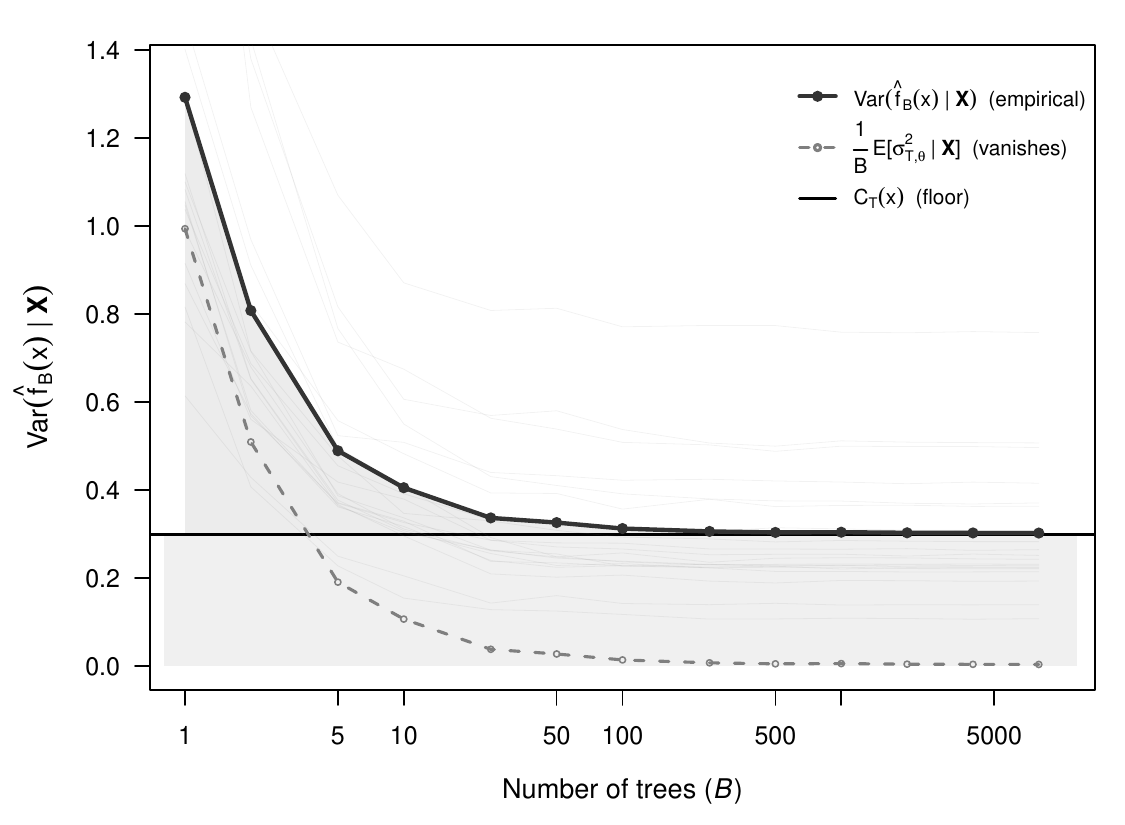}
	\caption{Empirical variance decomposition of $\hat{f}_B(x)$ as a function of the number of trees $B$ ($n = 400$, $p = 10$, $q = 4$, bootstrap; DGM in Appendix~\ref{app:dgm}). A single covariate matrix $X$ is fixed and $R = 300$ outcome vectors are drawn from the DGM. The total prediction variance (solid dark line) is the sample variance across draws, averaged over 20 test points; faint lines show individual test points before averaging. As $B$ increases, the Monte Carlo component (dashed line) goes to zero while the total variance converges to the mean covariance floor $\overline{C_T} = n_{\mathrm{pts}}^{-1}\sum_k C_T(x_k) > 0$ (horizontal line, shaded region), estimated independently from paired forests at $B = 8{,}000$ trees. Agreement is within $1.3\%$, supporting the decomposition of Theorem~\ref{thm:finite_var}.}
	\label{fig:variance_vs_B}
\end{figure}

The finite-sample variance identity in Theorem~\ref{thm:finite_var} implies that the predictive variability of a random forest decomposes into a Monte Carlo component that decreases with the number of trees and a design-induced component that persists under infinite aggregation. In particular,
\[
\lim_{B\to\infty}\Var\!\left(\hat f_B(x)\mid X\right)=C_T(x),
\]
so strict positivity of \(C_T(x)\) establishes an irreducible variance floor at fixed sample size under the fixed covariate configuration.

Sections~\ref{subsec:obs_overlap} and~\ref{subsec:partition_alignment} identified two distinct mechanisms contributing to \(C_T(x)\): reuse of training observations in local averages at \(x\), and alignment of the induced conditional prediction rules across trees. We now show that observation reuse alone is sufficient to guarantee strict positivity of the covariance floor, without invoking any alignment assumptions.

\begin{theorem}[Strict positivity of the covariance floor under observation reuse]
	\label{thm:cov_floor}
	Fix a prediction point \(x\) and a realized covariate configuration \(X\). Suppose there exists an index \(i^\star\) such that \(\sigma_{i^\star}^2>0\) and 
	\[
	\mathbb P\!\left(W_{i^\star}(x;\theta)>0 \mid X\right)>0.
	\]
	Then
	\[
	C_T(x)>0,
	\qquad\text{and hence}\qquad
	\lim_{B\to\infty}\Var\!\left(\hat f_B(x)\mid X\right)>0.
	\]
\end{theorem}

A proof for Theorem \ref{thm:cov_floor} is provided in the Supplemental Materials. 

\begin{remark}[Design interpretation]
	Theorem~\ref{thm:cov_floor} shows that repeated reuse of the same training outcomes in local averages at \(x\) is sufficient to induce a strictly positive covariance floor. This effect is governed entirely by the tree-generating design acting on the fixed covariate configuration and persists regardless of the number of trees aggregated. The condition 
	\(\mathbb{P}\!\bigl(W_{i^\star}(x;\theta) > 0 \mid X\bigr) > 0\) 
	is satisfied by all standard implementations whenever $x$ lies in the support of the design and at least one nearby observation has positive conditional variance; moreover, the alignment component provides a second, independent nonnegative contribution to the covariance floor that is not required for strict positivity but reinforces it.
\end{remark}

\section{Resolution and Structural Error in Random Forests}
\label{sec:resolution}

\subsection{The infinite-aggregation target and resolution}
\label{subsec:rf_resolution}

\begin{definition}[Design-based resolution]
	\label{def:resolution}
	Fix a prediction point $x$ and a forest design under a realized covariate configuration $X$. The resolution of the infinite-aggregation forest predictor at $x$ is defined by the distribution of the random weight vector $W(x;\theta)=(W_1(x;\theta),\ldots,W_n(x;\theta))$ induced by the tree-generating mechanism. Equivalently, resolution is characterized by the distribution of the weights appearing in the representation
	\[
	f_\infty(x)
	=
	\sum_{i=1}^n \mathbb E\!\left[W_i(x;\theta)\mid X\right]Y_i.
	\]
	Designs that assign expected weight to many observations are said to have lower resolution at $x$, while designs that concentrate expected weight on fewer observations near $x$ are said to have higher resolution at $x$.
\end{definition}

\begin{remark}[Interpretation]
	\label{rem:resolution}
	We use the term \emph{resolution} rather than bias because it more precisely characterizes the fineness of the forest's conditional prediction rule without invoking a parametric reference model. Coarser weight distributions average responses over larger sets of observations, yielding increased stability but reduced sensitivity to local variation in the outcomes. Finer weight distributions concentrate mass on fewer observations, producing predictors that adapt more sharply to local structure but that are more susceptible to dependence across trees through alignment of the induced conditional predictors. All forest design parameters affect resolution through their influence on the distribution of $W(x;\theta)$.
\end{remark}

\subsection{Training fraction and resolution scaling}
\label{subsec:training_fraction_resolution}

A defining feature of random forest constructions is that each tree is trained on a strict subset of the available observations. Under bootstrap resampling, the expected number of distinct observations used to construct a tree is approximately $0.632n$ \citep{efron1983}; under subsampling without replacement, the per-tree training size is fixed at $m=np_{\mathrm{obs}}$ for some training fraction $p_{\mathrm{obs}}\leq1$. In both cases, the distribution of the weight vector $W(x;\theta)$ depends on $m$, and therefore the resolution of $f_\infty(x)$ is determined by the per-tree training size rather than by the full sample size.

Let $T_\theta^{(m)}(x)$ denote a tree constructed using $m$ training observations, and define the corresponding infinite-aggregation target
\[
f_{\infty,m}(x)
=
\mathbb E\!\left[T_\theta^{(m)}(x)\mid X\right].
\]
Whenever the distribution of $W(x;\theta)$ varies with the per-tree training size, $f_{\infty,m}(x)$ differs from $f_{\infty,n}(x)$ for $m<n$, and this discrepancy persists as the number of trees tends to infinity. Aggregation combines subsample-based conditional predictors but does not increase the resolution beyond that attainable by the underlying tree-level fits.

\subsection{Structural error pathways}
\label{subsec:struc_errors}

Resolution and dependence arise from different aspects of the same design-induced randomness under fixed $X$. The infinite-aggregation predictor $f_\infty(x)=\mathbb E[T_\theta(x)\mid X]$ is determined by the average weight vector $\mathbb{E}[W(x;\theta)\mid X]$, which governs the resolution of the forest at $x$. The covariance floor $C_T(x)$ depends on the variability and co-movement of the random weights across design realizations. Any design choice that reduces dependence necessarily alters the distribution of these weight vectors, and therefore also affects the resolution of the infinite-aggregation predictor. This relationship is analogous to the classical bias--variance tradeoff: the forest design redistributes error between resolution and dependence, and no choice of design parameters can optimize both simultaneously.

\section{Design parameters and their effects on resolution, dependence, and error}
\label{sec:design_parameters}

\subsection{Parameterizing the tree-generating mechanism}
\label{subsec:design_parameterization}

We parameterize the tree-generating mechanism $\theta$ through a collection of design components (i.e. hyperparameters) defined relative to the observed sample size $n$. Although $n$ is not a tuning parameter in the same sense as subsampling or split restriction, it plays a structural role by determining the support over which the induced averaging sets $A_\theta(x)$ and weight vectors $W(x;\theta)$ are defined, and therefore influences both resolution and dependence.

Each tree is generated through the following randomized steps.

\begin{enumerate}
	\item \textbf{Observation selection.}
	A random index set $\mathcal S_\theta \subset \{1,\dots,n\}$ is drawn with fixed cardinality $m = |\mathcal S_\theta|$, defining the per-tree observation-selection fraction $p_{\mathrm{obs}} = m/n$. The corresponding indicator vector $\mathcal I_\theta = (I_{\theta,1},\dots,I_{\theta,n})$, where $I_{\theta,i} := \mathbf 1\{ i \in \mathcal S_\theta \}$, encodes which observations are exposed to the tree construction.
	
	\item \textbf{Split-candidate restriction.}
	At each internal node, a random candidate set of variables $\mathcal J_{\theta,v} \subset \{1,\dots,p\}$ is drawn with fixed size $q$. The split at that node is chosen by optimizing the splitting criterion over $\mathcal J_{\theta,v}$.
	
	\item \textbf{Terminal occupancy constraint.}
	A stopping rule enforces a lower bound $s \ge 1$ on the size of the induced averaging set, $|A_\theta(x)| \ge s$ almost surely, ensuring that each tree-level prediction averages at least $s$ observed outcomes. A maximum tree depth $d_{\max}$ operates as an alternative constraint on partition fineness by limiting the number of recursive splits along any root-to-leaf path.
	
	\item \textbf{Aggregation level.}
	The forest predictor averages $B$ independent realizations $\theta_1,\dots,\theta_B$. This parameter governs only Monte Carlo approximation error: it does not affect the distribution of the underlying tree-level conditional predictor $T_\theta(x)$, and therefore has no effect on resolution, dependence, or the covariance floor.
\end{enumerate}

The forest design is therefore characterized by the parameter vector $(n, p_{\mathrm{obs}}, q, s, B)$. For fixed $n$, the parameters $(p_{\mathrm{obs}}, q, s)$ jointly determine the distribution of the induced averaging sets $A_\theta(x)$ and weight vectors $W(x;\theta)$, and thereby control both the resolution of the infinite-aggregation target and the dependence structure across trees. The role of $B$ is purely to approximate this design-based target through Monte Carlo averaging. The following subsections analyze how each remaining structural design parameter acts on resolution and dependence.

\subsection{Observation-selection fraction $\boldsymbol{p}_{\mathrm{obs}}$}
\label{subsec:obs_fraction}

The observation-selection fraction $p_{\mathrm{obs}}$ directly alters both the dependence structure across trees and the resolution of the induced predictor. From the perspective of dependence, $p_{\mathrm{obs}}$ acts through the observation-overlap component of the covariance decomposition via joint inclusion probabilities
\[
\mathbb P(I_{\theta,i}=1,\;I_{\theta',i}=1),
\]
which determine how frequently the same observation can receive positive weight at $x$ in multiple trees. Increasing $p_{\mathrm{obs}}$ increases these probabilities and, holding the within-tree partitioning mechanism fixed, strengthens the observation-overlap contribution, thus increasing $C_T(x)$.

Beyond its effect on dependence, $p_{\mathrm{obs}}$ also governs resolution by restricting which observations are available to define candidate splits and terminal regions. Smaller values of $p_{\mathrm{obs}}$ limit the data available to each tree, inducing coarser averaging and lower resolution. Larger values increase resolution at the cost of stronger reuse-driven dependence between trees. Subsampling without replacement makes this trade-off explicit: $p_{\mathrm{obs}}$ becomes a directly tunable parameter controlling both resolution and dependence.

The bootstrap corresponds to a particular, non-tunable choice of $p_{\mathrm{obs}}$. Although each tree is grown from a resample of size $n$, the expected fraction of distinct observations satisfies
\[
\mathbb P(i\in\mathcal I_\theta)=1-(1-1/n)^n \approx 1-e^{-1},
\qquad
\mathbb P(i\in\mathcal I_\theta\cap\mathcal I_{\theta'})
\approx (1-e^{-1})^2,
\]
so the observation-selection fraction and the resulting magnitude of observation-induced dependence are intrinsic features of the bootstrap design rather than tunable choices.

\subsection{Split-candidate restriction size $\boldsymbol{\mathrm(q)}$}
\label{subsec:candidate_restriction}

Restricting the number of candidate variables at each split influences the dependence between trees by changing how frequently independently generated trees make similar splitting decisions along the path to $x$. This effect operates through the alignment component of the covariance decomposition.

The mechanism can be made explicit through a combinatorial property of how candidate variables are sampled at each split. 

\begin{lemma}[Shared candidate variables under random subspaces]
	\label{lem:cand_overlap}
	Fix an internal node $v$. Under uniform sampling without replacement of candidate sets of size $q$ at node $v$,
	\[
	\mathbb E\!\left[
	\left|S_{\theta,v}^{(q)}\cap S_{\theta',v}^{(q)}\right|
	\right]
	=
	\frac{q^2}{p},
	\]
	where $S_{\theta,v}^{(q)}$ and $S_{\theta',v}^{(q)}$ denote the candidate-variable sets drawn independently for two trees. In particular, the expected overlap is strictly increasing in $q$ for $q\in\{1,\dots,p\}$.
\end{lemma}

A proof is provided in Appendix~\ref{app:proof_candidate_overlap}. Lemma~\ref{lem:cand_overlap} isolates the design mechanism by which candidate restriction controls the opportunity for repeated split selection across trees. The following result shows that this increased opportunity translates directly into stronger alignment.

\begin{proposition}[Candidate restriction weakens alignment probability]
	\label{prop:restrict_alignment}
	Fix a prediction point $x$. Consider two forest designs that differ only in candidate-set size, with $q'<q$. Define the alignment probability under candidate-set size $q$ as
	\[
	\alpha_q(x) := \mathbb P\!\left(Z_\theta^{(q)}(x) = Z_{\theta'}^{(q)}(x)\right),
	\]
	where $Z_\theta^{(q)}(x) = \mathbb E[T_\theta^{(q)}(x) \mid \mathcal I_\theta, X]$ denotes the resampling-conditional predictor under candidate-set size $q$, and $\theta, \theta'$ are independent draws from the tree-generating mechanism. Assume that at each internal node along the path to $x$, the splitting criterion selects a single best split almost surely. Then
	\[
	\alpha_{q'}(x)\le \alpha_q(x),
	\]
	with strict inequality whenever restricting from $q$ to $q'$ alters split selection with positive probability along the path to $x$.
\end{proposition}

A proof is provided in Appendix ~\ref{app:proof_restrict_overlap}. Reducing $q$ therefore decreases the probability that independently generated trees lead to similar conditional predictors at $x$, reducing the alignment component of the $C_T(x)$.

Candidate restriction simultaneously affects resolution. Increasing $q$ concentrates the distribution of induced conditional predictors by reducing variability in split selection, producing more stable resolution at $x$. Decreasing $q$ increases diversity across tree realizations but broadens the weight distribution. Unlike $p_{\mathrm{obs}}$, which acts exclusively through joint inclusion probabilities, candidate restriction also affects the observation-overlap component indirectly: when trees make similar splitting decisions, a jointly included observation is more likely to receive positive weight at $x$ in both trees. Candidate restriction therefore modulates both terms of the covariance decomposition~\eqref{eq:lotc}.

\begin{figure}[t!]
	\centering
	\includegraphics[width=\textwidth]{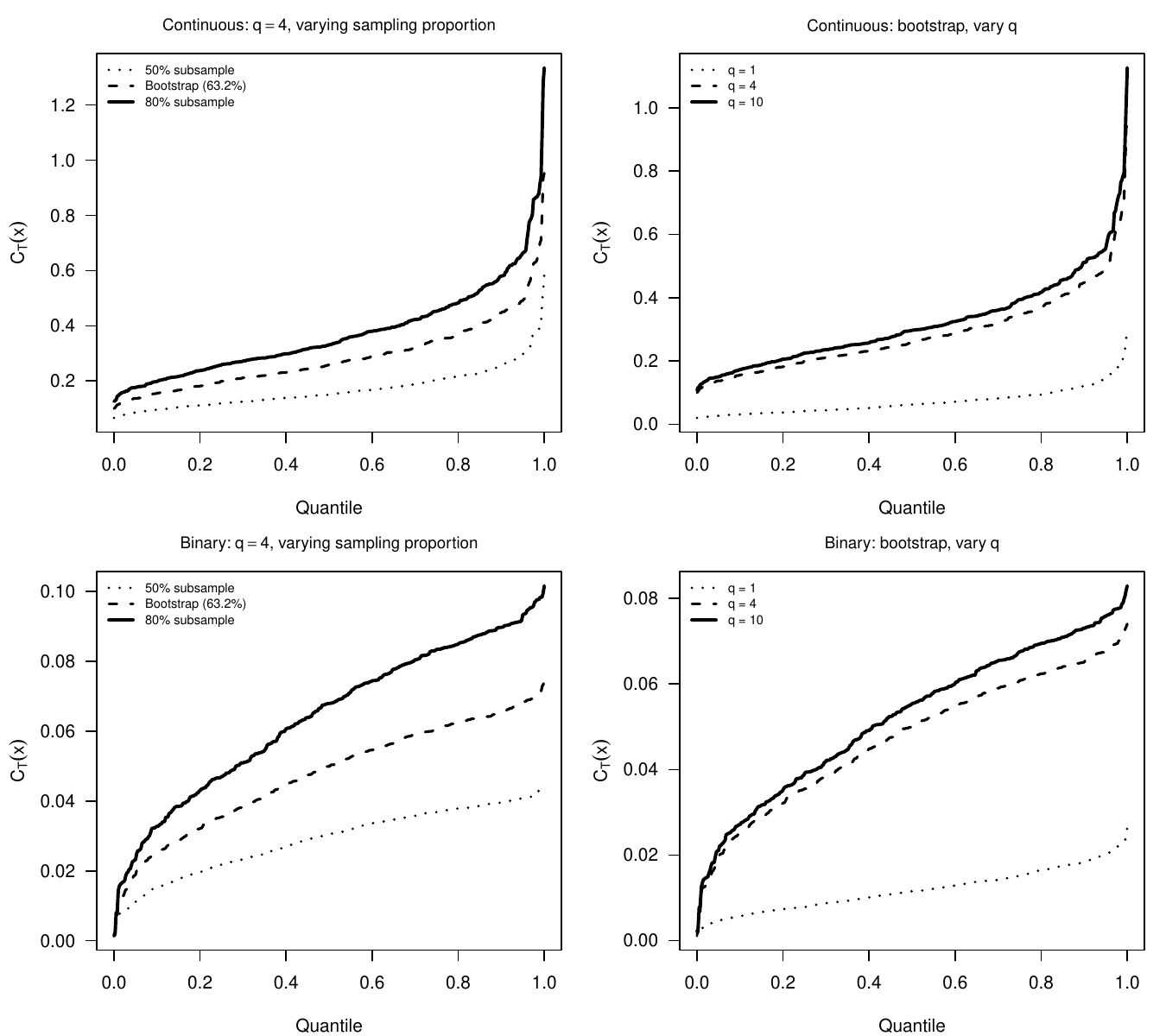}
	\caption{Quantile profiles of the true covariance floor $C_T(x)$ across $n_{\mathrm{test}} = 400$ test points (Scenario: $n = 400$, $p = 10$; DGM in Appendix~\ref{app:dgm}). Left column varies the observation sampling scheme at fixed $q = 4$; right column varies $q$ at fixed bootstrap resampling. Top row: continuous outcomes; bottom row: binary outcomes. Increasing $q$ increases $C_T(x)$. Increasing $p_{obs}$ increases $C_T(x)$, consistent with the increased tree correlation induced by larger candidate sets.}
	\label{fig:Ct_profiles}
\end{figure}

\subsection{Minimum terminal occupancy $\boldsymbol{\mathrm(s)}$ and maximum tree depth}
\label{subsec:terminal_occupancy}

The minimum terminal node occupancy $s$ imposes a lower bound on the number of observations in each terminal region, defining the scale at which local averaging may be realized. Larger values of $s$ lead to coarser partitions with greater stability but reduced resolution; smaller values lead to finer partitioning and higher-resolution predictors at the cost of increased single-tree variability. Because coarser partitions are more likely to be reused across tree realizations, increasing $s$ strengthens alignment-induced dependence and raises $C_T(x)$ relative to a forest with smaller terminal nodes but the same candidate set size $q$. This effect is structural and persists under infinite aggregation. A maximum tree depth constraint $d_{\max}$ operates through the same mechanism: shallower trees produce coarser partitions, increasing between-tree agreement and inflating $C_T(x)$. When both constraints are active, the effective resolution at $x$---and therefore the magnitude of the covariance floor---is determined by whichever is more restrictive.

\subsection{Sample size and persistence of structural dependence}
\label{subsec:sample_size_dependence}
We now address how increasing sample size affects the covariance floor when the design parameters $(p_{\mathrm{obs}},q,s)$ are held fixed.

From the finite-sample covariance decomposition, $C_{T,n}(x)$ is governed by the joint distribution of resampling patterns and induced conditional predictors. Increasing $n$ has two competing effects. First, larger samples improve the quality of each tree's local estimate: terminal nodes contain more observations, reducing single-tree variance and shrinking the Monte Carlo component $B^{-1}\sigma^2_{T,\theta}(x)$. Second, with more data each tree estimates $\mu(x)$ more accurately, reducing the between-tree dispersion that drives $C_T(x)$. This tighter agreement across trees can reduce $C_T(x)$ in absolute terms, as will be shown empirically in Section~\ref{sec:simulations} where $\overline{C_T}$ is consistently smaller at $n = 400$ than at $n = 200$ for matched $(p, q, p_{\mathrm{obs}})$ configurations. However, increasing $n$ does not eliminate structural dependence. The covariance floor decreases but the design-induced randomization that produces it---candidate restriction, enforced terminal occupancy, non-unique split choices---is unchanged, so $C_T(x)$ remains strictly positive for any finite $n$.

\subsection{Why added randomization improves finite-sample behavior}
\label{subsec:why_randomness}

Additional randomization in split selection directly alters the dependence structure of the forest. Unlike aggregation, which averages independent realizations of a fixed randomized procedure, added randomization modifies the tree-generating mechanism itself and therefore changes both the induced predictor and its structural dependence properties.

A concrete example is augmented bagging, which introduces $r\ge0$ independent noise variables into the predictor space prior to growing each tree \citep{mentch2022}. Increasing $r$ reduces the probability that any particular split is repeatedly selected along the path to $x$, dispersing the distribution of conditional predictors induced by the forest design. From the perspective of Section~\ref{subsec:partition_alignment}, this weakens alignment across trees. At the same time, added randomization spreads probability mass across a broader collection of conditional prediction rules, altering the resolution of the induced predictor.

These two effects act in opposition. Added randomization weakens alignment and reduces structural dependence, while simultaneously redistributing probability mass across conditional predictors and altering resolution. Improvements in finite-sample behavior arise when the reduction in alignment-induced dependence outweighs the accompanying change in resolution. This perspective unifies several existing explanations for the empirical benefits of randomization in tree ensembles, including regularization effects \citep{mentch2020}, performance gains in randomized constructions \citep{liu2025}, and smoothing properties controlled by design parameters \citep{curth2024}, by identifying alignment-induced dependence and resolution as the two competing mechanisms through which added randomization acts.

\section{Estimation of the Covariance Floor}
\label{sec:ct_estimation}

\subsection{A variance representation of $\boldsymbol{C}_T(x)$}
\label{subsec:ct_identity}

The covariance floor $C_T(x)$ is defined as a covariance under the joint randomness of $(Y \mid X)$ and the tree-generating mechanism. To construct an estimator, we exploit an equivalent representation that expresses $C_T(x)$ as variance over repeated draws of $(Y \mid X)$. For a realized pair $(X,Y)$, define
\[
f_\infty(x;X,Y)
:=
\mathbb E\!\left[T_\theta(x)\mid X,Y\right],
\]
the deterministic limit of a forest constructed on realized data and averaged over the tree-generating mechanism. Unlike the $X$-conditional target $f_\infty(x)=\mathbb E[T_\theta(x)\mid X]$ introduced in Section~\ref{subsec:forest_estimator}, which remains random through $Y$, the quantity $f_\infty(x;X,Y)$ is fully determined once both $X$ and $Y$ are observed.

\begin{lemma}[Variance representation of the covariance floor]
	\label{lem:ct_identity}
	Fix a prediction point $x$ and condition on the realized covariate configuration $X$. Then
	\[
	C_T(x)
	=
	\Cov\!\bigl(T_\theta(x),T_{\theta'}(x)\mid X\bigr)
	=
	\Var\!\left(f_\infty(x;X,Y)\mid X\right).
	\]
\end{lemma}

A proof is provided in Appendix~\ref{app:ct_identity}. The estimand remains $C_T(x)$ as defined throughout this paper; Lemma~\ref{lem:ct_identity} simply provides an equivalent characterization as the variability of the infinite-aggregation predictor across draws of $(Y \mid X)$. This identity is the foundation for estimation: using the covariance of two independent forests fitted to the same data eliminates finite-aggregation bias by construction, because the Monte Carlo errors of the two forests are independent and cancel in the cross-product.

\subsection{Estimation of $C_T(x)$}
\label{subsec:ct_algorithm}

The covariance floor $C_T(x)$ is not a population parameter that exists independently of modeling assumptions. It is a property of the random forest procedure as defined relative to a specified model for $(Y\mid X)$, and must be estimated under that same model. We therefore generate synthetic outcomes from the identical model specification used to construct the predictor whose variance we seek to quantify.

The logic of this approach rests on a simple observation. A deployed forest maps the training data to a fitted conditional distribution $\widehat{\mathbb{P}}_n(Y \mid X)$. All downstream inference---predictions, intervals, uncertainty estimates---operates on quantities defined by this fitted distribution, not by the unknown true distribution. Synthetic replicates drawn from $\widehat{\mathbb{P}}_n(Y \mid X)$ at the observed covariate configuration $X$ are therefore exact samples from the population that governs the deployed predictions. Because the fitted model is the data-generating process for the synthetic outcomes by construction rather than by approximation, refitting forests on these replicates and computing sample covariances yields unbiased estimates of the covariance structure under the fitted model. Generating synthetic outcomes from any other distribution---including a nonparametric bootstrap over observed residuals or an independently estimated outcome model---would target a covariance floor defined by that alternative distribution rather than the one governing the deployed procedure. The remaining question---how closely the fitted-model covariance floor $C_T^{\widehat{\mathbb{P}}_n}(x)$ approximates the true floor $C_T(x)$---reduces to the quality of nuisance estimation, which we address in Proposition~\ref{prop:conservative} for continuous outcomes and Proposition~\ref{prop:binary_unbiased} for binary outcomes. 

\medskip
\noindent\textbf{Step 1: Specify the outcome model.}
Let $\widehat{\mathbb P}_n(\cdot\mid X)$ denote a fitted model for $(Y \mid X)$ at the realized covariate configuration. This model is constructed by fitting a random forest of the same implementation and hyperparameters as the procedure under analysis, and using its predictions to define a synthetic outcome distribution. Concrete specifications for continuous and binary outcomes are given in Section~\ref{subsec:practical_specification}.

\medskip
\noindent\textbf{Step 2: Generate synthetic outcome replicates.}
Draw independent synthetic outcome vectors
\[
Y^{(1)},\ldots,Y^{(R)}
\stackrel{\text{i.i.d.}}{\sim}
\widehat{\mathbb P}_n(\cdot\mid X),
\qquad
Y^{(r)}=(Y^{(r)}_1,\ldots,Y^{(r)}_n).
\]

\medskip
\noindent\textbf{Step 3: Fit two independent forests per synthetic replicate.}
For each replicate $r$, fit two forests using independent draws from the tree-generating mechanism at the same synthetic data:
\[
\widehat f_{A}^{(r)}(x)
:=
\frac{1}{B_{\mathrm{mc}}}
\sum_{b=1}^{B_{\mathrm{mc}}}
T_{\theta_{rb}^A}\!\left(x;X,Y^{(r)}\right),
\qquad
\widehat f_{B}^{(r)}(x)
:=
\frac{1}{B_{\mathrm{mc}}}
\sum_{b=1}^{B_{\mathrm{mc}}}
T_{\theta_{rb}^B}\!\left(x;X,Y^{(r)}\right),
\]
where $\{\theta_{r1}^A,\ldots,\theta_{rB_{\mathrm{mc}}}^A\}$ and $\{\theta_{r1}^B,\ldots,\theta_{rB_{\mathrm{mc}}}^B\}$ are independent collections of draws from the tree-generating mechanism. Within each replicate, the two forests share the same synthetic outcome vector $Y^{(r)}$ and the same covariate configuration $X$, but use independent tree-generating randomness.

\medskip
\noindent\textbf{Step 4: Estimate the covariance floor.}
Define
\begin{equation}
	\label{eq:ct_hat}
	\widehat C_T(x)
	:=
	\frac{1}{R-1}
	\sum_{r=1}^R
	\left(
	\widehat f_{A}^{(r)}(x)
	-
	\overline f_{A}(x)
	\right)
	\left(
	\widehat f_{B}^{(r)}(x)
	-
	\overline f_{B}(x)
	\right),
\end{equation}
where $\overline f_{A}(x)=R^{-1}\sum_{r=1}^R\widehat f_{A}^{(r)}(x)$ and $\overline f_{B}(x)=R^{-1}\sum_{r=1}^R\widehat f_{B}^{(r)}(x)$.

Because the tree-generating mechanisms are independent across the two forests conditional on $(X,Y^{(r)})$, the finite-aggregation errors $\varepsilon_A^{(r)}(x) := \widehat f_{A}^{(r)}(x) - f_\infty(x;X,Y^{(r)})$ and $\varepsilon_B^{(r)}(x) := \widehat f_{B}^{(r)}(x) - f_\infty(x;X,Y^{(r)})$ are conditionally independent given $(X,Y^{(r)})$ with mean zero. Consequently, these errors do not contribute to the covariance across replicates, and the estimator targets $\Var(f_\infty(x;X,Y)\mid X) = C_T(x)$ without bias from finite aggregation at any value of $B_{\mathrm{mc}}$.

We refer to this estimation strategy as \emph{procedure-aligned synthetic resampling} (PASR), reflecting the requirement that synthetic outcomes are generated under the same model specification that defines the procedure whose variance is being estimated. Errors in the fitted model alter which covariance floor is targeted, but do not compromise the internal validity of the estimation procedure. The properties of $\widehat C_T(x)$ as an estimator of $C_T^{\widehat{\mathbb P}_n}(x)$, and the gap between this target and the true $C_T(x)$, are formalized in Section~\ref{subsec:ct_consistency}.

\subsection{Practical specification of $\widehat{\mathbb P}_n(\cdot\mid X)$}
\label{subsec:practical_specification}

The PASR estimator~\eqref{eq:ct_hat} requires an explicit model for $(Y \mid X)$ from which synthetic outcomes can be generated. In each case, this model is constructed from estimates of the conditional mean and conditional variance of $Y$ given $X$. These estimated quantities serve only to define the synthetic outcome distribution and are not themselves the objects of inferential interest; following standard semiparametric terminology, we refer to them as \emph{nuisance components} of the estimation procedure. The accuracy of these nuisance components determines how closely the fitted model approximates the true data-generating process, and therefore governs the gap between the estimated and true covariance floors.

\medskip
\noindent\textbf{Continuous outcomes.}
We adopt a location--scale model
\[
Y_i^\star = \hat m(X_i) + \hat\sigma(X_i)\,Z_i, \qquad Z_i\stackrel{\text{i.i.d.}}{\sim}\mathcal N(0,1),
\]
where $\hat m(X)$ is an estimated conditional mean and $\hat\sigma(X)$ is an estimated conditional standard deviation.

\medskip
\noindent\textit{Conditional mean.}\quad
The conditional mean $\hat m(X)$ is estimated from the observed data by averaging cross-fitted random forest predictions. For each of $R_{\mathrm{cf}}$ independent repetitions, the training indices are randomly partitioned into two disjoint halves. A random forest is fitted on each half using the full half-sample (without internal subsampling) and used to predict on the complementary half, yielding out-of-fold predictions for every observation. Two such sequences of cross-fitted predictions are constructed using independent random partitions, producing $\hat m_1(X_i)$ and $\hat m_2(X_i)$ for each observation. The final location estimate is their average, $\hat m(X_i) = \tfrac{1}{2}(\hat m_1(X_i) + \hat m_2(X_i))$. This step is performed once on the observed outcomes before any synthetic replicates are generated.

\medskip
\noindent\textit{Conditional variance.}\quad
The conditional variance is estimated via the cross-fitted residual product \citep{chernozhukov2018}
\[
s_i = (Y_i - \hat m_1(X_i))(Y_i - \hat m_2(X_i)).
\]
Because $\hat m_1$ and $\hat m_2$ are constructed from independent training partitions, their estimation errors are approximately independent conditional on $X_i$. Consequently,
\begin{equation}
	\label{eq:sprod_expectation}
	\mathbb E[s_i \mid X_i] = \sigma^2(X_i) + \mathrm{Bias}_1(X_i)\,\mathrm{Bias}_2(X_i),
\end{equation}
where $\mathrm{Bias}_k(X_i) = m(X_i) - \mathbb E[\hat m_k(X_i) \mid X_i]$. The bias product is a second-order term that is typically much smaller than the mean squared error of either estimator individually. A random forest fitted to $\{(X_i, s_i)\}_{i=1}^n$ estimates the conditional variance function $\sigma^2(x)$; predictions from this forest at the training points yield $\hat\sigma^2(X_i)$, and the conditional standard deviation is $\hat\sigma(X_i) = \sqrt{\max(\hat\sigma^2(X_i),\,\epsilon)}$ for a small constant $\epsilon > 0$.

\medskip
\noindent\textit{Nuisance hyperparameters.}\quad
The cross-fitted mean forests and the variance forest are nuisance components: their role is to approximate the conditional mean and variance as accurately as possible, not to replicate the design of the procedure under analysis. We use random forests to estimate these nuisance components out of convenience and consistency with the overall framework, but the theory of the PASR estimator does not require the nuisance components to be estimated by the same procedure or with matching hyperparameters; any sufficiently accurate estimator of the conditional mean and variance would suffice. Investigation of alternative nuisance specifications is deferred to future work. Accordingly, these forests use the candidate-set size $q = p$ and minimum node size $1$, regardless of the hyperparameters of the target forest. Each cross-fitted mean forest trains on the full half-sample without replacement, maximizing the data available in each fold.

\medskip
\noindent\textbf{Binary outcomes.}
For binary responses $Y_i\in\{0,1\}$, the model is $Y_i^\star\sim\mathrm{Bernoulli}(\hat p(X_i))$, where $\hat p(X_i)$ is the out-of-bag predicted probability from a probability forest fitted to the observed binary outcomes using the same forest class and hyperparameters. The conditional variance $\hat p(X_i)(1-\hat p(X_i))$ is determined by the estimated probability and requires no separate variance modeling.

\medskip
In both cases, the specification of $\widehat{\mathbb P}_n(\cdot\mid X)$ is fully determined by the fitted nuisance components. Errors in $\hat m$ or $\hat\sigma$ (or $\hat p$) alter which covariance floor the PASR estimator targets, but do not compromise the internal consistency of the procedure: $\widehat C_T(x)$ remains a consistent estimator of $\Var_{\widehat{\mathbb P}_n}\!\left(f_\infty(x;X,Y)\mid X\right)$ under the fitted model, regardless of whether that model coincides with the true data-generating process.

\subsection{Unbiasedness, consistency, and the nuisance gap}
\label{subsec:ct_consistency}

The PASR estimator~\eqref{eq:ct_hat} targets the covariance floor under the model that generates the synthetic replicates. This subsection establishes that the estimator is unbiased and consistent for this target, then characterizes the gap between the fitted-model and true covariance floors separately for continuous and binary outcomes.

\begin{assumption}[Monte Carlo approximation]
	\label{ass:mc}
	For each fixed $(X,Y)$ and prediction point $x$, the Monte Carlo average satisfies
	\[
	\frac{1}{B_{\mathrm{mc}}}\sum_{b=1}^{B_{\mathrm{mc}}} T_{\theta_{b}}(x;X,Y)\longrightarrow f_\infty(x;X,Y) \quad\text{in probability as }B_{\mathrm{mc}}\to\infty,
	\]
	and the second moment of $T_\theta(x;X,Y)$ conditional on $(X,Y)$ is finite.
\end{assumption}

\begin{proposition}[Unbiasedness and consistency of $\widehat C_T(x)$]
	\label{prop:ct_consistency}
	Fix $x$ and condition on the realized covariate configuration $X$. Let $C_T^{\widehat{\mathbb P}_n}(x) := \Var_{\widehat{\mathbb P}_n}(f_\infty(x;X,Y) \mid X)$ denote the covariance floor under the fitted model. Given Assumption~\ref{ass:mc}, the PASR estimator~\eqref{eq:ct_hat} is unbiased for $C_T^{\widehat{\mathbb P}_n}(x)$ at any $B_{\mathrm{mc}}$:
	\[
	\mathbb E\!\left[\widehat C_T(x) \mid X\right] = C_T^{\widehat{\mathbb P}_n}(x).
	\]
	Under the sequential limits $B_{\mathrm{mc}} \to \infty$ followed by $R \to \infty$,
	\[
	\widehat C_T(x) \xrightarrow{P} C_T^{\widehat{\mathbb P}_n}(x).
	\]
	When the fitted model coincides with the true data-generating process, $\widehat{\mathbb P}_n = \mathbb P(\cdot \mid X)$, both results hold with $C_T(x)$ in place of $C_T^{\widehat{\mathbb P}_n}(x)$.
\end{proposition}

A proof is provided in Appendix~\ref{app:ct_consistency}. The unbiasedness at any $B_{\mathrm{mc}}$ follows from the independence of the two forests' tree-generating mechanisms conditional on $(X, Y^{(r)})$, which ensures that finite-aggregation errors cancel in the cross-product. Proposition~\ref{prop:ct_consistency} guarantees that the PASR estimator is internally valid: it recovers the covariance floor of whatever model is imposed, without systematic error from finite aggregation.

\medskip
\noindent\textit{Error decomposition.}\quad
The total estimation error can be written as
\begin{equation}
	\label{eq:error_decomp}
	\widehat C_T(x) - C_T(x) = \underbrace{\widehat C_T(x) - C_T^{\widehat{\mathbb P}_n}(x)}_{\text{Monte Carlo error}} + \underbrace{C_T^{\widehat{\mathbb P}_n}(x) - C_T(x)}_{\text{nuisance gap}}.
\end{equation}
The first term has mean zero by Proposition~\ref{prop:ct_consistency} and concentrates around zero as $R$ increases. The second is deterministic given the fitted nuisance components and reflects how well the fitted model approximates the true data-generating process. All finite-sample bias of $\widehat{C}_T(x)$ relative to $C_T(x)$ is attributable to the nuisance gap; the covariance estimation in Steps~2--4 introduces no systematic error.

\begin{corollary}[Finite-sample concentration in $R$]
	\label{cor:concentration}
	Fix $x$ and condition on the realized covariate configuration $X$. Suppose the tree-level predictions satisfy $T_\theta(x;X,Y)\in[a,b]$ almost surely under $\widehat{\mathbb P}_n(\cdot\mid X)$ for some finite interval $[a,b]$. Then for any $\epsilon>0$,
	\[
	\mathbb P\!\left( \left|\widehat C_T(x) - C_T^{\widehat{\mathbb P}_n}(x)\right| > \epsilon \;\middle|\; X \right) \;\le\; 2\exp\!\left( -\frac{c\,R\,\epsilon^2}{(b-a)^4} \right),
	\]
	for a universal constant $c>0$.
\end{corollary}

\begin{proof}
	Under the fitted model, the pairs $(\widehat f_A^{(r)}(x), \widehat f_B^{(r)}(x))$ are i.i.d.\ across $r=1,\ldots,R$ conditional on $X$. Each forest prediction lies in $[a,b]$, so the centered cross-product is bounded in an interval of width at most $(b-a)^2$. The result follows from Hoeffding's inequality applied to the U-statistic representation of the sample covariance.
\end{proof}

The concentration bound depends on the computational budget $R$ and the range of tree predictions, but not on $p$ or $n$. The Monte Carlo error is therefore controlled entirely by computational effort. Furthermore, because partition selection in squared-error forests is driven by variance-reduction criteria and within-leaf averaging is linear in $Y$, the covariance floor depends on the conditional distribution of $Y \mid X$ primarily through its first two moments. The Gaussian location--scale specification can therefore yield small nuisance gaps even when the true conditional error distribution is non-Gaussian, provided the conditional mean and variance are well approximated.

We now characterize the nuisance gap separately for each outcome type.

\medskip
\noindent\textit{Continuous outcomes: conservative nuisance gap.}\quad

\begin{assumption}[Monotonicity of $C_T$ in outcome dispersion]
	\label{ass:monotone}
	Fix $x$ and condition on the realized covariate configuration $X$. For two models $\mathbb Q_1(\cdot\mid X)$ and $\mathbb Q_2(\cdot\mid X)$ that share the same conditional mean function but satisfy $\Var_{\mathbb Q_2}(Y_i\mid X_i) \geq \Var_{\mathbb Q_1}(Y_i\mid X_i)$ for all $i$, the covariance floor satisfies $C_T^{\mathbb Q_2}(x) \geq C_T^{\mathbb Q_1}(x)$.
\end{assumption}

Assumption~\ref{ass:monotone} states that increasing outcome noise at a fixed conditional mean does not decrease the covariance floor. This holds when the split criterion responds monotonically to within-leaf outcome dispersion, as is the case for variance-reduction criteria used in standard CART implementations.

\begin{proposition}[Conservative nuisance gap for continuous outcomes]
	\label{prop:conservative}
	Consider the continuous outcome specification of Section~\ref{subsec:practical_specification}, with conditional variance estimated via the cross-fitted residual product. Suppose $\hat m_1$ and $\hat m_2$ are constructed by the same algorithm applied to independent training partitions of equal size, so that their conditional bias functions coincide: $\mathrm{Bias}_1(x) = \mathrm{Bias}_2(x) := \beta(x)$ for all $x$. Then:
	\[
	\mathbb E[s_i \mid X_i] = \sigma^2(X_i) + \beta^2(X_i) \geq \sigma^2(X_i)
	\]
	at every design point, with equality if and only if $\hat m_1$ is conditionally unbiased at $X_i$. Under Assumption~\ref{ass:monotone},
	\[
	C_T^{\widehat{\mathbb P}_n}(x) \;\geq\; C_T(x),
	\]
	with equality when $\hat m_1$ is conditionally unbiased.
\end{proposition}

A proof is provided in Appendix~\ref{app:conservative}. The cross-fitted residual product estimates $\sigma^2(X_i)$ with upward bias equal to $\beta^2(X_i)$, the squared systematic error of the mean estimator. Because $\beta^2(X_i)$ is the same quantity that governs the forest's predictive mean squared error, the PASR estimator's accuracy is directly tied to the quality of the deployed prediction: forests with low prediction error produce small nuisance gaps, while forests with large prediction error produce wider but still valid intervals. The conservative direction ensures that estimation difficulty manifests as overcoverage rather than undercoverage---the nuisance gap can never produce anti-conservative intervals for continuous outcomes. The PASR estimator thus inherits the difficulty of the prediction problem it seeks to characterize: settings where the forest predicts well are precisely the settings where its uncertainty can be estimated most accurately.

\medskip
\noindent\textit{Binary outcomes: asymptotic unbiasedness.}\quad

For binary outcomes, the nuisance gap has a qualitatively different character. The conditional variance under the fitted Bernoulli model is $\hat{p}(X_i)(1 - \hat{p}(X_i))$---a deterministic function of the forest's own prediction rather than a separately estimated quantity. This eliminates the source of conservative bias in continuous outcomes.

\begin{proposition}[Asymptotic unbiasedness for binary outcomes]
	\label{prop:binary_unbiased}
	Consider a classification forest with tree-level predictions $T_\theta(x) = \sum_{i=1}^n W_i(x;\theta) Y_i$, where $Y_i \mid X \sim \mathrm{Bernoulli}(p(X_i))$ independently. Let $\bar{W}_i(x) = \mathbb{E}_\theta[W_i(x;\theta) \mid X]$ denote the expected forest weight for observation $i$. The PASR estimator targets $C_T^{\widehat{\mathbb{P}}_n}(x)$ and the true covariance floor is $C_T(x)$, defined respectively by
	\[
	C_T^{\widehat{\mathbb{P}}_n}(x) = \sum_{i=1}^n \bar{W}_i(x)^2 \, \hat{p}(X_i)(1 - \hat{p}(X_i)), \qquad C_T(x) = \sum_{i=1}^n \bar{W}_i(x)^2 \, p(X_i)(1 - p(X_i)),
	\]
	where $\hat{p}(X_i)$ is the forest's fitted probability at training point $X_i$. Then:
	\begin{enumerate}
		\item[\textnormal{(i)}] The PASR estimator is exactly unbiased for $C_T^{\widehat{\mathbb{P}}_n}(x)$ at any finite $B_{\mathrm{mc}}$ (Proposition~\ref{prop:ct_consistency}).
		\item[\textnormal{(ii)}] Taking expectations over $Y \mid X$,
		\[
		\mathbb{E}_Y\!\left[C_T^{\widehat{\mathbb{P}}_n}(x) \mid X\right]
		= C_T(x) - \sum_{i=1}^n \bar{W}_i(x)^2 \, \Var\!\left(\hat{p}(X_i) \mid X\right).
		\]
		\item[\textnormal{(iii)}] Under the consistency conditions assumed throughout, $\Var(\hat{p}(X_i) \mid X) = O(n^{-1})$ uniformly in $i$ and $\sum_{i=1}^n \bar{W}_i(x)^2 = O(n^{-1})$, so
		\[
		\mathbb{E}_Y\!\left[C_T^{\widehat{\mathbb{P}}_n}(x) \mid X\right] = C_T(x) + O(n^{-2}),
		\]
		and the bias decreases at rate $O(n^{-2})$.
	\end{enumerate}
\end{proposition}

A proof is provided in Appendix~\ref{app:binary_unbiased}. The bias in part~(ii) arises from Jensen's inequality applied to the concave function $p \mapsto p(1-p)$: the variability of $\hat{p}(X_i)$ around its target produces a small downward bias in the fitted-model variance at each training point. The $O(n^{-2})$ rate in part~(iii) reflects the product of two $O(n^{-1})$ quantities---the variance of the forest prediction at each training point and the concentration of forest weights---so the bias decreases faster than either component alone.

\begin{remark}
	\label{rem:binary_vs_continuous}
	The contrast between the two outcome types is instructive. For continuous outcomes, the nuisance gap is nonnegative and first-order: the separately estimated conditional variance absorbs $\beta^2(X_i)$, producing conservative bias that scales with the difficulty of estimating the conditional mean. For binary outcomes, the nuisance gap is $O(n^{-2})$ and slightly anti-conservative: the concavity of $p(1-p)$ produces a small downward bias that is empirically negligible across all configurations evaluated, including a high-dimensional stress test where $n = p$. The key structural difference is that the Bernoulli variance is a deterministic function of the prediction itself, eliminating the independent estimation channel that drives conservative inflation in the continuous setting.
\end{remark}

\subsection{Numerical stability of the estimation procedure}
\label{subsec:numerical_stability}

The PASR procedure requires fitting forests to synthetic outcome vectors across $R$ replicates. A natural concern is whether the resulting variance estimates are numerically stable, particularly when the covariates exhibit collinearity. Recall that conditional on a realized tree structure, the prediction at $x$ is a weighted average of outcomes $\tilde f(x;Y) = \sum_{i=1}^n W_i(x)\,Y_i$ with nonnegative weights summing to one.

\begin{proposition}[Uniform outcome stability conditional on tree structure]
	\label{prop:lipschitz_rf}
	Fix a realized tree structure and prediction point \(x\). For any perturbation vector \(\delta Y\in\mathbb R^n\),
	\[
	\bigl|\tilde f(x;Y+\delta Y)-\tilde f(x;Y)\bigr|
	\;\le\;
	\|\delta Y\|_\infty.
	\]
	The mapping $Y\mapsto\tilde f(x;Y)$ is $1$-Lipschitz; that is, the change in prediction at $x$ is bounded by the largest change in any single observed outcome.
\end{proposition}

\begin{proof}
	By linearity,
	$\tilde f(x;Y+\delta Y)-\tilde f(x;Y) = \sum_{i=1}^n W_i(x)\,\delta Y_i$.
	Because the weights are nonnegative and sum to one,
	\[
	\left|\sum_{i=1}^n W_i(x)\,\delta Y_i\right|
	\le
	\sum_{i=1}^n W_i(x)\,|\delta Y_i|
	\le
	\max_{1\le i\le n}|\delta Y_i|.
	\]
\end{proof}

\begin{remark}[Implications for variance estimation under collinearity]
	\label{rem:glm_contrast}
	For linear and generalized linear models, predictions depend on outcomes through fitted coefficients involving an inverse information matrix, so small perturbations in $Y$ can be amplified under near collinearity. The prediction variance at a point $x$ can grow without bound as the design matrix becomes ill-conditioned. Tree-based predictors do not share this property. Conditional on their realized structure, tree-based predictors are convex combinations of outcomes, and because $\sum W_i^2 \leq \sum W_i = 1$, the prediction variance satisfies $\Var(\tilde f(x;Y) \mid X) \leq \max_i \sigma_i^2$ regardless of the collinearity structure of the covariates. This bound carries through to the covariance floor: because $C_T(x) = \Var(f_\infty(x;X,Y) \mid X)$ and $f_\infty(x;X,Y)$ is an average of such convex combinations, the covariance floor itself cannot be inflated by collinearity. The PASR procedure therefore produces stable variance estimates across synthetic replicates even when the predictor space contains highly correlated variables.
\end{remark}

\subsection{Operational variance estimation}
\label{subsec:operational}

Combining Theorem~\ref{thm:finite_var} with Lemma~\ref{lem:ct_identity} yields a decomposition of the finite-$B$ variance that separates the $\theta$-variability at realized data from the outcome-induced structural component:
\[
\Var\!\left(\hat f_B(x)\mid X\right)
=
\frac{1}{B}\,\mathbb E\!\left[\sigma_{T,\theta}^2(x)\mid X\right]
+
C_T(x),
\]
where $\sigma_{T,\theta}^2(x):=\Var(T_\theta(x)\mid X,Y)$ denotes the tree-level variance arising purely from the tree-generating mechanism at fixed data. The identity follows from writing $\sigma_T^2(x)=\mathbb E[\sigma_{T,\theta}^2(x)\mid X]+C_T(x)$ via the law of total variance and substituting into Theorem~\ref{thm:finite_var}.

For a practitioner with data in hand, the within-forest sample variance of tree-level predictions $\{T_{\theta_b}(x)\}_{b=1}^B$ from a single fitted forest estimates $\sigma_{T,\theta}^2(x)$ at the observed outcome vector. The covariance floor $C_T(x)$ is estimated by $\widehat C_T(x)$ via PASR. The operational finite-$B$ variance estimator is therefore
\begin{equation}
	\label{eq:var_hat_operational}
	\widehat{\Var}\!\left(\hat f_B(x)\mid X\right)
	=
	\frac{1}{B}\,\widehat\sigma_{T,\theta}^{\,2}(x)
	+
	\widehat C_T(x),
\end{equation}
where $\widehat\sigma_{T,\theta}^{\,2}(x)=\frac{1}{B-1}\sum_{b=1}^B\bigl(T_{\theta_b}(x)-\hat f_B(x)\bigr)^2$ is the empirical variance of tree-level predictions from the fitted forest.

\subsection{Prediction and confidence intervals for deployed forests}
\label{subsec:prediction_intervals}

The operational variance estimator~\eqref{eq:var_hat_operational} quantifies uncertainty in the forest prediction $\hat f_B(x)$ itself. The appropriate interval construction depends on whether the prediction targets a continuous outcome or a probability.

\subsubsection*{Continuous outcomes: prediction intervals for $Y_{\mathrm{new}}$}

For predicting a new outcome $Y_{\mathrm{new}}$ at a covariate value $x$, an additional component is needed: the irreducible outcome variability $\sigma^2(x) = \Var(Y_{\mathrm{new}} \mid X_{\mathrm{new}} = x)$, which represents noise that would remain even if the conditional mean were known exactly. The total predictive variance is
\[
\Var\!\left(Y_{\mathrm{new}} - \hat f_B(x) \mid X\right)
=
\sigma^2(x)
\;+\;
\frac{1}{B}\,\sigma_{T,\theta}^2(x)
\;+\;
C_T(x),
\]
where the three terms correspond to irreducible outcome noise, Monte Carlo variability from finite aggregation, and the design-induced covariance floor. An estimated prediction interval at level $1-\alpha$ is
\begin{equation}
	\label{eq:pred_interval}
	\hat f_B(x)
	\;\pm\;
	z_{\alpha/2}
	\sqrt{
		\hat\sigma^2(x)
		\;+\;
		\frac{1}{B}\,\widehat\sigma_{T,\theta}^{\,2}(x)
		\;+\;
		\widehat C_T(x)
	},
\end{equation}
where $\hat\sigma^2(x)$ is available from the variance forest fitted during the nuisance estimation of Section~\ref{subsec:practical_specification}. By Proposition~\ref{prop:conservative}, the bias direction of both $\hat\sigma^2(x)$ and $\widehat C_T(x)$ is upward, so any systematic error in the prediction interval favors overcoverage rather than undercoverage.

\subsubsection*{Binary outcomes: confidence intervals for $p(x)$}

For probability forests targeting the conditional probability $p(x) = \mathbb P(Y = 1 \mid X = x)$, the natural inferential target is the precision of the probability estimate itself. A prediction interval for the binary outcome $Y_{\mathrm{new}} \in \{0,1\}$ is not meaningful: any interval spanning both values achieves trivial coverage, while the Gaussian approximation underlying~\eqref{eq:pred_interval} is poorly suited to discrete outcomes.

The operationally relevant quantity is a confidence interval for $p(x)$:
\begin{equation}
	\label{eq:ci_binary}
	\hat p_B(x)
	\;\pm\;
	z_{\alpha/2}
	\sqrt{
		\frac{1}{B}\,\widehat\sigma_{T,\theta}^{\,2}(x)
		\;+\;
		\widehat C_T(x)
	}.
\end{equation}
The irreducible outcome variance $p(x)(1-p(x))$ does not appear because the interval targets the estimand $p(x)$, not a new realization from it. This construction answers the question a practitioner faces when deploying a probability forest: given the data in hand and the forest already fitted, how precisely is the conditional probability estimated at $x$?

As noted in Section~\ref{sec:intro}, no existing method provides this type of pointwise uncertainty quantification for predicted probabilities from a deployed probability forest. Existing asymptotic approaches target sampling variability of the infinite-aggregation target under repeated dataset generation, quantifying how the forest's target would change if new training data were collected. The PASR-based intervals in~\eqref{eq:pred_interval} and~\eqref{eq:ci_binary} answer a different question: given a trained forest and a new observation characterized by $X_{\mathrm{new}}$, how much uncertainty surrounds the prediction at that point? They condition on the realized training data and quantify the total uncertainty of the deployed prediction at $x$, including the design-induced covariance floor that persists under infinite aggregation. All components are estimable from the fitted forest and the PASR procedure alone, without resampling over new datasets or appealing to asymptotic sampling theory.

For continuous outcomes, ignoring the covariance floor and constructing prediction intervals from the within-forest variance alone---which estimates only the Monte Carlo component $\sigma_{T,\theta}^2(x)/B$---produces intervals that undercover systematically as $B$ grows, because the floor component is not captured. For binary outcomes, the consequence is more severe: the within-forest variance estimates only the Monte Carlo variability of $\hat{p}_B(x)$ around the data-conditional limit $f_\infty(x;X,Y)$, which goes to zero as $B$ increases. Without the covariance floor, the confidence interval collapses to a point, yielding coverage that goes to zero for the true $p(x)$. The covariance floor is therefore not merely an improvement over existing intervals---it is the component that makes pointwise inference for deployed probability forests possible. Formal coverage properties, including finite-sample calibration and comparison with conformal prediction methods, are deferred to future work.

\section{Simulation study of the empirical estimation}
\label{sec:simulations}

\subsection{Simulation design}
\label{subsec:sim_design}

We evaluated the PASR estimator across 36 scenarios spanning two sample sizes ($n \in \{200, 400\}$), two predictor dimensions ($p \in \{10, 30\}$), three values of the candidate split set size ($q \in \{1, \lceil\sqrt{p}\,\rceil, p\}$), and three observation sampling schemes (50\% and 80\% subsampling without replacement and standard bootstrap). Both continuous and binary outcomes are evaluated for each scenario. The data-generating mechanism uses mixed predictor types (continuous, binary, and multinomial), nonlinear mean and variance functions with interactions, and heteroscedastic errors for continuous outcomes; binary outcomes target 40\% prevalence. Full specification appears in Appendix~\ref{app:dgm}.

The simulation proceeds in two stages for each scenario. Both stages share a single fixed covariate matrix $X$ of $n$ rows, generated once from the predictor distribution (Appendix~\ref{app:dgm}) and held constant throughout. All forests in both stages are trained on these same $n$ observations---only the outcome vector $Y$ changes across replications. A separate set of $n_{\mathrm{test}} = 400$ evaluation points is constructed by resampling rows from $X$ with replacement and adding Gaussian jitter ($\sigma = 0.02$) to continuous predictors, producing covariate configurations near the training support but not identical to any training row. These evaluation points are also fixed across all replications and serve only as locations at which $C_T(x)$ and $\widehat C_T(x)$ are assessed; they do not enter any forest's training data.

\medskip
\noindent\textit{Stage 1: True covariance floor.}\quad
For each of $R_{\mathrm{true}} = 300$ replications, we draw a fresh outcome vector $Y$ of length $n$ from the data-generating mechanism at the fixed $X$, train a random forest of $B_{\mathrm{true}} = 8{,}000$ trees on the $n$-row dataset $(X, Y)$ under the scenario's $(q, \text{sampling scheme})$ configuration using a common realization of the tree-generating mechanism across replications, and record the forest's predictions at all $n_{\mathrm{test}}$ evaluation points. Because only $Y$ varies across replications while both $X$ and the tree-generating seeds are held fixed, the pointwise sample covariance of predictions across the $R_{\mathrm{true}}$ replications estimates $C_T(x) = \Var(f_\infty(x; X, Y) \mid X)$.

\medskip
\noindent\textit{Stage 2: PASR estimation.}\quad
For each of $S = 50$ independent datasets---each generated by drawing a fresh $Y$ of length $n$ at the same fixed $X$---we apply the full PASR procedure: fit a deployed forest of $B_{\mathrm{deploy}} = 2{,}000$ trees on the $n$-row dataset, estimate the covariance floor using $R_{\mathrm{syn}} = 150$ synthetic resamples, and estimate the cross-fitted residual product using $R_{\mathrm{cf}} = 5$ half-splits with saturated mean forests ($q = p$, min.node.size $= 1$, no subsampling) as specified in Section~\ref{subsec:practical_specification}. The PASR estimate $\widehat C_T(x)$ reported in the main tables is the pointwise mean across the $S = 50$ datasets; the dispersion of single-run estimates across datasets is examined in Appendix~\ref{app:stability}.

\medskip
\noindent\textit{Analysis and presentation.}\quad
The primary diagnostic is the pointwise bias $\widehat C_T(x) - C_T(x)$, summarized by its mean, median, and interquartile range across test points. We supplement these with a pointwise regression $\widehat C_T(x) = \hat\alpha + \hat\beta\, C_T(x)$ fit across the $n_{\mathrm{test}}$ evaluation points: the slope $\hat\beta$ measures how well the estimator tracks pointwise variation in the true floor, the intercept $\hat\alpha$ captures any additive shift, and the correlation $r$ summarizes overall pointwise agreement. A slope near 1 with intercept near zero indicates that the estimator resolves pointwise differences accurately; a slope below unity with an elevated intercept indicates that estimates are pooled toward a global mean rather than tracking local variation---a pattern that reflects smoothing rather than systematic distortion. These diagnostics are reported for all 36 scenarios in Tables~\ref{tab:continuous} and~\ref{tab:binary} (Appendix~\ref{app:tables}).

In the main text, we present detailed results for two representative scenarios: a favorable regime ($n = 400$, $p = 10$, $q = 4$, bootstrap) where the nuisance estimation has sufficient data relative to model complexity, and a challenging regime ($n = 200$, $p = 30$, $q = 6$, bootstrap) where each cross-fitting fold ($n/2$) sees only 100 observations with 30 features. These two scenarios span the range of estimator performance and illustrate how the PASR estimator behaves under both well-specified and difficult nuisance estimation conditions. A high-dimensional stress test ($n = p = 200$) is presented separately in Section~\ref{subsec:sim_stress}.

All simulation code, including scripts for reproducing figures, running the PASR procedure, and summarizing results, is made available on the first author's Github page: https://github.com/NateOConnellPhD

\subsection{Estimation results}
\label{subsec:sim_results}

Figure~\ref{fig:scatter} displays the pointwise relationship between $\widehat C_T(x)$ and $C_T(x)$ for the two representative scenarios, and Figure~\ref{fig:bias} summarizes the cross-scenario relationship between mean estimated and mean true covariance floors. Tables~\ref{tab:continuous} and~\ref{tab:binary} in Appendix~\ref{app:tables} report the full diagnostics across all 36 scenarios.

\begin{figure}[t!]
	\centering
	\includegraphics[width=\textwidth]{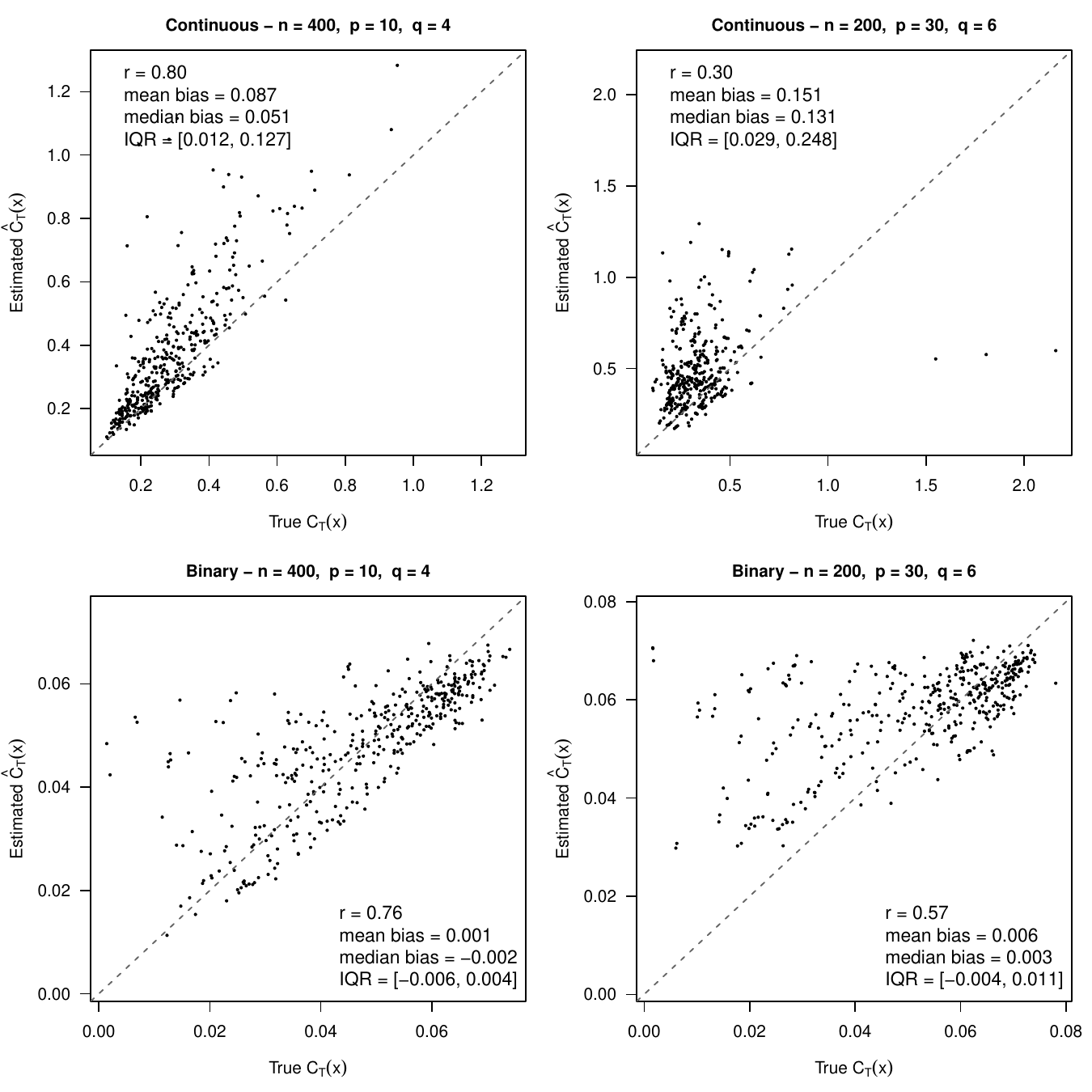}
	\caption{Pointwise relationship between the PASR estimate $\widehat C_T(x)$ and the true covariance floor $C_T(x)$ at $n_{\mathrm{test}} = 400$ test points, for the favorable regime (left; $n = 400$, $p = 10$, $q = 4$, bootstrap) and the challenging regime (right; $n = 200$, $p = 30$, $q = 6$, bootstrap). Top row: continuous outcomes; bottom row: binary outcomes. Dashed line is the identity. The PASR estimate is the mean across $S = 50$ independent datasets. For continuous outcomes, the majority of test points form a tight cluster with moderate conservative shift; test points with large true $C_T(x)$ are pooled toward the level of this cluster. Binary outcomes exhibit near-zero mean bias in both regimes.}
	\label{fig:scatter}
\end{figure}

\medskip
\noindent\textit{Continuous outcomes.}\quad
The PASR estimator is uniformly conservative for continuous outcomes: the mean estimate $\overline{\widehat C_T}$ exceeds the mean true floor $\overline{C_T}$ in every one of the 36 scenarios (Figure~\ref{fig:bias}, left panel), consistent with the theoretical guarantee of Proposition~\ref{prop:conservative}. The cross-scenario correlation between $\overline{\widehat C_T}$ and $\overline{C_T}$ is $r = 0.997$, indicating that the estimator tracks the true floor level with near-perfect fidelity across a wide range of design configurations.

The magnitude of the conservative bias is governed by the nuisance gap $\beta^2(x)$ and therefore varies with the difficulty of estimating the conditional mean. In the favorable regime ($n = 400$, $p = 10$), median bias ranges from $0.007$ to $0.088$ across the nine $(q, \text{sampling scheme})$ configurations, and the median is consistently much smaller than the mean (Table~\ref{tab:continuous}). In the challenging regime ($n = 200$, $p = 30$), each cross-fitting fold sees only 100 observations with 30 features, and the nuisance gap is correspondingly larger: median bias ranges from $0.030$ to $0.198$.

The pointwise scatterplots in Figure~\ref{fig:scatter} demonstrate the structure behind these summaries. In both regimes, the majority of test points form a cluster where most $C_T(x)$ lie, and $\widehat C_T(x)$ tracks $C_T(x)$ closely with a moderate conservative shift. A smaller number of test points with low true $C_T(x)$ exhibit substantially inflated estimates, appearing well above the identity line at the low end of the horizontal axis. These are test points where the conditional mean is locally difficult to estimate, producing a large nuisance gap $\beta^2(x)$ that inflates the fitted variance and drives the PASR estimate upward---exactly the mechanism characterized by Proposition~\ref{prop:conservative}. This concentration of overestimation at difficult-to-estimate test points explains why mean bias exceeds median bias across all scenarios: the typical test point within the main cluster is estimated with substantially less bias than the mean suggests. In the challenging regime, these locally inflated estimates are more prevalent, producing slopes below 1 and elevated intercepts (Table~\ref{tab:continuous}), but the direction of bias remains uniformly conservative. These difficult-to-estimate test points are more prevalent in regimes where true $C_T(x)$ is low for most points---that is, regimes with high $p$ relative to $n$, which drives down the theoretical covariance floor as described in Section~\ref{sec:design_parameters}.

The cross-fitted residual product $\bar s_i$ is approximately constant within each $(n, p)$ configuration across all values of $q$ and the observation sampling scheme (mean $\bar s \approx 1.25$ for $p = 10$ and $\bar s \approx 1.59$ for $p = 30$ at $n = 400$). This confirms that the nuisance estimation is effectively decoupled from the target procedure's design parameters, as intended by the saturated mean forest specification.

\medskip
\noindent\textit{Binary outcomes.}\quad
Binary outcomes exhibit markedly different behavior. Mean empirical bias is effectively zero across all 36 scenarios consistent with ~\ref{prop:binary_unbiased}, ranging from $-0.007$ to $+0.007$ (Table~\ref{tab:binary}; Figure~\ref{fig:bias}, right panel). This absence of systematic bias reflects a structural feature of classification forests producting predicted probabilities: the conditional outcome variance $p(x)(1-p(x))$ is not estimated by a separate nuisance procedure but is instead determined entirely by the tree-level probability predictions themselves, so the nuisance gap that drives conservative bias in the continuous case does not arise.

The pointwise scatterplots in Figure~\ref{fig:scatter} show a related pattern. In the favorable regime, the PASR estimate tracks the true floor with high fidelity across the full range of test points, with the central mass concentrated around the identity line. In the challenging regime, the central mass of test points concentrates at higher values of true $C_T(x)$ and is still tracked relatively accurately among the predominant mass of true $C_T(x)$ points, but test points with low true $C_T(x)$ are pulled upward toward the grand mean of the estimated floor. This upward pull reflects smoothing inherent to the forest-based estimator, which compresses pointwise variation toward the population average when estimation is difficult. In Table~\ref{tab:binary}, this is reflected by attenuated slopes below unity and elevated intercepts in the challenging regime, compared with slopes closer to unity and intercepts near zero in the favorable regime. Because the low-floor test points are overestimated while the central mass is estimated accurately, the deviations approximately cancel, yielding near-zero mean bias in both regimes. The cross-scenario correlation between $\overline{\widehat C_T}$ and $\overline{C_T}$ is $r = 0.991$, confirming that the estimator recovers the correct floor level even when pointwise resolution is limited. Since classification is the dominant use case for random forests in applied work, the strong mean-level performance and near-zero bias for binary outcomes is the most operationally relevant finding.

\medskip
\noindent\textit{Cross-scenario summary.}\quad
Across all 36 scenarios, the PASR estimator exhibits two consistent properties: uniform conservatism for continuous outcomes and near-zero bias for binary outcomes. These properties hold across variations in sample size, dimensionality, candidate set size, and observation sampling scheme. Figure~\ref{fig:bias} confirms that the bias direction is stable and predictable: continuous scenarios lie uniformly above zero, while binary scenarios cluster tightly around zero at all floor magnitudes.

The results summarized above represent averages of $\widehat C_T(x)$ across $S = 50$ independent simulations per scenario. In practice, a user applies the PASR procedure once to a single dataset, so the stability of individual runs is critical to the method's utility. Appendix~\ref{app:stability} examines single-run stability by decomposing run-to-run variability into a systematic component (driven by nuisance estimation quality, which shifts all test-point estimates in the same direction) and a pointwise component (driven by synthetic resampling noise at individual test points). In the favorable regime, the pointwise component dominates (mean pairwise correlation of bias across test points of $0.26$), meaning that a single run may over- or underestimate $C_T(x)$ at individual test points but does not systematically inflate or deflate estimates across the covariate space. In the harder regime, the systematic component is larger (mean pairwise correlation of $0.46$), reflecting greater variability in nuisance estimation quality across runs. Crucially, this systematic component remains uniformly conservative: even runs with poor nuisance estimates shift all test-point estimates upward rather than producing anticonservative estimates at any test point. The bias direction established above therefore holds not only in expectation across datasets, but at the single-run level relevant to practice.

\begin{figure}[t!]
	\centering
	\includegraphics[width=\textwidth]{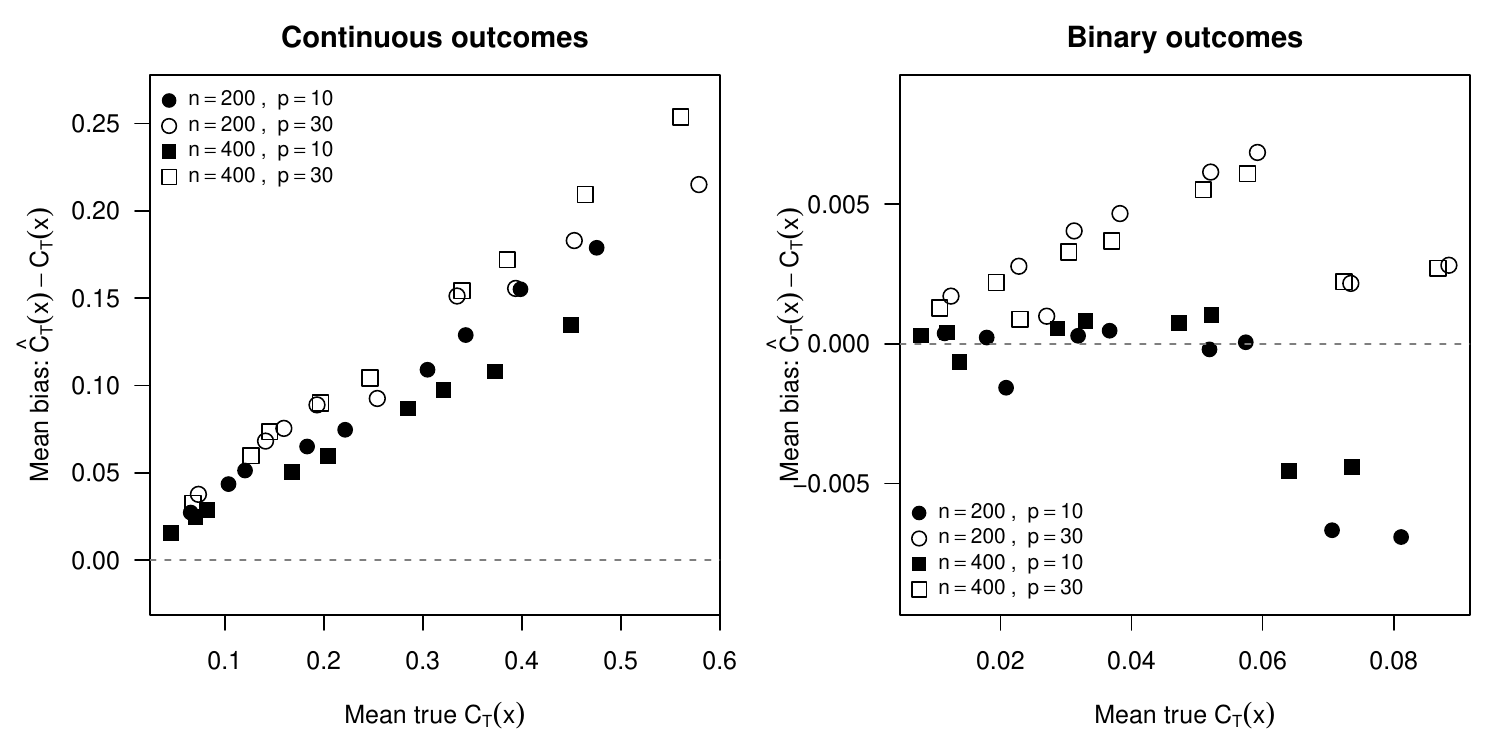}
	\caption{Mean bias of the PASR estimator as a function of the mean true covariance floor, across all 36 simulation scenarios. Left: continuous outcomes (cross-scenario $r = 0.997$); right: binary outcomes ($r = 0.991$). Each point represents one scenario; shapes encode sample size and dimensionality. For continuous outcomes, all points lie above the zero line, confirming the uniformly conservative bias direction established in Proposition~\ref{prop:conservative}. Binary outcomes exhibit near-zero bias at all floor magnitudes.}
	\label{fig:bias}
\end{figure}

\subsection{Prediction and confidence interval coverage}
\label{subsec:sim_coverage}

\begin{figure}[t!]
	\centering
	\includegraphics[width=\textwidth]{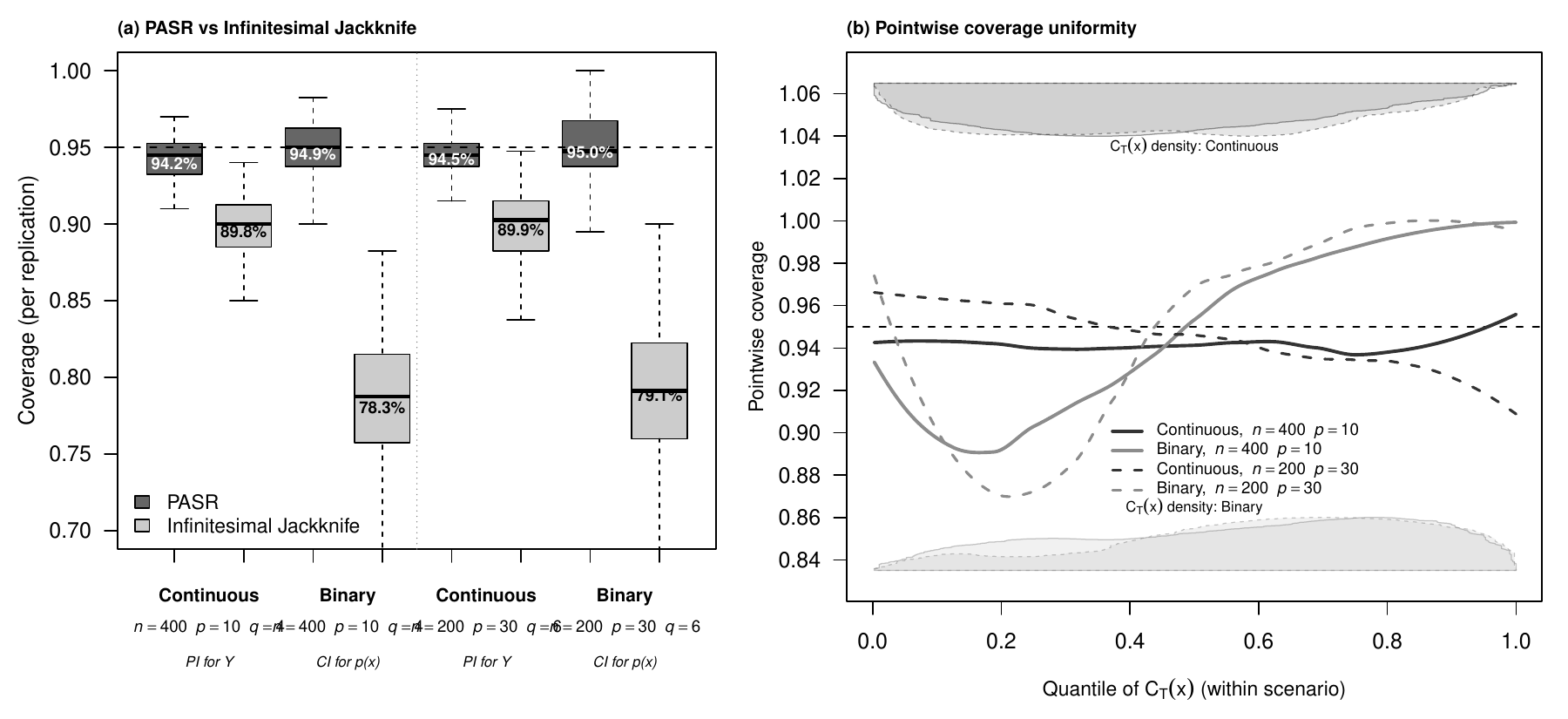}
	\caption{Coverage properties of PASR-based intervals across $R_{\mathrm{cov}} = 200$ independent replications compared with the infinitesimal jackknife (IJ). (a)~Marginal coverage for continuous prediction intervals (PI for $Y$) and binary confidence intervals (CI for $p(x)$) at two design configurations. Dark boxes: PASR; light boxes: IJ. Dashed line marks the nominal 95\% level. The IJ undercovers for continuous outcomes ($\sim$90\%) because it does not capture the covariance floor, and undercovers substantially for binary outcomes ($\sim$78--79\%) where no principled IJ estimator exists for probability forests. (b)~Pointwise PASR coverage as a function of the within-scenario quantile rank of $C_T(x)$, displayed as loess smooths. Shaded polygons show the density of true $C_T(x)$ values: continuous outcomes (top, dark) and binary outcomes (bottom, light). Solid lines: $n = 400$, $p = 10$; dashed lines: $n = 200$, $p = 30$.}
	\label{fig:coverage}
\end{figure}

To evaluate the utility of the variance decomposition, we construct prediction intervals for continuous outcomes and confidence intervals for probability estimates from binary outcomes, as described in Section~\ref{subsec:prediction_intervals}. For each of the two focal scenarios---the favorable regime ($n = 400$, $p = 10$, $q = 4$, bootstrap) and the challenging regime ($n = 200$, $p = 30$, $q = 6$, bootstrap)---we repeat the following procedure $R_{\mathrm{cov}} = 200$ times: draw training outcomes from the true data-generating mechanism at the fixed covariate configuration $X$, fit a deployed forest, estimate all variance components via PASR, and check whether the resulting interval covers the target at each of $n_{\mathrm{test}} = 400$ test points. For continuous outcomes, the coverage target is a new outcome draw $Y_{\mathrm{new}}$ at each test point, and the prediction interval is constructed as in~\eqref{eq:pred_interval}. For binary outcomes, the confidence interval~\eqref{eq:ci_binary} targets the true conditional probability $p(x)$.

As a comparator, we construct intervals using the infinitesimal jackknife (IJ) variance estimate of \citet{wager2014}, implemented via \texttt{ranger}'s built-in standard error estimation with \texttt{keep.inbag = TRUE}. We use this implementation rather than the variance estimator in the \texttt{grf} package \citep{athey2019}, which is based on the asymptotic theory of \citet{wager2018} and requires subsampling without replacement at fractions below 0.5---a structurally different forest design than the bootstrap-based forests evaluated here. The \texttt{ranger} IJ operates on the same deployed forest that generates the predictions, ensuring that the variance estimate corresponds to the actual fitted procedure. For continuous outcomes, the IJ interval adds the same cross-fitted residual product estimate $\hat\sigma^2(x)$ used in the PASR interval, giving the IJ the best possible chance: $\hat f_B(x) \pm z_{\alpha/2}\sqrt{\hat\sigma^2_{\mathrm{IJ}}(x) + \hat\sigma^2(x)}$. For binary outcomes, the IJ is not directly available for probability forests, so we fit a separate regression forest to the 0/1 outcomes with \texttt{keep.inbag = TRUE} and extract the IJ standard error from that forest---a generous interpretation of the best a practitioner could currently achieve---and construct $\hat p_B(x) \pm z_{\alpha/2}\sqrt{\hat\sigma^2_{\mathrm{IJ}}(x)}$. 

\medskip
\noindent\textit{Marginal coverage.}\quad
Figure~\ref{fig:coverage}(a) displays marginal coverage across replications. All four PASR scenario-outcome combinations achieve near-nominal coverage: 94.2\% and 94.9\% for the favorable regime (continuous and binary, respectively), and 94.5\% and 95.0\% for the challenging regime. The modest undercoverage of continuous prediction intervals relative to the 95\% target is not attributable to the variance estimator, which is conservative by Proposition~\ref{prop:conservative}, but rather to the Gaussian quantile approximation in~\eqref{eq:pred_interval}: the distribution of the prediction error $Y_{\mathrm{new}} - \hat f_B(x)$ includes a non-Gaussian component from the forest's data-adaptive partitioning, producing slightly heavier tails than the normal reference distribution. This is a standard limitation of Wald-type intervals and is not specific to the PASR framework. Binary confidence intervals achieve nominal or near-nominal coverage in both regimes.

The IJ intervals undercover relative to the PASR intervals in all four comparisons. For continuous outcomes, IJ coverage is 89.8\% and 89.9\% in the favorable and challenging regimes, respectively---approximately 5 percentage points below the PASR intervals. This gap reflects the IJ's target estimand: the IJ estimates the sampling variability of $f_\infty(x)$ under repeated data generation, which does not include the covariance floor $C_T(x)$. Because the PASR interval incorporates both the Monte Carlo variance and the covariance floor, it captures a larger share of the total prediction uncertainty and achieves coverage closer to the nominal level.

For binary outcomes, the IJ undercovers more severely, achieving 78.3\% and 79.1\% coverage. The IJ variance estimate from a regression forest on 0/1 outcomes targets the sampling variability of the forest's probability prediction under new training data, but does not account for the design-induced covariance floor---which, for probability forests, constitutes a substantial fraction of the total prediction variance. The approximately 16 percentage point gap between PASR and IJ coverage for binary outcomes underscores the practical significance of the covariance floor for classification: existing tools provide no mechanism for incorporating this component into uncertainty estimates for predicted probabilities.

\medskip
\noindent\textit{Pointwise coverage uniformity.}\quad
Figure~\ref{fig:coverage}(b) examines PASR coverage uniformity by plotting pointwise coverage as a function of the quantile rank of $C_T(x)$ within each scenario, with shaded density plots showing the distribution of true $C_T(x)$ values on the quantile axis. The continuous intervals in the favorable regime maintain approximately uniform coverage across the full range of floor magnitudes. In the challenging regime, continuous coverage dips modestly at intermediate quantiles before recovering at the upper end, consistent with the greater variability in the PASR estimate at difficult-to-estimate test points identified in Section~\ref{subsec:sim_results}.

Binary confidence intervals show undercoverage at low $C_T(x)$ quantiles where the PASR estimate is pulled upward by smoothing, transitioning to overcoverage at high $C_T(x)$ where the conservative shift dominates. The density plots at the bottom of Figure~\ref{fig:coverage}(b) show that the mass of true $C_T(x)$ values for binary outcomes concentrates in the middle and upper quantile ranges---precisely the region where coverage is at or above nominal. The undercoverage at low quantiles affects a small minority of test points with atypically low true $C_T(x)$, and has minimal impact on marginal coverage. This pattern corresponds directly to the pooling toward the grand mean documented in Section~\ref{subsec:sim_results}: low-floor test points are overestimated by the PASR smoother, producing intervals that are too narrow locally, while the central mass is estimated accurately and achieves nominal coverage. Marginal coverage remains nominal because the over- and undercoverage regions balance on average. Improved pointwise resolution in $\widehat C_T(x)$---for instance through larger $R_{\mathrm{syn}}$ or alternative aggregation strategies---would be expected to flatten these curves toward uniform coverage, and is a natural direction for future work.

\subsection{High-dimensional evaluation}
\label{subsec:sim_stress}

\begin{figure}[t!]
	\centering
	\includegraphics[width=\textwidth]{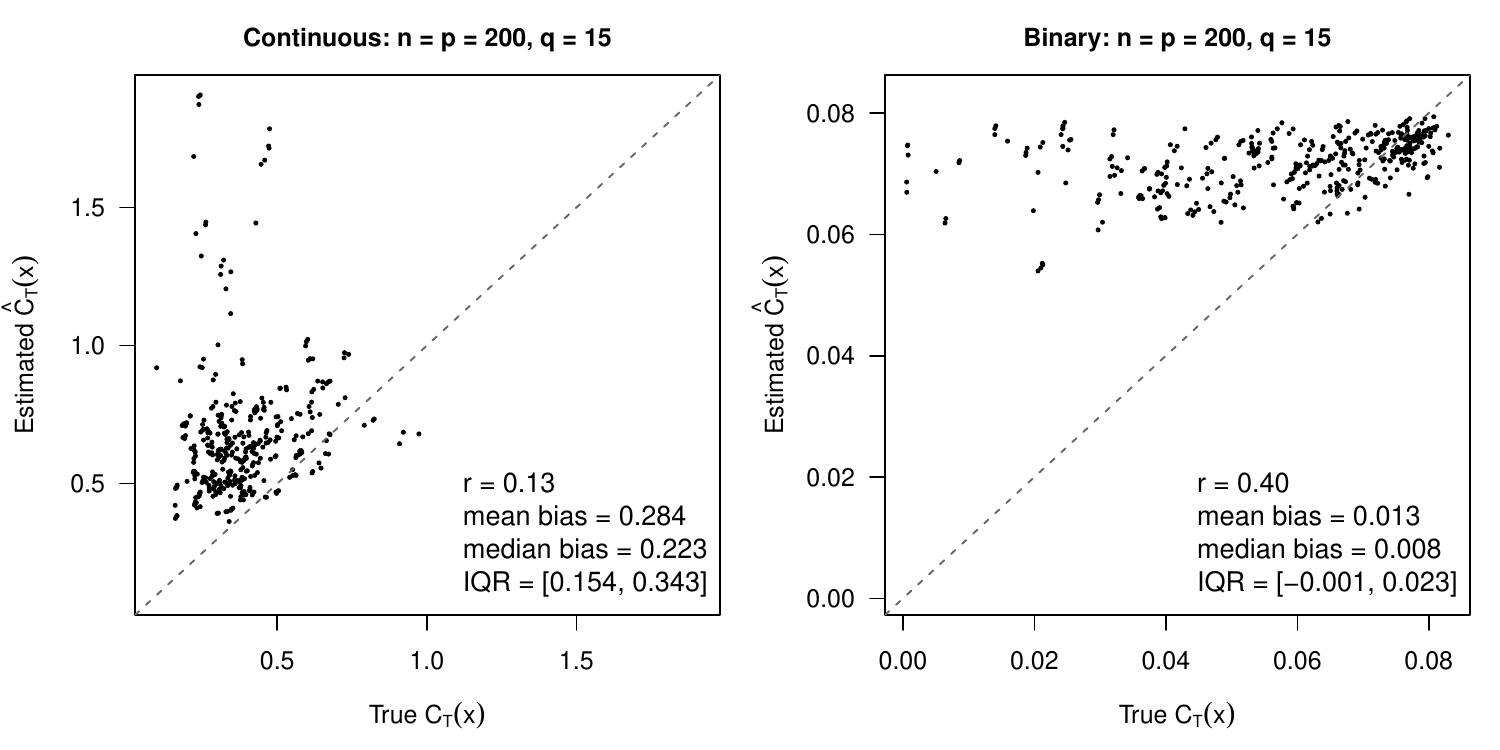}
	\caption{Pointwise relationship between $\widehat C_T(x)$ and $C_T(x)$ for the high-dimensional evaluation ($n = p = 200$, $q = 15$, bootstrap). Left: continuous outcomes. The majority of test points cluster in a tight region just above the identity line; a small number with large true $C_T(x)$ are poorly resolved, inflating the mean bias relative to the median. Right: binary outcomes. The estimator remains approximately unbiased with median bias of $0.008$ despite the extreme dimensionality.}
	\label{fig:stress}
\end{figure}

\begin{figure}[t!]
	\centering
	\includegraphics[width=\textwidth]{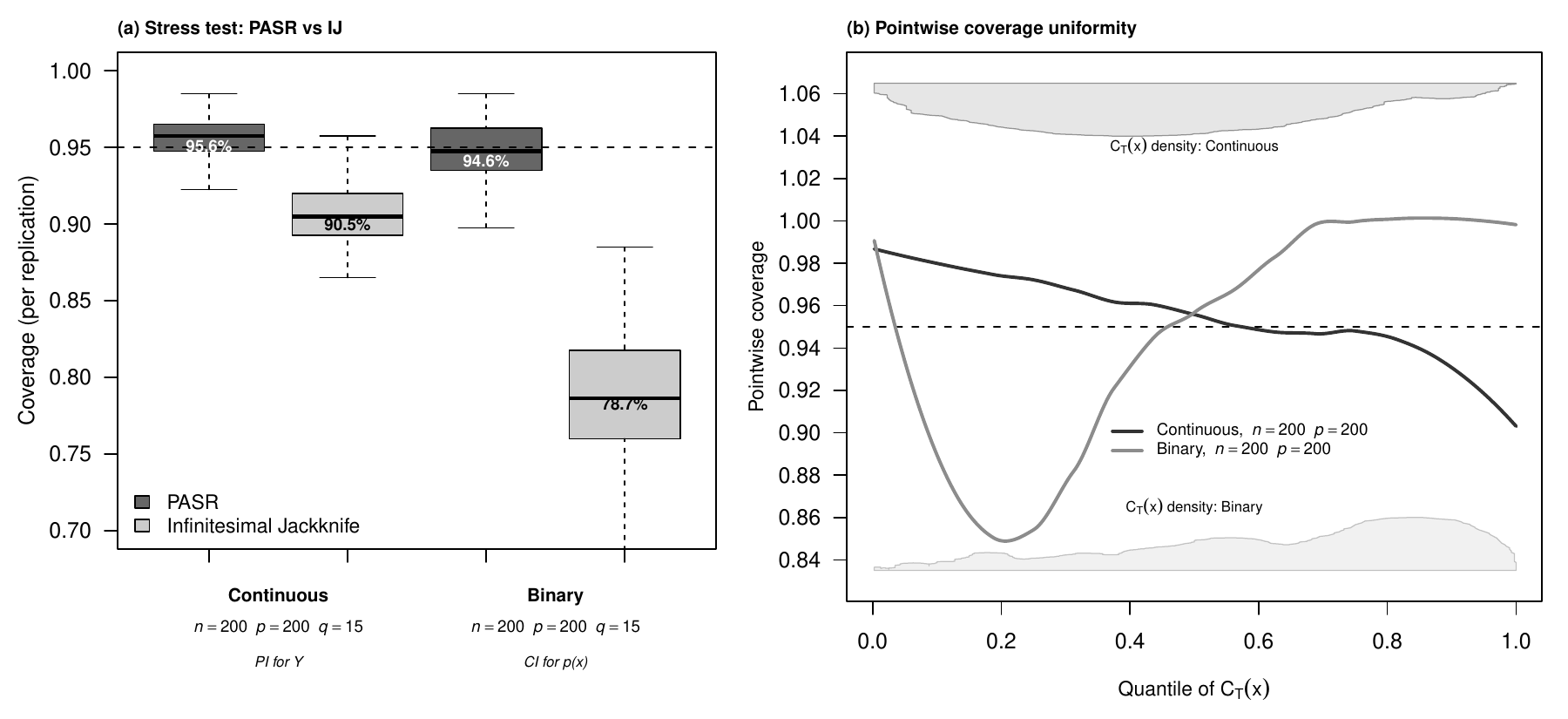}
	\caption{Coverage properties under the high-dimensional stress test ($n = p = 200$, $q = 15$). (a)~Marginal coverage for continuous prediction intervals and binary confidence intervals, comparing PASR (dark) with the infinitesimal jackknife (light). PASR achieves near-nominal coverage for both outcome types despite severely ill-conditioned nuisance estimation, while the IJ undercovers by approximately 5 percentage points for continuous outcomes and 16 percentage points for binary outcomes---consistent with the gaps observed in the main scenarios. (b)~Pointwise PASR coverage as a function of the quantile rank of $C_T(x)$, with shaded density polygons showing the distribution of true $C_T(x)$ values. The binary coverage dip at low quantiles is more pronounced than in the main scenarios, reflecting the stronger pooling documented in the estimation results above, but affects a small minority of test points where the density is sparse. Continuous coverage remains approximately uniform with slight overcoverage consistent with the larger conservative bias in this regime.}
	\label{fig:stress_coverage}
\end{figure}

To test the robustness of the PASR estimator, we evaluate a single scenario with $n = p = 200$ and $q = 15$ (bootstrap), where the majority of predictors ($>150$) are pure noise variables as defined in Appendix~\ref{app:dgm}. This represents an extreme version of the challenging regime: each cross-fitting fold sees only 100 observations with 200 features, and the nuisance estimation problem is severely underdetermined.

\medskip
\noindent\textit{Estimation accuracy.}\quad
Figure~\ref{fig:stress} displays the pointwise relationship between the estimated and true covariance floor for this scenario. For continuous outcomes, the pattern observed in the challenging regime of the main simulations is amplified. The dominant cluster of test points is estimated with a moderate conservative shift, but some test points with low true $C_T(x)$ are inflated substantially---the same nuisance-driven mechanism identified in Section~\ref{subsec:sim_results}, now more pronounced because $\beta^2(x)$ is larger when the conditional mean is estimated from 100 observations with 200 features. The median bias of $0.223$ on a median true floor of $0.357$---an overestimate of roughly 60\%---is substantially more representative of typical performance than the mean bias of $0.284$, which is inflated by these difficult-to-estimate test points (pointwise $r = 0.13$, slope $= 0.22$). For binary outcomes, the estimator performs well even in this regime: median bias of $0.008$ on a median true floor of $0.065$, with mean bias of only $0.013$.

\medskip
\noindent\textit{Interval coverage.}\quad
Figure~\ref{fig:stress_coverage} evaluates whether the conservative estimation bias translates into usable intervals under these extreme conditions. Despite the substantial overestimation of $C_T(x)$ for continuous outcomes, the PASR prediction intervals achieve near-nominal coverage. Binary confidence intervals also achieve near-nominal coverage, consistent with the near-zero estimation bias documented above.

The IJ comparison reveals the same structural gap observed in the main scenarios. For continuous outcomes, the IJ undercovers by approximately 5 percentage points; for binary outcomes, by approximately 16 percentage points. The consistency of these gaps across the main and stress test scenarios---despite vastly different dimensionality and nuisance estimation quality---confirms that the undercoverage reflects the IJ's inability to capture the covariance floor rather than any scenario-specific artifact. The covariance floor constitutes a substantial and stable fraction of the total prediction variance that existing tools do not account for.

Pointwise coverage uniformity (Figure~\ref{fig:stress_coverage}b) shows that continuous intervals maintain approximately uniform coverage with slight overcoverage, consistent with the larger conservative bias in this regime. Binary intervals exhibit the same pattern as the main results---undercoverage at low $C_T(x)$ quantiles transitioning to overcoverage at high quantiles---but more pronounced, reflecting the stronger pooling toward the grand mean. The density distribution show that this dip affects a sparse minority of test points; the central mass where coverage is at or above nominal accounts for the majority of predictions.

Even at the boundary where $n = p$ and the nuisance estimation problem is severely ill-conditioned, the PASR estimator degrades gracefully: it produces conservative estimates consistent with the true scale for continuous outcomes and approximately unbiased estimates for binary outcomes, and the resulting intervals maintain near-nominal coverage for both outcome types. The bias remains uniformly in the conservative direction, consistent with the theoretical guarantee of Proposition~\ref{prop:conservative}.

\section{Discussion}
\label{sec:discussion}

This paper reframes random forests as randomized statistical procedures acting on a fixed set of covariates, and explicitly targets the variance of $\hat f(x) \mid X$, akin to that of standard regression procedures. This structural reframing reveals that the ensemble's finite-sample variance is governed by two structurally distinct dependence mechanisms: observation reuse and partition alignment. Alignment is the more fundamental---it arises whenever trees trained on data from the same data-generating process discover similar local prediction rules, and persists even when sample splitting eliminates observation overlap entirely. Intuitively, this occurs because predictions are based on outcomes drawn from the same conditional parent population, governed by similar splits in the constructed trees. The covariance floor $C_T(x)$ that these mechanisms produce persists regardless of the number of trees in a forest and represents an intrinsic property of the deployed procedure.

The PASR estimator provides a practical framework for estimating this variance. It is unbiased for the covariance floor under the fitted model at any finite aggregation level, requires no asymptotic machinery, and degrades gracefully under difficult nuisance estimation conditions. For continuous outcomes, the cross-fitted residual product construction guarantees that any systematic error is conservative; for binary outcomes, the absence of a separate variance estimation step yields asymptotic unbiasedness at rate $O(n^{-2})$ (Proposition~\ref{prop:binary_unbiased}), confirmed empirically across all design configurations tested. The operational variance decomposition combines the PASR floor estimate with the within-forest Monte Carlo variance to produce prediction and confidence intervals that achieve near-nominal coverage for both outcome types, providing --- for the the first time --- pointwise uncertainty quantification for predicted probabilities from deployed classification forests.

The covariance floor addresses a different inferential question than existing asymptotic approaches. Methods based on U-statistic or influence-function representations \citep{mentch2016, wager2018} target the sampling variability of $f_\infty(x)$ under repeated data generation. The PASR-based intervals instead condition on the realized training data and quantify the total uncertainty of the deployed prediction, including design-induced variability that persists under infinite aggregation. These two perspectives are complementary rather than competing; a complete characterization of predictive uncertainty would require both.

The framework applies in high-dimensional settings without modification to its theoretical properties, since the estimator operates on i.i.d.\ synthetic replicates conditional on $X$. The binding constraint is nuisance estimation quality: for continuous outcomes, reliable estimation of the conditional variance function degrades when $p$ is large relative to $n$, while for binary outcomes the $O(n^{-2})$ bias rate of Proposition~\ref{prop:binary_unbiased} ensures that the estimator remains accurate even at $p > n$.. Developing nuisance strategies that remain stable in high dimensions is a natural direction for future work and represents the primary obstacle to high-dimensional application, rather than any limitation of the PASR framework itself. Nevertheless, the current implementation provides usable results in high-dimensional settings---conservative but well-calibrated---particularly for the prediction of probabilities from classification forests.

\medskip
\noindent\textbf{Extensions to other forest procedures.}
In Appendix~\ref{app:extensions} we describe how the framework extends to quantile regression forests, random survival forests, honest forests, and causal forests. Full development and empirical validation of PASR within these procedures is deferred to future work; this discussion serves to demonstrate that the theory developed here is not specific to the standard random forest but applies broadly to tree-based ensembles whose trees are generated by an exchangeable mechanism. In each case, the tree-level prediction is a deterministic functional of realized outcomes in a random averaging set, and the variance identity in Theorem~\ref{thm:finite_var} applies without modification. The covariance floor is strictly positive whenever the alignment mechanism is active, regardless of the outcome type, loss function, or within-node estimator. Honest forests provide a particularly informative case: the sample-splitting design eliminates observation overlap while leaving alignment intact, confirming that alignment is the more fundamental source of dependence. The framework does not extend directly to boosted ensembles such as gradient boosting, where trees are constructed sequentially with each tree depending on the residuals of its predecessors, violating the exchangeability assumption that underlies Theorem~\ref{thm:finite_var}. Whether analogous variance decompositions can be developed under sequential dependence is an open question for future work. 

\medskip
\noindent\textbf{Future directions.}
The framework developed here opens several avenues for further investigation. The error decomposition in~\eqref{eq:error_decomp} provides a precise target for improving finite-sample performance: advances in calibration strategies or alternative nuisance estimation would tighten the nuisance gap and improve the PASR estimates for continuous outcomes. Deriving the asymptotic distribution of $\widehat{C}_T(x)$---a central limit theorem for the PASR estimator---would enable formal inference on the floor itself. The synthetic replicates generated by the PASR procedure contain distributional information beyond the first two moments; leveraging empirical quantiles from these replicates for coverage calibration, and exploring connections with conformal prediction methods, are natural extensions of the current interval construction. The extensions in Appendix~\ref{app:extensions} establish that the framework applies to causal, honest, survival, and quantile regression forests; developing tailored PASR specifications and empirical validation for these procedures would broaden the practical reach of the methodology. 

Honest forests provide a particularly valuable setting for extending this framework beyond pointwise intervals to hypothesis testing and simultaneous inference. By eliminating observation overlap through sample splitting, honesty isolates the alignment mechanism as the sole source of $C_T(x)$, making design-conditional inference procedures more tractable. Natural extensions include testing treatment effect heterogeneity ($H_0: \tau(x_1) = \tau(x_2)$) using $\text{Cov}(\hat{\tau}(x_1), \hat{\tau}(x_2) \mid X)$, subgroup identification with explicit uncertainty bounds that distinguish confident responders from uncertain regions, and simultaneous inference calibrated to the joint distribution available from PASR replicates. These applications would leverage the design-conditional variance decomposition to enable inference tasks not currently addressable with existing asymptotic sampling-based methods.

Finally, the design-based perspective suggests that forest hyperparameters could be tuned with explicit awareness of the resolution--dependence trade-off characterized in Section~\ref{sec:design_parameters}, rather than through prediction error alone.

\section{Conclusion}
\label{sec:conclusion}

This paper develops a finite-sample, design-based theory for random forest predictors. An exact variance identity separates aggregation noise from structural dependence, and a covariance decomposition identifies observation reuse and partition alignment as the two mechanisms generating the covariance floor. The floor is estimable via procedure-aligned synthetic resampling, and the resulting estimator is unbiased under the fitted model at any finite aggregation level. Operational prediction and confidence intervals constructed from the PASR variance decomposition achieve nominal coverage for both continuous and binary outcomes. For continuous outcomes, the nuisance gap is provably conservative; for classification forests, the PASR estimator is asymptotically unbiased at rate $O(n^{-2})$, providing---for the first time---theoretically grounded pointwise uncertainty quantification for predicted probabilities from deployed classification forests. Together, these results offer a principled basis for understanding how forest design parameters govern ensemble behavior, for quantifying the total predictive uncertainty of fitted forests, and for extending these ideas to tree-based ensembles with exchangeable tree-generating mechanisms.

%% file: supplementary_material.tex
\section{Proofs and Derivations}

\subsection{Sampling variability of the infinite-forest target}
\label{app:sampling_rate}

This appendix studies the second term in the unconditional variance decomposition presented in Section~\ref{subsec:variance_decomp},
\[
\Var\!\left(
\mathbb E\!\bigl[\hat f_B(x;\mathcal D_n^\star,\theta)\mid X\bigr]
\right),
\]
which captures how the forest predictor's conditional mean varies across repeated draws of the covariate configuration from the population. For large $B$, $\hat f_B(x;\mathcal D_n^\star,\theta)$ approximates its infinite-aggregation limit $f_\infty(x;\mathcal D_n^\star)=\mathbb E_\theta[T_\theta(x;\mathcal D_n^\star)]$, and the conditional expectation $\mathbb E[\hat f_B(x;\mathcal D_n^\star,\theta)\mid X]$ converges to $\mathbb E[f_\infty(x;\mathcal D_n^\star)\mid X]$. In the infinite-aggregation limit, the second term therefore reduces to
\[
\Var\!\left(
\mathbb E\!\bigl[f_\infty(x;\mathcal D_n^\star)\mid X\bigr]
\right),
\]
which measures how the design-based target of the forest changes across repeated draws of $X$ from the population. Because $f_\infty(x;\mathcal D_n^\star)$ depends on both $X$ and $Y$, the inner expectation averages over $Y\mid X$ and the remaining variability is driven by sampling of the covariate configuration itself.

To bound this term, we establish a rate for the stronger quantity $\Var(f_\infty(x;\mathcal D_n^\star))$, which provides an upper bound since $\Var(\mathbb E[f_\infty\mid X])\le \Var(f_\infty)$ by the law of total variance.

Let $\mathcal D_n^\star=\{Z_i=(X_i,Y_i)\}_{i=1}^n$ denote an i.i.d.\ sample from an underlying population. \citet{mentch2016} and \citet{wager2018} establish that, under regularity conditions, $f_\infty(x;\mathcal D_n^\star)$ admits an asymptotically linear expansion of the form
\[
f_\infty(x;\mathcal D_n^\star)
=
f(x)
+
\frac{1}{n}\sum_{i=1}^n \psi_x(Z_i)
+
r_n(x),
\]
where $\psi_x$ is a mean-zero influence function with finite variance and $r_n(x)$ is a remainder whose variance is of smaller order than $n^{-1}$.

Taking variances yields
\[
\Var\!\bigl(f_\infty(x;\mathcal D_n^\star)\bigr)
=
\Var\!\left(\frac{1}{n}\sum_{i=1}^n \psi_x(Z_i)\right)
+
\Var\!\bigl(r_n(x)\bigr)
+
2\,\Cov\!\left(\frac{1}{n}\sum_{i=1}^n \psi_x(Z_i),\,r_n(x)\right).
\]

By independence and finite second moment,
\[
\Var\!\left(\frac{1}{n}\sum_{i=1}^n \psi_x(Z_i)\right)
=
\frac{\Var(\psi_x(Z_1))}{n}.
\]

The remainder terms satisfy $\Var(r_n(x))=o(n^{-1})$, and the cross-term is $o(n^{-1})$ by the Cauchy--Schwarz inequality applied to the covariance. Consequently,
\[
\Var\!\bigl(f_\infty(x;\mathcal D_n^\star)\bigr)
=
\frac{\Var(\psi_x(Z_1))}{n}
+
o(n^{-1}),
\]
and therefore
\[
\Var\!\left(
\mathbb E\!\bigl[f_\infty(x;\mathcal D_n^\star)\mid X\bigr]
\right)
\le
\frac{\Var(\psi_x(Z_1))}{n}
+
o(n^{-1}).
\]

The sampling variability of the infinite-forest target is thus generically of order $O(n^{-1})$, justifying the focus in the main text on the design-induced variability captured by the first term $\mathbb E[\Var(\hat f_B(x;\mathcal D_n^\star,\theta)\mid X)]$, which contains the structural covariance floor $C_T(x)$ that persists under infinite aggregation at fixed sample size.

\subsection{Proof of Theorem~\ref{thm:cov_floor}}
\label{app:proof_strict_covariance}

\begin{proof}
	Fix a prediction point \(x\) and a realized covariate configuration \(X=\{X_i\}_{i=1}^n\). All expectations and covariances are taken with respect to the joint distribution of $(Y\mid X)$ and the independent tree-generating randomizations. For this proof, we suppress $X$ throughout for notational clarity; all statements are understood to be conditional on the fixed covariate configuration consistent with the main text. 
	
	\medskip
	
	\textbf{Conditional second moments.}
	The quantities \(\sigma_i^2=\Var(Y_i\mid X)\) denote the conditional outcome variances evaluated at the realized covariate values and enter the covariance expansion as deterministic scale factors once $X$ is fixed.
	
	\medskip
	
	\textbf{Step 1: A nonnegative lower bound for \(C_T(x)\).}
	By the law-of-total-covariance decomposition in \eqref{eq:lotc},
	\[
	C_T(x)
	=
	\mathbb E\!\left[
	\Cov\!\left(
	T_\theta(x),T_{\theta'}(x)\mid \mathcal I_\theta,\mathcal I_{\theta'}
	\right)
	\right]
	+
	\Cov\!\left(
	\mathbb E[T_\theta(x)\mid \mathcal I_\theta],\,
	\mathbb E[T_{\theta'}(x)\mid \mathcal I_{\theta'}]
	\right).
	\]
	In particular,
	\[
	C_T(x)
	\;\ge\;
	\mathbb E\!\left[
	\Cov\!\left(
	T_\theta(x),T_{\theta'}(x)\mid \mathcal I_\theta,\mathcal I_{\theta'}
	\right)
	\right].
	\]
	
	\medskip
	
	\textbf{Step 2: Strict positivity of the observation-reuse contribution.}
	Under conditional independence across observational units given the fixed design points, Section~\ref{subsec:obs_overlap} shows that
	\[
	\mathbb E\!\left[
	\Cov\!\left(
	T_\theta(x),T_{\theta'}(x)\mid \mathcal I_\theta,\mathcal I_{\theta'}
	\right)
	\right]
	=
	\sum_{i=1}^n \sigma_i^2\,
	\mathbb E\!\left[
	W_i(x;\theta)\,W_i(x;\theta')
	\right],
	\]
	a sum of nonnegative terms.
	
	By assumption, there exists \(i^\star\) such that \(\sigma_{i^\star}^2>0\) and \(\mathbb P(W_{i^\star}(x;\theta)>0)>0\). Independence of \(\theta\) and \(\theta'\) implies
	\[
	\mathbb P\!\left(
	W_{i^\star}(x;\theta)>0,\;
	W_{i^\star}(x;\theta')>0
	\right)
	=
	\mathbb P\!\left(
	W_{i^\star}(x;\theta)>0
	\right)^2
	>
	0.
	\]
	On this event, \(W_{i^\star}(x;\theta)\,W_{i^\star}(x;\theta')>0\), and therefore
	\[
	\mathbb E\!\left[
	W_{i^\star}(x;\theta)\,W_{i^\star}(x;\theta')
	\right]
	>
	0.
	\]
	Hence the \(i^\star\) summand is strictly positive, and the entire sum is strictly positive.
	
	\medskip
	
	\textbf{Step 3: Conclusion.}
	Combining Step 1 and Step 2 yields \(C_T(x)>0\). The limit statement follows from Theorem~\ref{thm:finite_var}, since \(\lim_{B\to\infty}\Var(\hat f_B(x)\mid X)=C_T(x).\).
\end{proof}

\subsection{Proof of Lemma~\ref{lem:cand_overlap}}
\label{app:proof_candidate_overlap}

\begin{proof}
	Fix a realized covariate configuration $X$ and a node index $v$. All probability statements are taken with respect to the tree-generating randomization. For each tree $\theta$, let $S^{(m)}_{\theta,v}\subset\{1,\dots,p\}$ denote the candidate variable set drawn at node $v$ under candidate-set size $m$.
	
	We construct a coupling across values of $m$ as follows. For each tree $\theta$ and node $v$, draw an independent uniform random permutation $\pi_{\theta,v}$ of $\{1,\dots,p\}$, and define
	\[
	S^{(m)}_{\theta,v}
	=
	\{\pi_{\theta,v}(1),\dots,\pi_{\theta,v}(m)\}.
	\]
	This construction yields the correct marginal distribution for uniform sampling without replacement of size-$m$ candidate sets.
	
	Under this coupling, for any $m_1<m_2$,
	\[
	S^{(m_1)}_{\theta,v}\subset S^{(m_2)}_{\theta,v},
	\qquad
	S^{(m_1)}_{\theta',v}\subset S^{(m_2)}_{\theta',v}.
	\]
	Since $\pi_{\theta,v}$ and $\pi_{\theta',v}$ are independent, the intersection size
	\[
	\bigl|S^{(m)}_{\theta,v}\cap S^{(m)}_{\theta',v}\bigr|
	\]
	follows a hypergeometric distribution with expectation
	\[
	\mathbb E\!\left[
	\bigl|S^{(m)}_{\theta,v}\cap S^{(m)}_{\theta',v}\bigr|
	\right]
	=
	\frac{m^2}{p}.
	\]
	This expectation is strictly increasing in $m$ for $m\in\{1,\dots,p\}$, establishing the claimed monotonicity of candidate-set overlap.
\end{proof}

\subsection{Proof of Proposition~\ref{prop:restrict_alignment}}
\label{app:proof_restrict_overlap}
All probability statements are taken with respect to the tree-generating randomization under fixed $X$.
Fix a prediction point $x$. Consider two forest designs that differ only in the candidate-set size at each split, with $q'<q$. Let
\[
Z_\theta^{(q)}(x)=\mathbb E[T_\theta^{(q)}(x)\mid \mathcal I_\theta, X],
\qquad
Z_\theta^{(q')}(x)=\mathbb E[T_\theta^{(q')}(x)\mid \mathcal I_\theta, X]
\]
denote the corresponding resampling-conditional prediction rules.

We construct the two designs on a common probability space as follows. At each internal node, the candidate set under the $q'$-design is obtained by subsampling without replacement from the candidate set under the $q$-design. All remaining sources of randomness (observation selection, split-point evaluation, and tie-breaking) are shared across designs. This coupling is conceptual and used only for comparison of induced rule distributions; it does not correspond to an implemented algorithm.

Along the path to $x$, the $q$-design evaluates splits over a superset of coordinates relative to the $q'$-design. Consequently, whenever the $q'$-design selects a particular split at a node, the $q$-design either selects the same split or selects a different split only if an additional candidate variable available under $q$ yields a strictly better value of the splitting criterion. Under the stated nondegeneracy conditions, ties occur with probability zero or are resolved identically.

Therefore, whenever two independent trees generated under the $q'$-design induce identical conditional prediction rules at $x$, the corresponding trees under the $q$-design also induce identical conditional prediction rules. Writing
\[
A_q(x)=\{Z_\theta^{(q)}(x)=Z_{\theta'}^{(q)}(x)\},
\qquad
A_{q'}(x)=\{Z_\theta^{(q')}(x)=Z_{\theta'}^{(q')}(x)\},
\]
this implies
\[
A_{q'}(x)\subseteq A_q(x)
\quad\text{with probability one under the coupling}.
\]
Taking probabilities yields $\alpha_{q'}(x)\le \alpha_q(x)$. If restricting from $q$ to $q'$ alters the selected split with positive probability along the path to $x$, then the inclusion is strict on an event of positive probability, implying $\alpha_{q'}(x)<\alpha_q(x)$. \hfill$\square$

\subsection{Proof of Lemma~\ref{lem:ct_identity}}
\label{app:ct_identity}
\begin{proof}
	Fix a prediction point $x$ and condition throughout on the realized covariate configuration $X$. By definition,
	\[
	C_T(x)
	=
	\Cov\!\bigl(T_\theta(x),T_{\theta'}(x)\mid X\bigr),
	\]
	where $\theta$ and $\theta'$ are independent draws from the tree-generating mechanism. Apply the law of total covariance with respect to $Y$:
	\[
	\Cov\!\bigl(T_\theta(x),T_{\theta'}(x)\mid X\bigr)
	=
	\mathbb E\!\left[
	\Cov\!\bigl(T_\theta(x),T_{\theta'}(x)\mid X,Y\bigr)
	\mid X
	\right]
	+
	\Cov\!\left(
	\mathbb E[T_\theta(x)\mid X,Y],\,
	\mathbb E[T_{\theta'}(x)\mid X,Y]
	\mid X
	\right).
	\]
	
	Conditional on $(X,Y)$, the tree-generating mechanisms $\theta$ and $\theta'$ are independent, so their predictions at $x$ are independent and
	\[
	\Cov\!\bigl(T_\theta(x),T_{\theta'}(x)\mid X,Y\bigr)
	=
	0.
	\]
	
	Therefore,
	\[
	C_T(x)
	=
	\Cov\!\left(
	\mathbb E[T_\theta(x)\mid X,Y],\,
	\mathbb E[T_{\theta'}(x)\mid X,Y]
	\mid X
	\right).
	\]
	
	Since $\theta$ and $\theta'$ are identically distributed and independent given $(X,Y)$,
	\[
	\mathbb E[T_\theta(x)\mid X,Y]
	=
	\mathbb E[T_{\theta'}(x)\mid X,Y]
	=
	f_\infty(x;X,Y).
	\]
	
	Thus,
	\[
	C_T(x)
	=
	\Cov\!\bigl(f_\infty(x;X,Y),\,f_\infty(x;X,Y)\mid X\bigr)
	=
	\Var\!\left(f_\infty(x;X,Y)\mid X\right).
	\qedhere
	\]
\end{proof}

\subsection{Proof of Proposition~\ref{prop:ct_consistency}}
\label{app:ct_consistency}
\begin{proof}
	Fix a prediction point $x$ and condition throughout on the realized covariate configuration $X$. Let $C_T^{\widehat{\mathbb P}_n}(x) := \Var_{\widehat{\mathbb P}_n}(f_\infty(x;X,Y) \mid X)$ denote the covariance floor under the imposed law.
	
	\medskip
	\noindent\textbf{Step 1: Unbiasedness at any $B_{\mathrm{mc}}$.}
	
	For each synthetic replicate $r$, decompose the two forest predictions as
	\[
	\widehat f_{A}^{(r)}(x) = f_\infty(x;X,Y^{(r)}) + \varepsilon_A^{(r)}(x),
	\qquad
	\widehat f_{B}^{(r)}(x) = f_\infty(x;X,Y^{(r)}) + \varepsilon_B^{(r)}(x),
	\]
	where $\varepsilon_A^{(r)}(x)$ and $\varepsilon_B^{(r)}(x)$ are the finite-aggregation errors from forests $A$ and $B$ respectively. Conditional on $(X,Y^{(r)})$, these errors have mean zero and are independent because the two forests use independent draws from the tree-generating mechanism.
	
	The sample covariance across replicates expands as
	\begin{align*}
		\widehat C_T(x)
		&=
		\frac{1}{R-1}\sum_{r=1}^R
		\bigl(\widehat f_A^{(r)} - \overline f_A\bigr)
		\bigl(\widehat f_B^{(r)} - \overline f_B\bigr) \\
		&=
		\frac{1}{R-1}\sum_{r=1}^R
		\bigl(f_\infty^{(r)} + \varepsilon_A^{(r)} - \overline{f_\infty} 
		- \overline\varepsilon_A\bigr)
		\bigl(f_\infty^{(r)} + \varepsilon_B^{(r)} - \overline{f_\infty} 
		- \overline\varepsilon_B\bigr),
	\end{align*}
	where we write $f_\infty^{(r)}=f_\infty(x;X,Y^{(r)})$ and overbars denote sample means. Taking expectations conditional on $X$ and using the tower property through $(X,Y^{(r)})$:
	
	The cross-terms involving $\varepsilon_A^{(r)}$ and $f_\infty^{(r)}$ vanish because $\mathbb E[\varepsilon_A^{(r)}\mid X,Y^{(r)}]=0$. The cross-term $\varepsilon_A^{(r)}\varepsilon_B^{(r)}$ vanishes because $\varepsilon_A^{(r)}$ and $\varepsilon_B^{(r)}$ are conditionally independent with mean zero given $(X,Y^{(r)})$. What remains is the sample covariance of the $f_\infty^{(r)}$ terms plus the sample covariances of the $\varepsilon$ terms with themselves, which contribute zero in expectation by the same independence argument. Therefore,
	\[
	\mathbb E\!\left[\widehat C_T(x)\mid X\right]
	=
	\Var_{\widehat{\mathbb P}_n}\!\left(f_\infty(x;X,Y)\mid X\right)
	=
	C_T^{\widehat{\mathbb P}_n}(x).
	\]
	This holds for any finite $B_{\mathrm{mc}}$.
	
	\medskip
	\noindent\textbf{Step 2: Convergence as $B_{\mathrm{mc}}\to\infty$ followed 
		by $R\to\infty$.}
	
	As $B_{\mathrm{mc}}\to\infty$ for fixed $R$, Assumption~\ref{ass:mc} gives $\widehat f_A^{(r)}(x)\to f_\infty(x;X,Y^{(r)})$ and $\widehat f_B^{(r)}(x)\to f_\infty(x;X,Y^{(r)})$ in probability for each $r$. Since $R$ is fixed, by the continuous mapping theorem,
	\[
	\widehat C_T(x)
	\longrightarrow
	\frac{1}{R-1}\sum_{r=1}^R
	\left(f_\infty^{(r)}(x) - \overline f_\infty(x)\right)^2
	\quad\text{in probability}.
	\]
	As $R\to\infty$, this sample variance converges almost surely to $\Var_{\widehat{\mathbb P}_n}(f_\infty(x;X,Y)\mid X) = C_T^{\widehat{\mathbb P}_n}(x)$ by the strong law of large numbers applied to the i.i.d.\ sequence $\{f_\infty(x;X,Y^{(r)})\}_{r=1}^R$ under $\widehat{\mathbb P}_n(\cdot\mid X)$, using the finite second moment guaranteed by Assumption~\ref{ass:mc}. Combining the sequential limits yields 
	$\widehat C_T(x)\xrightarrow{P} C_T^{\widehat{\mathbb P}_n}(x)$.
\end{proof}

\subsection{Proof of Proposition~\ref{prop:conservative}}
\label{app:conservative}

\begin{proof}
Fix an observation $i$ and condition on $X$. Write $Y_i = m(X_i) + \sigma(X_i)\varepsilon_i$ where $\varepsilon_i$ has mean zero and unit variance conditional on $X_i$, and define the estimation errors $\delta_k(X_i) = m(X_i) - \hat m_k(X_i)$ for $k \in \{1,2\}$.

\medskip
\noindent\textbf{Step 1: Expand the residual product.}
\begin{align*}
	s_i 
	&= (Y_i - \hat m_1(X_i))(Y_i - \hat m_2(X_i)) \\
	&= (\sigma(X_i)\varepsilon_i + \delta_1(X_i))(\sigma(X_i)\varepsilon_i + \delta_2(X_i)) \\
	&= \sigma^2(X_i)\varepsilon_i^2 
	+ \sigma(X_i)\varepsilon_i\bigl(\delta_1(X_i) + \delta_2(X_i)\bigr) 
	+ \delta_1(X_i)\delta_2(X_i).
\end{align*}

\medskip
\noindent\textbf{Step 2: Take conditional expectations.}
Taking $\mathbb E[\cdot \mid X_i]$:
\begin{itemize}
	\item First term: $\mathbb E[\sigma^2(X_i)\varepsilon_i^2 \mid X_i] = \sigma^2(X_i)$, since $\mathbb E[\varepsilon_i^2 \mid X_i] = 1$.
	\item Second term: Each $\delta_k(X_i)$ decomposes as $\delta_k(X_i) = \mathrm{Bias}_k(X_i) + \eta_k(X_i)$ where $\mathrm{Bias}_k(X_i) = m(X_i) - \mathbb E[\hat m_k(X_i) \mid X_i]$ and $\eta_k(X_i) = \mathbb E[\hat m_k(X_i) \mid X_i] - \hat m_k(X_i)$ has mean zero. Since $\varepsilon_i$ is independent of $\hat m_1$ and $\hat m_2$ conditional on $X_i$ (the noise on observation $i$ is independent of the estimators, which for the cross-fitted construction are trained on data excluding observation $i$'s fold), we have $\mathbb E[\sigma(X_i)\varepsilon_i(\delta_1(X_i) + \delta_2(X_i)) \mid X_i] = 0$.
	\item Third term: Since $\hat m_1$ and $\hat m_2$ are trained on independent data partitions, $\eta_1(X_i)$ and $\eta_2(X_i)$ are conditionally independent given $X_i$ with mean zero. Therefore:
	\begin{align*}
		\mathbb E[\delta_1(X_i)\delta_2(X_i) \mid X_i] 
		&= \mathbb E[(\mathrm{Bias}_1(X_i) + \eta_1(X_i))(\mathrm{Bias}_2(X_i) + \eta_2(X_i)) \mid X_i] \\
		&= \mathrm{Bias}_1(X_i)\,\mathrm{Bias}_2(X_i) + \underbrace{\mathbb E[\eta_1(X_i)\eta_2(X_i) \mid X_i]}_{=\,0\text{ by independence}}.
	\end{align*}
\end{itemize}

Combining:
\[
\mathbb E[s_i \mid X_i] = \sigma^2(X_i) + \mathrm{Bias}_1(X_i)\,\mathrm{Bias}_2(X_i).
\]

\medskip
\noindent\textbf{Step 3: Apply the symmetry condition.}
By assumption, $\hat m_1$ and $\hat m_2$ are constructed by the same algorithm on training partitions of equal size, so $\mathrm{Bias}_1(x) = \mathrm{Bias}_2(x) := \beta(x)$ for all $x$. Therefore:
\[
\mathbb E[s_i \mid X_i] = \sigma^2(X_i) + \beta^2(X_i) \geq \sigma^2(X_i),
\]
with equality if and only if $\beta(X_i) = 0$.

\medskip
\noindent\textbf{Step 4: Propagation to the imposed law.}
The scale forest estimates $\mathbb E[s_i \mid X_i] = \sigma^2(X_i) + \beta^2(X_i)$, so the imposed law $\widehat{\mathbb P}_n(\cdot \mid X)$ generates synthetic outcomes with conditional variance $\hat\sigma^2(X_i) \geq \sigma^2(X_i)$ in expectation. Under the same conditional mean (both the imposed and true laws center at $\hat m(X_i)$ and $m(X_i)$ respectively; the mean discrepancy affects $C_T$ through a separate channel), the inflated variance implies that the imposed law has weakly greater conditional dispersion than the true law at every design point.

\medskip
\noindent\textbf{Step 5: Apply monotonicity.}
Under Assumption~\ref{ass:monotone}, $C_T^{\widehat{\mathbb P}_n}(x) \geq C_T(x)$ follows directly from the weakly greater conditional variance of the imposed law. Equality holds when $\beta(x) = 0$ for all $x$, i.e., when the mean estimators are conditionally unbiased.
\end{proof}

\subsection{Proof of Proposition~\ref{prop:binary_unbiased}}
\label{app:binary_unbiased}

\begin{proof}
	Fix a prediction point $x$ and condition throughout on the realized covariate configuration $X$.
	
	\medskip
	\noindent\textbf{Part (i).}
	This follows directly from Proposition~\ref{prop:ct_consistency}, which establishes unbiasedness of $\widehat{C}_T(x)$ for $C_T^{\widehat{\mathbb{P}}_n}(x)$ at any finite $B_{\mathrm{mc}}$ regardless of the outcome type.
	
	\medskip
	\noindent\textbf{Part (ii).}
	The tree-level prediction is linear in $Y$: $T_\theta(x) = \sum_{i=1}^n W_i(x;\theta) Y_i$, where $W_i(x;\theta)$ is the weight assigned to observation $i$ by tree $\theta$ at prediction point $x$. Conditional on $X$, the infinite-forest prediction is
	\[
	f_\infty(x; X, Y) = \mathbb{E}_\theta[T_\theta(x) \mid X, Y] = \sum_{i=1}^n \bar{W}_i(x) Y_i,
	\]
	where $\bar{W}_i(x) = \mathbb{E}_\theta[W_i(x;\theta) \mid X]$. Under the true law, $Y_i \mid X \sim \mathrm{Bernoulli}(p(X_i))$ independently, so
	\[
	C_T(x) = \Var(f_\infty(x; X, Y) \mid X) = \sum_{i=1}^n \bar{W}_i(x)^2 \, p(X_i)(1 - p(X_i)).
	\]
	Under the fitted model, $Y_i^{(r)} \sim \mathrm{Bernoulli}(\hat{p}(X_i))$ independently, so
	\[
	C_T^{\widehat{\mathbb{P}}_n}(x) = \sum_{i=1}^n \bar{W}_i(x)^2 \, \hat{p}(X_i)(1 - \hat{p}(X_i)).
	\]
	
	We now take $\mathbb{E}_Y[\cdot \mid X]$, averaging over the training outcomes that determine $\hat{p}(X_i)$. For each tree $\theta$ and training point $X_i$, the within-node proportion assigns $\hat{p}_\theta(X_i) = \sum_{j} W_j(X_i;\theta) Y_j$, where the sum runs over observations sharing the terminal node of $X_i$. Conditional on the partition induced by $\theta$ and on $X$, this is a sum of independent Bernoulli random variables, so $\mathbb{E}[\hat{p}_\theta(X_i) \mid X, \theta] = \sum_j W_j(X_i;\theta) p(X_j)$. By exchangeability of trees, the forest prediction satisfies
	\[
	\mathbb{E}_Y[\hat{p}(X_i) \mid X] = p_\infty(X_i; X),
	\]
	where $p_\infty(X_i; X) = \mathbb{E}_\theta[\sum_j W_j(X_i;\theta) p(X_j) \mid X]$ is the infinite-forest target at the training point. Under standard regularity conditions on terminal node sizes ensuring that $p(\cdot)$ is approximately constant within each node, $p_\infty(X_i; X) = p(X_i) + o(1)$.
	
	Applying the identity $\mathbb{E}[Z(1-Z)] = \mathbb{E}[Z] - \mathbb{E}[Z^2] = \mathbb{E}[Z] - \Var(Z) - (\mathbb{E}[Z])^2$ with $Z = \hat{p}(X_i)$:
	\begin{align*}
		\mathbb{E}_Y[\hat{p}(X_i)(1 - \hat{p}(X_i)) \mid X]
		&= \mathbb{E}_Y[\hat{p}(X_i) \mid X] - \Var_Y(\hat{p}(X_i) \mid X) - \left(\mathbb{E}_Y[\hat{p}(X_i) \mid X]\right)^2 \\
		&= p_\infty(X_i; X)\bigl(1 - p_\infty(X_i; X)\bigr) - \Var_Y(\hat{p}(X_i) \mid X) \\
		&= p(X_i)(1 - p(X_i)) - \Var_Y(\hat{p}(X_i) \mid X) + o(\Var_Y(\hat{p}(X_i) \mid X)),
	\end{align*}
	where the last equality uses $p_\infty(X_i; X) = p(X_i) + o(1)$ and the smoothness of $p \mapsto p(1-p)$. Substituting into the sum:
	\begin{align*}
		\mathbb{E}_Y\!\left[C_T^{\widehat{\mathbb{P}}_n}(x) \mid X\right]
		&= \sum_{i=1}^n \bar{W}_i(x)^2 \, \mathbb{E}_Y[\hat{p}(X_i)(1 - \hat{p}(X_i)) \mid X] \\
		&= \sum_{i=1}^n \bar{W}_i(x)^2 \, p(X_i)(1-p(X_i)) - \sum_{i=1}^n \bar{W}_i(x)^2 \, \Var_Y(\hat{p}(X_i) \mid X) + o\!\left(\sum_{i=1}^n \bar{W}_i(x)^2 \, \Var_Y(\hat{p}(X_i) \mid X)\right) \\
		&= C_T(x) - \sum_{i=1}^n \bar{W}_i(x)^2 \, \Var_Y(\hat{p}(X_i) \mid X) + o\!\left(\sum_{i=1}^n \bar{W}_i(x)^2 \, \Var_Y(\hat{p}(X_i) \mid X)\right),
	\end{align*}
	which gives part~(ii) after absorbing the lower-order term.
	
	\medskip
	\noindent\textbf{Part (iii).}
	Under standard conditions on terminal node sizes, the forest weights satisfy $\bar{W}_i(x) = O(n^{-1})$ for each $i$, so that $\sum_{i=1}^n \bar{W}_i(x)^2 = O(n^{-1})$. The forest prediction at each training point is an average over $B$ tree-level within-node proportions, each computed from $O(n)$ observations in expectation. By the law of total variance applied across trees and within-node sampling, $\Var_Y(\hat{p}(X_i) \mid X) = O(n^{-1})$ uniformly in $i$. Each summand in the bias term is therefore $O(n^{-2})$, and summing over $n$ terms gives
	\[
	\sum_{i=1}^n \bar{W}_i(x)^2 \, \Var_Y(\hat{p}(X_i) \mid X) = O(n^{-1}) \cdot O(n^{-1}) = O(n^{-2}).
	\]
	Consequently,
	\[
	\mathbb{E}_Y\!\left[C_T^{\widehat{\mathbb{P}}_n}(x) \mid X\right] = C_T(x) + O(n^{-2}). \qedhere
	\]
\end{proof}

\section{Data-generatiion and Simulations}
\label{app:dgm}

The simulation study uses a single data-generating mechanism with two predictor dimensionalities ($p = 10$ and $p = 30$). The $p = 10$ specification defines the core predictor and outcome functions; the $p = 30$ specification adds predictors $x_{13}$--$x_{27}$ and additional terms in the outcome models. For the high-dimensional evaluation ($p = n  = 200$), predictors $x_{28}$--$x_{200}$ are appended as pure noise carrying no signal. All continuous predictors are standardized to zero mean and unit variance after generation, denoted $\mathrm{std}(\cdot)$.

\subsection*{Predictor distribution}

Table~\ref{tab:predictors} specifies the marginal distribution of each predictor. The first 12 predictors form the core and are generated independently. For $p \geq 30$, predictors $x_{13}$, $x_{14}$, and $x_{15}$ are correlated through a shared latent variable $Z \sim \mathcal{N}(0,1)$ with mixing weight $\rho = 0.35$. Each is constructed as
\[
x_j = \mathrm{std}\!\left(g_j\!\left(\sqrt{\rho}\, Z + \sqrt{1 - \rho}\, \epsilon_j\right)\right),
\]
where $\epsilon_j$ is an independent noise term and $g_j$ is a link function, both specified per predictor in Table~\ref{tab:predictors}. Predictors $x_{16}$--$x_{27}$ are generated independently of $Z$. For $p > 27$, predictors $x_{28}, \ldots, x_p$ use the same latent factor construction with $\epsilon_k \sim \mathcal{N}(0,1)$ and $g_k$ equal to the identity, except that every fourth is squared before standardization. These additional predictors enter none of the outcome models.

\begin{table}[ht!]
	\centering
	\small
	\caption{Predictor distributions. All continuous predictors are standardized to zero mean and unit variance. Predictors $x_{13}$--$x_{27}$ are included when $p \geq 30$; predictors $x_{28}$--$x_p$ are included when $p > 27$.}
	\label{tab:predictors}
	\begin{tabular}{cll}
		\hline
		Predictor & Type & Distribution \\
		\hline
		\multicolumn{3}{l}{\textit{Core predictors (all scenarios)}} \\[2pt]
		$x_1$    & Continuous   & $\mathcal{N}(0,\, 1.44)$ \\
		$x_2$    & Continuous   & $1.5 \cdot t_5$ \\
		$x_3$    & Continuous   & $\mathcal{N}(0,\, 1)$ \\
		$x_4$    & Continuous   & $\mathcal{N}(0,\, 0.0625)$ \\
		$x_5$    & Continuous   & $\mathrm{Gamma}(2,\, 0.5)$ \\
		$x_6$    & Binary       & $\mathrm{Bernoulli}(0.4)$ \\
		$x_7$    & Binary       & $\mathrm{Bernoulli}(0.5)$ \\
		$x_8$    & Categorical  & $\{0,1,2\}$;\; $P = (0.6,\, 0.3,\, 0.1)$ \\
		$x_9$    & Continuous   & $\mathcal{N}(0,\, 9)$ \\
		$x_{10}$ & Binary       & $\mathrm{Bernoulli}(0.5)$ \\
		$x_{11}$ & Continuous   & $\mathrm{LogNormal}(0,\, 0.6)$ \\
		$x_{12}$ & Continuous   & $\mathrm{Beta}(2,\, 5)$ \\[4pt]
		\multicolumn{3}{l}{\textit{Correlated continuous predictors ($p \geq 30$; $\rho = 0.35$)}} \\[2pt]
		$x_{13}$ & Continuous   & $\epsilon \sim \mathcal{N}(0,1)$;\; $g = $ identity \\
		$x_{14}$ & Continuous   & $\epsilon \sim t_7$;\; $g = $ identity \\
		$x_{15}$ & Continuous   & $\epsilon \sim \mathcal{N}(0,1)$;\; $g = \mathrm{asinh}$ \\[4pt]
		\multicolumn{3}{l}{\textit{Independent continuous predictors ($p \geq 30$)}} \\[2pt]
		$x_{16}$ & Continuous   & $\mathrm{Gamma}(2.2,\, 0.7)$ \\
		$x_{17}$ & Continuous   & $\mathrm{LogNormal}(0,\, 0.5)$ \\
		$x_{18}$ & Continuous   & $\mathrm{Uniform}(0,\, 1)$ \\[4pt]
		\multicolumn{3}{l}{\textit{Additional discrete predictors ($p \geq 30$)}} \\[2pt]
		$x_{19}$ & Binary       & $\mathrm{Bernoulli}(0.35)$ \\
		$x_{20}$ & Binary       & $\mathrm{Bernoulli}(0.45)$ \\
		$x_{21}$ & Binary       & $\mathrm{Bernoulli}(0.25)$ \\
		$x_{22}$ & Binary       & $\mathrm{Bernoulli}(0.55)$ \\
		$x_{23}$ & Binary       & $\mathrm{Bernoulli}(0.40)$ \\
		$x_{24}$ & Categorical  & $\{0,1,2\}$;\; $P = (0.50,\, 0.35,\, 0.15)$ \\
		$x_{25}$ & Categorical  & $\{0,1,2,3\}$;\; $P = (0.55,\, 0.25,\, 0.15,\, 0.05)$ \\
		$x_{26}$ & Categorical  & $\{0,1,2\}$;\; $P = (0.65,\, 0.25,\, 0.10)$ \\
		$x_{27}$ & Binary       & $\mathrm{Bernoulli}(0.06)$ \\[4pt]
		\multicolumn{3}{l}{\textit{Noise predictors ($p > 27$; no signal)}} \\[2pt]
		$x_{28}, \ldots, x_p$ & Continuous & Same latent factor; $\epsilon_k \sim \mathcal{N}(0,1)$; every 4th squared \\
		\hline
	\end{tabular}
\end{table}

\subsection*{Continuous outcomes}

The conditional mean function for $p = 10$ is
\begin{align*}
	\mu(X) &= 0.90\sin(1.1\,x_1) + 0.35\,x_2 + 0.55\,\tilde{x}_3^2 + 0.18\,x_4 + 0.30\cdot\mathbf{1}(x_5 > 0.4) + 0.22\,x_6 \\
	&\quad + 0.18\cdot\mathbf{1}(x_8{=}1) + 0.28\cdot\mathbf{1}(x_8{=}2) + 0.45\sin(x_{11}) + 0.25\cdot\mathbf{1}(x_{12} > 1.0) \\
	&\quad + 0.18\,(x_1 x_2) + 0.12\bigl(\sin(x_{11})\cdot\mathbf{1}(x_{12} > 1.0)\bigr),
\end{align*}
where $\tilde{x}_3^2 = x_3^2 - \overline{x_3^2}$ denotes the centered squared term. For $p = 30$, the following terms with some degree of association with the outcome are added:
\begin{align*}
	\mu_{30}(X) &= \mu(X) + 0.18\,x_{13} + 0.12\sin(x_{15}) + 0.10\cdot\mathbf{1}(x_{16} > 0) + 0.10\,x_{19} + 0.08\,x_{20} \\
	&\quad + 0.10\cdot\mathbf{1}(x_{24}{=}2) + 0.10\cdot\mathbf{1}(x_{25}{=}1) \\
	&\quad + 0.10\,(x_{13}\,x_{19}) + 0.08\bigl(x_{20}\cdot\mathbf{1}(x_{12} > 1.0)\bigr).
\end{align*}

Continuous outcomes are generated as $Y_i = \mu(X_i) + \sigma(X_i)\,\varepsilon_i$ with $\varepsilon_i \overset{\mathrm{iid}}{\sim} \mathcal{N}(0,1)$, where the error standard deviation $\sigma(X)$ varies with a subset of the covariates to induce heteroscedasticity:
\[
\sigma(X) = \max\!\Big(0.65 + 0.25\,|x_1| + 0.15\,|x_2| + 0.15\cdot\mathbf{1}(x_5 > 0.4) + 0.12\cdot\mathbf{1}(x_{12} > 1.0) + \delta_{30}(X),\;\; 0.15\Big),
\]
where $\delta_{30}(X) = 0.08\,x_{19} + 0.08\cdot\mathbf{1}(x_{24}{=}2) + 0.08\,x_{27}$ for $p = 30$ and $\delta_{30} \equiv 0$ for $p = 10$. The lower bound of $0.15$ prevents the noise variance from collapsing near zero. All covariates $x_p$ not specifically defined in the conditional mean function above were included as pure noise with no signal. 

\subsection*{Binary outcomes}
Binary outcomes use a logistic link with a linear predictor that involves the same covariates as the continuous specification but with distinct coefficients and additional terms. The $p = 10$ specification is
\begin{align*}
	\eta_0(X) &= 0.55\,x_1 + 0.35\,x_2 + 0.45\,\tilde{x}_3^2 + 0.20\,x_4 + 0.35\,x_5 + 0.25\,x_6 + 0.15\,x_7 \\
	&\quad + 0.18\cdot\mathbf{1}(x_8{=}1) + 0.28\cdot\mathbf{1}(x_8{=}2) + 0.10\,x_9 + 0.12\,x_{10} \\
	&\quad + 0.35\sin(x_{11}) + 0.22\cdot\mathbf{1}(x_{12} > 1.1) + 0.12\,\tilde{x}_2^3 \\
	&\quad + 0.18\,(x_1 x_2) + 0.18\,(x_1 x_6) + 0.15\bigl(x_3\cdot\mathbf{1}(x_8{=}2)\bigr) \\
	&\quad + 0.12\bigl(\sin(x_{11})\cdot\mathbf{1}(x_{12} > 1.1)\bigr),
\end{align*}
where $\tilde{x}_2^3 = x_2^3 - 3x_2$ is the centered cubic. For $p = 30$, the following terms are added:
\begin{align*}
	\eta_{0,30}(X) &= \eta_0(X) + 0.18\,x_{13} + 0.12\,x_{14} + 0.10\sin(x_{15}) + 0.14\cdot\mathbf{1}(x_{16} > 0) \\
	&\quad + 0.10\,x_{17} + 0.08\cdot\mathbf{1}(x_{18} > 0.5) \\
	&\quad + 0.12\,x_{19} + 0.10\,x_{20} - 0.08\,x_{21} + 0.10\,x_{22} + 0.08\,x_{23} \\
	&\quad + 0.10\cdot\mathbf{1}(x_{24}{=}2) + 0.10\cdot\mathbf{1}(x_{25}{=}3) + 0.08\cdot\mathbf{1}(x_{26}{=}1) + 0.10\,x_{27} \\
	&\quad + 0.12\,(x_{13}\,x_{19}) + 0.10\bigl(x_{14}\cdot\mathbf{1}(x_{24}{=}1)\bigr) + 0.10\bigl(x_{20}\cdot\mathbf{1}(x_{12} > 1.1)\bigr).
\end{align*}

The intercept $a_0$ is calibrated at the realized covariate matrix so that the marginal prevalence equals $40\%$ by solving
\[
\frac{1}{n}\sum_{i=1}^n \mathrm{logit}^{-1}\!\bigl(a_0 + \eta_0(X_i)\bigr) = 0.40.
\]
Binary outcomes are then drawn as $Y_i \sim \mathrm{Bernoulli}\!\bigl(\mathrm{logit}^{-1}(a_0 + \eta_0(X_i))\bigr)$.

\subsection*{Design rationale}

The data-generating mechanism was designed to present a realistic challenge for the model, with varying types of predictors and non-linear relationships with the outcome, while allowing exact computation of true conditional quantities. Key features include: (i)~mixed predictor types spanning continuous, binary, and multinomial variables with varied marginal shapes; (ii)~nonlinear main effects (sine, threshold, polynomial) and two-way interactions involving both continuous and categorical predictors; (iii)~heteroscedastic errors for continuous outcomes, with the variance depending on a subset of the mean-function predictors; (iv)~a binary outcome model that shares covariate involvement with the continuous specification but uses distinct coefficients and additional predictors ($x_7$, $x_9$, $x_{10}$, $x_{14}$, $x_{17}$, $x_{21}$--$x_{23}$); and (v)~correlated predictors for $p \geq 30$ via the shared latent variable $Z$.

\subsection*{Test point construction}

Test points are constructed as an anchored jittered cloud: $n_{\mathrm{test}} = 400$ training observations are sampled with replacement, and Gaussian noise with $\sigma_{\mathrm{jitter}} = 0.02$ is added to each continuous predictor. Discrete predictors are left unchanged. This ensures test points lie within the support of the training distribution while avoiding exact duplication of training rows.

\newpage 

\section{Simulation Results}

\subsection{Full simulation results: Bias}
\label{app:tables}

Tables~\ref{tab:continuous} and~\ref{tab:binary} report the estimation diagnostics for all 36 standard scenarios and the high-dimensional evaluation ($n = p = 200$). For each scenario, the pointwise regression $\widehat C_T(x) = \hat\alpha + \hat\beta\, C_T(x)$ is fit across $n_{\mathrm{test}} = 400$ test points, where $\widehat C_T(x)$ is the mean PASR estimate across $S = 50$ independent datasets. Mean and median summaries are computed across test points and datasets. The IQR column reports the interquartile range of the pointwise bias $(\widehat C_T(x) - C_T(x))$ as $[Q_1,\, Q_3]$. Sampling schemes are denoted Boot (bootstrap), 50\% (subsampling without replacement at rate 0.50), and 80\% (subsampling at rate 0.80).

\begin{table}[htbp!]
	\centering
	\footnotesize
	\caption{Estimation diagnostics: continuous outcomes. Median bias is substantially smaller than mean bias across all scenarios, indicating that conservative overestimation is concentrated at test points with large $C_T(x)$. The estimator is uniformly conservative ($\overline{\widehat C_T} > \overline{C_T}$) across all scenarios, consistent with Proposition~\ref{prop:conservative}.}
	\label{tab:continuous}
	\begin{tabular}{rlrrrrrrrrrl}
		\hline
		$q$ & Samp.\ & Slope & Int.\ & $r$ & $\overline{C_T}$ & $\overline{\widehat C_T}$ & Mean bias & Med.\ $C_T$ & Med.\ $\widehat C_T$ & Med.\ bias & IQR \\
		\hline
		\multicolumn{12}{l}{\textit{$n = 200$, $p = 10$}} \\[2pt]
		1  & Boot & 1.31 & 0.011    & 0.78 & 0.104 & 0.147 & 0.044 & 0.084 & 0.109 & 0.026 & [0.006, 0.052] \\
		1  & 50\% & 1.38 & 0.002    & 0.79 & 0.066 & 0.093 & 0.027 & 0.054 & 0.069 & 0.013 & [0.003, 0.035] \\
		1  & 80\% & 1.37 & 0.007    & 0.79 & 0.121 & 0.172 & 0.051 & 0.096 & 0.121 & 0.021 & [0.004, 0.054] \\
		4  & Boot & 1.11 & 0.076    & 0.72 & 0.305 & 0.414 & 0.109 & 0.261 & 0.339 & 0.067 & [0.016, 0.146] \\
		4  & 50\% & 1.23 & 0.024    & 0.73 & 0.183 & 0.248 & 0.065 & 0.159 & 0.200 & 0.037 & [0.007, 0.090] \\
		4  & 80\% & 1.16 & 0.092    & 0.66 & 0.399 & 0.554 & 0.155 & 0.337 & 0.433 & 0.082 & [0.014, 0.195] \\
		10 & Boot & 1.16 & 0.073    & 0.72 & 0.343 & 0.472 & 0.129 & 0.290 & 0.374 & 0.084 & [0.023, 0.158] \\
		10 & 50\% & 1.19 & 0.032    & 0.72 & 0.222 & 0.296 & 0.075 & 0.192 & 0.237 & 0.041 & [0.009, 0.096] \\
		10 & 80\% & 1.16 & 0.105    & 0.68 & 0.476 & 0.655 & 0.179 & 0.402 & 0.501 & 0.088 & [0.014, 0.212] \\[4pt]
		\multicolumn{12}{l}{\textit{$n = 400$, $p = 10$}} \\[2pt]
		1  & Boot & 1.49 & $-$0.010 & 0.87 & 0.071 & 0.095 & 0.025 & 0.062 & 0.074 & 0.011 & [0.002, 0.034] \\
		1  & 50\% & 1.57 & $-$0.010 & 0.89 & 0.046 & 0.061 & 0.016 & 0.041 & 0.050 & 0.007 & [0.001, 0.022] \\
		1  & 80\% & 1.63 & $-$0.023 & 0.88 & 0.082 & 0.111 & 0.029 & 0.070 & 0.081 & 0.009 & [$-$0.000, 0.039] \\
		4  & Boot & 1.24 & 0.019    & 0.80 & 0.285 & 0.372 & 0.087 & 0.258 & 0.320 & 0.052 & [0.012, 0.127] \\
		4  & 50\% & 1.26 & 0.008    & 0.80 & 0.168 & 0.219 & 0.051 & 0.150 & 0.190 & 0.031 & [0.006, 0.077] \\
		4  & 80\% & 1.26 & 0.012    & 0.79 & 0.373 & 0.481 & 0.108 & 0.331 & 0.408 & 0.052 & [0.008, 0.172] \\
		10 & Boot & 1.29 & 0.003    & 0.81 & 0.321 & 0.418 & 0.097 & 0.297 & 0.355 & 0.060 & [0.011, 0.140] \\
		10 & 50\% & 1.32 & $-$0.007 & 0.83 & 0.204 & 0.264 & 0.060 & 0.186 & 0.228 & 0.035 & [0.005, 0.090] \\
		10 & 80\% & 1.32 & $-$0.011 & 0.80 & 0.450 & 0.585 & 0.135 & 0.414 & 0.503 & 0.070 & [0.003, 0.200] \\[4pt]
		\multicolumn{12}{l}{\textit{$n = 200$, $p = 30$}} \\[2pt]
		1  & Boot & 0.50 & 0.138    & 0.40 & 0.141 & 0.209 & 0.068 & 0.127 & 0.184 & 0.054 & [0.013, 0.107] \\
		1  & 50\% & 0.62 & 0.065    & 0.44 & 0.073 & 0.111 & 0.038 & 0.068 & 0.095 & 0.030 & [0.010, 0.056] \\
		1  & 80\% & 0.67 & 0.129    & 0.47 & 0.160 & 0.235 & 0.076 & 0.145 & 0.199 & 0.058 & [0.012, 0.115] \\
		6  & Boot & 0.34 & 0.372    & 0.30 & 0.335 & 0.486 & 0.151 & 0.316 & 0.434 & 0.131 & [0.029, 0.248] \\
		6  & 50\% & 0.43 & 0.200    & 0.34 & 0.193 & 0.282 & 0.089 & 0.181 & 0.257 & 0.075 & [0.019, 0.138] \\
		6  & 80\% & 0.40 & 0.453    & 0.33 & 0.453 & 0.636 & 0.183 & 0.423 & 0.584 & 0.146 & [0.015, 0.300] \\
		30 & Boot & 0.19 & 0.475    & 0.24 & 0.394 & 0.549 & 0.156 & 0.355 & 0.493 & 0.156 & [0.027, 0.270] \\
		30 & 50\% & 0.19 & 0.299    & 0.26 & 0.254 & 0.347 & 0.093 & 0.231 & 0.307 & 0.088 & [0.010, 0.165] \\
		30 & 80\% & 0.14 & 0.714    & 0.20 & 0.579 & 0.794 & 0.215 & 0.519 & 0.692 & 0.198 & [0.034, 0.382] \\[4pt]
		\multicolumn{12}{l}{\textit{$n = 400$, $p = 30$}} \\[2pt]
		1  & Boot & 1.33 & 0.018    & 0.73 & 0.126 & 0.186 & 0.060 & 0.115 & 0.158 & 0.038 & [0.004, 0.096] \\
		1  & 50\% & 1.42 & 0.004    & 0.75 & 0.068 & 0.100 & 0.032 & 0.062 & 0.083 & 0.022 & [0.004, 0.051] \\
		1  & 80\% & 1.53 & $-$0.004 & 0.75 & 0.145 & 0.219 & 0.074 & 0.132 & 0.185 & 0.045 & [0.006, 0.110] \\
		6  & Boot & 1.08 & 0.129    & 0.64 & 0.340 & 0.494 & 0.154 & 0.315 & 0.442 & 0.119 & [0.021, 0.248] \\
		6  & 50\% & 1.12 & 0.066    & 0.63 & 0.197 & 0.287 & 0.090 & 0.184 & 0.256 & 0.068 & [0.013, 0.150] \\
		6  & 80\% & 1.19 & 0.124    & 0.65 & 0.464 & 0.674 & 0.209 & 0.434 & 0.598 & 0.150 & [0.032, 0.351] \\
		30 & Boot & 1.07 & 0.144    & 0.66 & 0.385 & 0.557 & 0.172 & 0.354 & 0.506 & 0.132 & [0.035, 0.269] \\
		30 & 50\% & 1.09 & 0.081    & 0.66 & 0.247 & 0.351 & 0.104 & 0.227 & 0.312 & 0.080 & [0.016, 0.174] \\
		30 & 80\% & 1.10 & 0.197    & 0.63 & 0.561 & 0.814 & 0.254 & 0.516 & 0.733 & 0.189 & [0.039, 0.398] \\[4pt]
		\multicolumn{12}{l}{\textit{$n = p = 200$ (high-dimensional)}} \\[2pt]
		15 & Boot & 0.22 & 0.583    & 0.13 & 0.384 & 0.668 & 0.284 & 0.357 & 0.619 & 0.223 & [0.154, 0.343] \\
		\hline
	\end{tabular}
\end{table}

\newpage

\begin{table}[htbp!]
	\centering
	\footnotesize
	\caption{Estimation diagnostics: binary outcomes. Near-zero mean and median bias across all scenarios confirms the absence of the nuisance gap for probability forests.}
	\label{tab:binary}
	\begin{tabular}{rlrrrrrrrrrl}
		\hline
		$q$ & Samp.\ & Slope & Int.\ & $r$ & $\overline{C_T}$ & $\overline{\widehat C_T}$ & Mean bias & Med.\ $C_T$ & Med.\ $\widehat C_T$ & Med.\ bias & IQR \\
		\hline
		\multicolumn{12}{l}{\textit{$n = 200$, $p = 10$}} \\[2pt]
		1  & Boot & 0.58 & 0.008 & 0.74 & 0.018 & 0.018 & $<$0.001  & 0.018 & 0.018 & $-$0.001 & [$-$0.002, 0.001] \\
		1  & 50\% & 0.54 & 0.006 & 0.72 & 0.012 & 0.012 & $<$0.001  & 0.012 & 0.012 & $<$0.001 & [$-$0.001, 0.001] \\
		1  & 80\% & 0.58 & 0.007 & 0.77 & 0.021 & 0.019 & $-$0.002  & 0.021 & 0.019 & $-$0.003 & [$-$0.004, $-$0.001] \\
		4  & Boot & 0.48 & 0.027 & 0.72 & 0.052 & 0.052 & $<$0.001  & 0.055 & 0.054 & $-$0.003 & [$-$0.007, 0.002] \\
		4  & 50\% & 0.49 & 0.016 & 0.75 & 0.032 & 0.032 & $<$0.001  & 0.034 & 0.034 & $-$0.001 & [$-$0.004, 0.002] \\
		4  & 80\% & 0.48 & 0.030 & 0.75 & 0.071 & 0.064 & $-$0.007  & 0.075 & 0.067 & $-$0.010 & [$-$0.017, $-$0.002] \\
		10 & Boot & 0.50 & 0.029 & 0.73 & 0.057 & 0.058 & $<$0.001  & 0.062 & 0.060 & $-$0.003 & [$-$0.008, 0.002] \\
		10 & 50\% & 0.51 & 0.019 & 0.76 & 0.037 & 0.037 & 0.001     & 0.039 & 0.039 & $-$0.001 & [$-$0.004, 0.003] \\
		10 & 80\% & 0.50 & 0.034 & 0.76 & 0.081 & 0.074 & $-$0.007  & 0.087 & 0.078 & $-$0.011 & [$-$0.018, $-$0.002] \\[4pt]
		\multicolumn{12}{l}{\textit{$n = 400$, $p = 10$}} \\[2pt]
		1  & Boot & 0.76 & 0.003 & 0.85 & 0.012 & 0.012 & $<$0.001  & 0.012 & 0.013 & $<$0.001 & [$-$0.001, 0.001] \\
		1  & 50\% & 0.73 & 0.003 & 0.84 & 0.008 & 0.008 & $<$0.001  & 0.008 & 0.008 & $<$0.001 & [$-$0.001, 0.001] \\
		1  & 80\% & 0.72 & 0.003 & 0.87 & 0.014 & 0.013 & $-$0.001  & 0.013 & 0.013 & $-$0.001 & [$-$0.002, $<$0.001] \\
		4  & Boot & 0.58 & 0.021 & 0.76 & 0.047 & 0.048 & 0.001     & 0.050 & 0.051 & $-$0.002 & [$-$0.006, 0.004] \\
		4  & 50\% & 0.60 & 0.012 & 0.79 & 0.029 & 0.029 & 0.001     & 0.031 & 0.032 & $-$0.001 & [$-$0.003, 0.002] \\
		4  & 80\% & 0.57 & 0.023 & 0.79 & 0.064 & 0.059 & $-$0.005  & 0.068 & 0.063 & $-$0.008 & [$-$0.014, $<$0.001] \\
		10 & Boot & 0.59 & 0.023 & 0.76 & 0.052 & 0.053 & 0.001     & 0.055 & 0.057 & $-$0.002 & [$-$0.006, 0.005] \\
		10 & 50\% & 0.61 & 0.014 & 0.80 & 0.033 & 0.034 & 0.001     & 0.035 & 0.037 & $-$0.001 & [$-$0.004, 0.003] \\
		10 & 80\% & 0.57 & 0.027 & 0.79 & 0.074 & 0.069 & $-$0.004  & 0.078 & 0.074 & $-$0.008 & [$-$0.014, 0.001] \\[4pt]
		\multicolumn{12}{l}{\textit{$n = 200$, $p = 30$}} \\[2pt]
		1  & Boot & 0.35 & 0.018 & 0.56 & 0.023 & 0.026 & 0.003     & 0.025 & 0.026 & 0.001 & [$-$0.002, 0.005] \\
		1  & 50\% & 0.33 & 0.010 & 0.57 & 0.013 & 0.014 & 0.002     & 0.014 & 0.015 & 0.001 & [$-$0.001, 0.003] \\
		1  & 80\% & 0.38 & 0.018 & 0.59 & 0.027 & 0.028 & 0.001     & 0.030 & 0.029 & $-$0.001 & [$-$0.004, 0.004] \\
		6  & Boot & 0.31 & 0.042 & 0.58 & 0.052 & 0.058 & 0.006     & 0.058 & 0.061 & 0.003 & [$-$0.004, 0.011] \\
		6  & 50\% & 0.32 & 0.025 & 0.63 & 0.031 & 0.035 & 0.004     & 0.035 & 0.037 & 0.002 & [$-$0.002, 0.008] \\
		6  & 80\% & 0.35 & 0.050 & 0.62 & 0.073 & 0.076 & 0.002     & 0.082 & 0.079 & $-$0.002 & [$-$0.011, 0.010] \\
		30 & Boot & 0.33 & 0.047 & 0.60 & 0.059 & 0.066 & 0.007     & 0.066 & 0.069 & 0.003 & [$-$0.004, 0.013] \\
		30 & 50\% & 0.34 & 0.030 & 0.66 & 0.038 & 0.043 & 0.005     & 0.042 & 0.045 & 0.003 & [$-$0.002, 0.010] \\
		30 & 80\% & 0.36 & 0.059 & 0.64 & 0.088 & 0.091 & 0.003     & 0.099 & 0.096 & $-$0.002 & [$-$0.013, 0.013] \\[4pt]
		\multicolumn{12}{l}{\textit{$n = 400$, $p = 30$}} \\[2pt]
		1  & Boot & 0.36 & 0.015 & 0.50 & 0.019 & 0.022 & 0.002     & 0.020 & 0.022 & $<$0.001 & [$-$0.002, 0.004] \\
		1  & 50\% & 0.35 & 0.008 & 0.50 & 0.011 & 0.012 & 0.001     & 0.011 & 0.012 & $<$0.001 & [$-$0.001, 0.002] \\
		1  & 80\% & 0.39 & 0.015 & 0.53 & 0.023 & 0.024 & 0.001     & 0.024 & 0.024 & $-$0.001 & [$-$0.003, 0.003] \\
		6  & Boot & 0.29 & 0.042 & 0.50 & 0.051 & 0.057 & 0.006     & 0.054 & 0.059 & 0.001 & [$-$0.004, 0.011] \\
		6  & 50\% & 0.31 & 0.024 & 0.56 & 0.030 & 0.034 & 0.003     & 0.032 & 0.035 & 0.001 & [$-$0.002, 0.006] \\
		6  & 80\% & 0.32 & 0.052 & 0.54 & 0.072 & 0.075 & 0.002     & 0.077 & 0.078 & $-$0.003 & [$-$0.011, 0.010] \\
		30 & Boot & 0.32 & 0.046 & 0.53 & 0.058 & 0.064 & 0.006     & 0.061 & 0.067 & 0.001 & [$-$0.005, 0.012] \\
		30 & 50\% & 0.33 & 0.028 & 0.59 & 0.037 & 0.041 & 0.004     & 0.038 & 0.043 & 0.001 & [$-$0.003, 0.008] \\
		30 & 80\% & 0.33 & 0.061 & 0.56 & 0.087 & 0.089 & 0.003     & 0.091 & 0.095 & $-$0.004 & [$-$0.014, 0.013] \\[4pt]
		\multicolumn{12}{l}{\textit{$n = p = 200$ (high-dimensional)}} \\[2pt]
		15 & Boot & 0.09 & 0.067 & 0.40 & 0.059 & 0.072 & 0.013     & 0.065 & 0.073 & 0.008 & [$-$0.001, 0.023] \\
		\hline
	\end{tabular}
\end{table}

\newpage

\subsection{Simulation results: Single-run stability}
\label{app:stability}

Tables~\ref{tab:variability_cont} and~\ref{tab:variability_bin} report the variability of the PASR estimator across $S = 50$ independent datasets for each scenario. In practice, a user applies the estimator once to a single dataset. These tables characterize how much a single-run estimate $\widehat C_T(x)$ would vary if the user had drawn a different training sample $Y$ from the same population.

For binary outcomes, single-run standard deviations are an order of magnitude smaller than the estimated floor across all scenarios (e.g., median SD of $0.007$ against median $\widehat C_T$ of $0.059$ for $n = 200$, $p = 30$, $q = 6$, bootstrap), indicating that individual runs are highly precise. For continuous outcomes, single-run variability is larger in absolute terms but remains moderate relative to the estimated floor, with median SD roughly 35--55\% of the median $\widehat C_T$ in the favorable regime ($n = 400$, $p = 10$) and 40--55\% in the harder regime ($n = 400$, $p = 30$).

Figure~\ref{fig:stability} examines the structure of this variability for continuous outcomes by plotting the pointwise bias $\widehat C_T(x) - C_T(x)$ for each of the $S = 50$ runs as a function of the true floor, alongside the distribution of pairwise correlations of bias across test points. Two sources of single-run error can be distinguished: a \textit{systematic} component, in which the nuisance estimate $\hat\sigma^2(x)$ for a given run shifts all test-point estimates in the same direction, and a \textit{pointwise} component, in which the synthetic resampling introduces noise that is approximately independent across test points.

In the favorable regime ($n = 400$, $p = 10$, $q = 4$), the mean pairwise correlation across test points is $0.26$, indicating that the pointwise component dominates: runs that overestimate at one covariate configuration do not strongly predict overestimation at others. In the harder regime ($n = 200$, $p = 30$, $q = 6$), the mean pairwise correlation rises to $0.46$, reflecting a larger systematic contribution from nuisance estimation noise. With only 100 observations per cross-fitting fold and 30 predictors, the cross-fitted residual product $\hat\sigma^2(x)$ has substantial variance that propagates uniformly across the covariate space. Crucially, this systematic component is uniformly conservative---runs shift upward rather than crossing zero---so the bias direction established in the main text holds at the single-run level, not only in expectation across datasets.

These results have two practical implications. First, in settings where $n/p$ is moderate or large, the PASR estimator's single-run variability is primarily pointwise, meaning that local estimates are noisy but the overall level is stable. Second, in harder settings where nuisance estimation dominates the variability, the conservative direction is preserved, so the estimator does not produce overconfident (anticonservative) variance estimates even in unfavorable single runs. The near-nominal coverage reported in Section~\ref{subsec:sim_coverage} confirms that this single-run stability translates into reliable interval performance in practice.

\begin{figure}[htbp!]
	\centering
	\includegraphics[width=\textwidth]{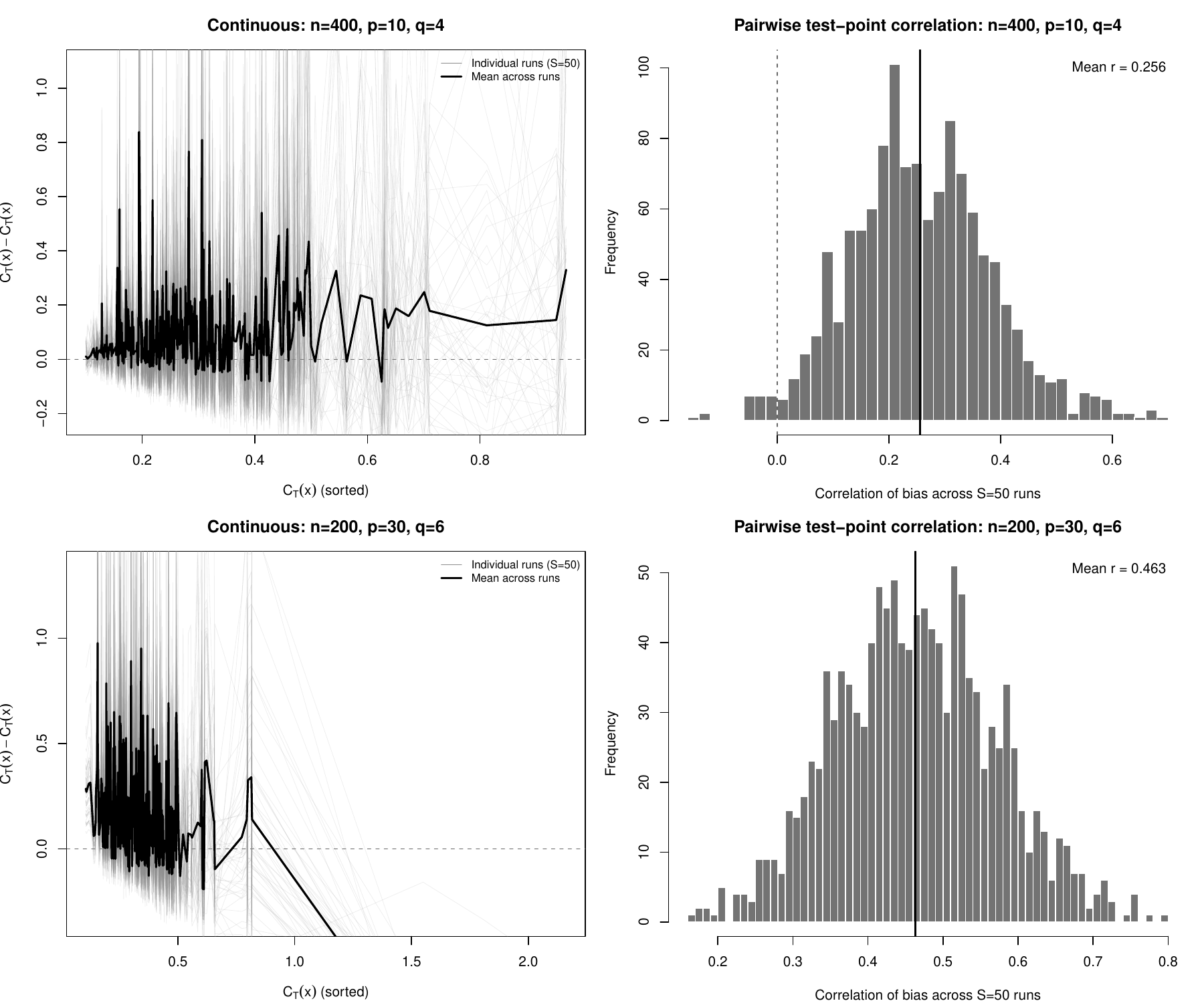}
	\caption{Single-run bias structure for continuous outcomes. Left column: pointwise bias $\widehat C_T(x) - C_T(x)$ for each of $S = 50$ independent datasets (gray lines) and the mean across datasets (black line), plotted against the true floor sorted in ascending order. Right column: distribution of pairwise correlations of pointwise bias across $n_{\mathrm{test}} = 400$ test points, computed across the $S = 50$ runs. Top row: $n = 400$, $p = 10$, $q = 4$ (bootstrap); bottom row: $n = 200$, $p = 30$, $q = 6$ (bootstrap). In the favorable regime, the mean pairwise correlation is $0.26$, indicating that most variability is pointwise. In the harder regime, the mean rises to $0.46$, reflecting a larger systematic contribution from nuisance estimation.}
	\label{fig:stability}
\end{figure}

\newpage
\begin{table}[htbp!]
	\centering
	\footnotesize
	\caption{Single-run variability of the empirical estimator: continuous outcomes. For each test point, the standard deviation and interquartile range of $\widehat C_T(x)$ are computed across $S = 50$ independent datasets; mean and median summaries are then taken across test points. Scale reference columns ($\overline{\widehat C_T}$ and Med.\ $\widehat C_T$) are reproduced from Table~\ref{tab:continuous}.}
	\label{tab:variability_cont}
	\begin{tabular}{rlrrrrrr}
		\hline
		$q$ & Samp.\ & $\overline{\widehat C_T}$ & Med.\ $\widehat C_T$ & Mean SD & Med.\ SD & Mean IQR & Med.\ IQR \\
		\hline
		\multicolumn{8}{l}{\textit{$n = 200$, $p = 10$}} \\[2pt]
		1  & Boot & 0.147 & 0.109 & 0.078 & 0.049 & 0.081 & 0.059 \\
		1  & 50\% & 0.093 & 0.069 & 0.041 & 0.027 & 0.047 & 0.032 \\
		1  & 80\% & 0.172 & 0.121 & 0.114 & 0.067 & 0.107 & 0.066 \\
		4  & Boot & 0.414 & 0.339 & 0.205 & 0.149 & 0.219 & 0.170 \\
		4  & 50\% & 0.248 & 0.200 & 0.103 & 0.071 & 0.117 & 0.084 \\
		4  & 80\% & 0.554 & 0.433 & 0.321 & 0.224 & 0.341 & 0.247 \\
		10 & Boot & 0.472 & 0.374 & 0.228 & 0.165 & 0.254 & 0.194 \\
		10 & 50\% & 0.296 & 0.237 & 0.127 & 0.086 & 0.141 & 0.102 \\
		10 & 80\% & 0.655 & 0.501 & 0.397 & 0.271 & 0.417 & 0.307 \\[4pt]
		\multicolumn{8}{l}{\textit{$n = 400$, $p = 10$}} \\[2pt]
		1  & Boot & 0.095 & 0.074 & 0.047 & 0.031 & 0.049 & 0.037 \\
		1  & 50\% & 0.061 & 0.050 & 0.026 & 0.017 & 0.028 & 0.020 \\
		1  & 80\% & 0.111 & 0.081 & 0.068 & 0.041 & 0.067 & 0.044 \\
		4  & Boot & 0.372 & 0.320 & 0.168 & 0.127 & 0.185 & 0.147 \\
		4  & 50\% & 0.219 & 0.190 & 0.087 & 0.067 & 0.098 & 0.076 \\
		4  & 80\% & 0.481 & 0.408 & 0.262 & 0.204 & 0.274 & 0.217 \\
		10 & Boot & 0.418 & 0.355 & 0.186 & 0.140 & 0.205 & 0.161 \\
		10 & 50\% & 0.264 & 0.228 & 0.103 & 0.077 & 0.116 & 0.088 \\
		10 & 80\% & 0.585 & 0.503 & 0.315 & 0.226 & 0.326 & 0.265 \\[4pt]
		\multicolumn{8}{l}{\textit{$n = 200$, $p = 30$}} \\[2pt]
		1  & Boot & 0.209 & 0.184 & 0.100 & 0.078 & 0.107 & 0.090 \\
		1  & 50\% & 0.111 & 0.095 & 0.046 & 0.037 & 0.049 & 0.040 \\
		1  & 80\% & 0.235 & 0.199 & 0.137 & 0.113 & 0.139 & 0.114 \\
		6  & Boot & 0.486 & 0.434 & 0.214 & 0.175 & 0.226 & 0.191 \\
		6  & 50\% & 0.282 & 0.257 & 0.113 & 0.093 & 0.128 & 0.109 \\
		6  & 80\% & 0.636 & 0.584 & 0.353 & 0.284 & 0.364 & 0.328 \\
		30 & Boot & 0.549 & 0.493 & 0.250 & 0.208 & 0.259 & 0.222 \\
		30 & 50\% & 0.347 & 0.307 & 0.139 & 0.113 & 0.153 & 0.129 \\
		30 & 80\% & 0.794 & 0.692 & 0.472 & 0.361 & 0.429 & 0.366 \\[4pt]
		\multicolumn{8}{l}{\textit{$n = 400$, $p = 30$}} \\[2pt]
		1  & Boot & 0.186 & 0.158 & 0.085 & 0.060 & 0.089 & 0.066 \\
		1  & 50\% & 0.100 & 0.083 & 0.041 & 0.029 & 0.044 & 0.034 \\
		1  & 80\% & 0.219 & 0.185 & 0.125 & 0.088 & 0.127 & 0.093 \\
		6  & Boot & 0.494 & 0.442 & 0.207 & 0.154 & 0.219 & 0.174 \\
		6  & 50\% & 0.287 & 0.256 & 0.110 & 0.086 & 0.123 & 0.103 \\
		6  & 80\% & 0.674 & 0.598 & 0.350 & 0.277 & 0.354 & 0.287 \\
		30 & Boot & 0.557 & 0.506 & 0.243 & 0.182 & 0.250 & 0.205 \\
		30 & 50\% & 0.351 & 0.312 & 0.134 & 0.102 & 0.148 & 0.124 \\
		30 & 80\% & 0.814 & 0.733 & 0.447 & 0.342 & 0.418 & 0.359 \\[4pt]
		\multicolumn{8}{l}{\textit{$n = p = 200$ (high-dimensional)}} \\[2pt]
		15 & Boot & 0.668 & 0.619 & 0.305 & 0.246 & 0.307 & 0.267 \\
		\hline
	\end{tabular}
\end{table}

\newpage

\begin{table}[htbp!]
	\centering
	\footnotesize
	\caption{Single-run variability of the empirical estimator: binary outcomes. Layout follows Table~\ref{tab:variability_cont}. Single-run standard deviations are an order of magnitude smaller than the estimated floor, indicating that individual runs are highly precise.}
	\label{tab:variability_bin}
	\begin{tabular}{rlrrrrrr}
		\hline
		$q$ & Samp.\ & $\overline{\widehat C_T}$ & Med.\ $\widehat C_T$ & Mean SD & Med.\ SD & Mean IQR & Med.\ IQR \\
		\hline
		\multicolumn{8}{l}{\textit{$n = 200$, $p = 10$}} \\[2pt]
		1  & Boot & 0.018 & 0.018 & 0.004 & 0.004 & 0.006 & 0.005 \\
		1  & 50\% & 0.012 & 0.012 & 0.003 & 0.003 & 0.003 & 0.003 \\
		1  & 80\% & 0.019 & 0.019 & 0.006 & 0.006 & 0.009 & 0.009 \\
		4  & Boot & 0.052 & 0.054 & 0.013 & 0.013 & 0.017 & 0.016 \\
		4  & 50\% & 0.032 & 0.034 & 0.007 & 0.007 & 0.010 & 0.009 \\
		4  & 80\% & 0.064 & 0.067 & 0.022 & 0.023 & 0.031 & 0.030 \\
		10 & Boot & 0.058 & 0.060 & 0.014 & 0.014 & 0.018 & 0.018 \\
		10 & 50\% & 0.037 & 0.039 & 0.009 & 0.009 & 0.011 & 0.010 \\
		10 & 80\% & 0.074 & 0.078 & 0.026 & 0.026 & 0.036 & 0.036 \\[4pt]
		\multicolumn{8}{l}{\textit{$n = 400$, $p = 10$}} \\[2pt]
		1  & Boot & 0.012 & 0.013 & 0.003 & 0.002 & 0.003 & 0.003 \\
		1  & 50\% & 0.008 & 0.008 & 0.002 & 0.002 & 0.002 & 0.002 \\
		1  & 80\% & 0.013 & 0.013 & 0.004 & 0.004 & 0.005 & 0.005 \\
		4  & Boot & 0.048 & 0.051 & 0.011 & 0.011 & 0.015 & 0.015 \\
		4  & 50\% & 0.029 & 0.032 & 0.006 & 0.007 & 0.008 & 0.008 \\
		4  & 80\% & 0.059 & 0.063 & 0.020 & 0.020 & 0.028 & 0.028 \\
		10 & Boot & 0.053 & 0.057 & 0.012 & 0.012 & 0.016 & 0.016 \\
		10 & 50\% & 0.034 & 0.037 & 0.007 & 0.008 & 0.010 & 0.010 \\
		10 & 80\% & 0.069 & 0.074 & 0.023 & 0.023 & 0.032 & 0.031 \\[4pt]
		\multicolumn{8}{l}{\textit{$n = 200$, $p = 30$}} \\[2pt]
		1  & Boot & 0.026 & 0.026 & 0.004 & 0.004 & 0.005 & 0.005 \\
		1  & 50\% & 0.014 & 0.015 & 0.002 & 0.002 & 0.003 & 0.002 \\
		1  & 80\% & 0.028 & 0.029 & 0.007 & 0.007 & 0.009 & 0.009 \\
		6  & Boot & 0.058 & 0.061 & 0.010 & 0.010 & 0.012 & 0.010 \\
		6  & 50\% & 0.035 & 0.037 & 0.006 & 0.005 & 0.007 & 0.006 \\
		6  & 80\% & 0.076 & 0.079 & 0.020 & 0.020 & 0.026 & 0.025 \\
		30 & Boot & 0.066 & 0.069 & 0.011 & 0.011 & 0.014 & 0.012 \\
		30 & 50\% & 0.043 & 0.045 & 0.007 & 0.007 & 0.009 & 0.008 \\
		30 & 80\% & 0.091 & 0.096 & 0.024 & 0.024 & 0.031 & 0.030 \\[4pt]
		\multicolumn{8}{l}{\textit{$n = 400$, $p = 30$}} \\[2pt]
		1  & Boot & 0.022 & 0.022 & 0.004 & 0.003 & 0.004 & 0.004 \\
		1  & 50\% & 0.012 & 0.012 & 0.002 & 0.002 & 0.002 & 0.002 \\
		1  & 80\% & 0.024 & 0.024 & 0.006 & 0.006 & 0.007 & 0.007 \\
		6  & Boot & 0.057 & 0.059 & 0.009 & 0.009 & 0.012 & 0.010 \\
		6  & 50\% & 0.034 & 0.035 & 0.005 & 0.005 & 0.007 & 0.006 \\
		6  & 80\% & 0.075 & 0.078 & 0.019 & 0.018 & 0.024 & 0.023 \\
		30 & Boot & 0.064 & 0.067 & 0.011 & 0.011 & 0.013 & 0.012 \\
		30 & 50\% & 0.041 & 0.043 & 0.007 & 0.007 & 0.008 & 0.008 \\
		30 & 80\% & 0.089 & 0.095 & 0.022 & 0.022 & 0.029 & 0.028 \\[4pt]
		\multicolumn{8}{l}{\textit{$n = p = 200$ (high-dimensional)}} \\[2pt]
		15 & Boot & 0.072 & 0.073 & 0.007 & 0.007 & 0.009 & 0.009 \\
		\hline
	\end{tabular}
\end{table}

\newpage 

\section{Extensions to Other Forest Procedures}
\label{app:extensions}

The variance identity in Theorem~\ref{thm:finite_var} and the covariance decomposition in Section~\ref{sec:covariance_decomposition} are derived from two structural properties: exchangeability of the tree-generating mechanism across $\theta_1,\ldots,\theta_B$, and the representation of each tree prediction as a function of the realized outcomes determined by a randomized averaging rule. These properties are shared by a broad class of tree-based ensemble methods beyond the standard regression forest analyzed in the preceding sections. We briefly describe how the framework applies to four prominent extensions: quantile regression forests, survival forests, honest forests, and causal forests. While this is far from an all inclusive list, showing the framework represented in these common forest procedures serve to demonstrate the generalizability of the outlined design-based framework for tree-based ensembles. 

\subsection{Quantile regression forests}
\label{subsec:qrf}

Quantile regression forests \citep{meinshausen2006} retain the standard forest partitioning mechanism but replace conditional mean estimation with estimation of the full conditional distribution of $Y \mid X = x$. The ensemble induces weights $W_i(x) = \frac{1}{B}\sum_{b=1}^B W_i(x;\theta_b)$, which define a weighted empirical distribution
\[
\hat F(y \mid x)
=
\sum_{i=1}^n W_i(x)\,\mathbf{1}\{Y_i \le y\}.
\]
The conditional $\tau$-quantile is obtained as the generalized inverse
\[
\hat Q_\tau(x) = \inf\{ y : \hat F(y \mid x) \ge \tau \}.
\]

At the tree level, the contribution of a single tree $\theta$ can be written as a functional of the local empirical distribution induced by the averaging set $A_\theta(x)$. Conditional on the realized partition, this functional is deterministic in the outcomes $\{Y_i\}_{i \in A_\theta(x)}$. The ensemble prediction is therefore an average of exchangeable tree-level random functionals, exactly as in the regression setting.

The variance identity in Theorem~\ref{thm:finite_var} does not rely on linearity of the tree-level prediction in $Y_i$, but only on exchangeability of the tree-generating mechanism. It therefore applies directly to $\hat Q_\tau(x)$. The covariance floor
\[
C_T(x) = \Cov\!\bigl(Q_{\tau,\theta}(x), Q_{\tau,\theta'}(x) \mid X\bigr)
\]
is governed by the same two mechanisms identified in Section~\ref{sec:covariance_decomposition}. Strict positivity of $C_T(x)$ follows under the same structural condition as in Theorem~\ref{thm:cov_floor}: if there exists an observation whose conditional outcome distribution has positive dispersion and which receives positive weight at $x$ with positive probability, then independent trees that align locally will induce nondegenerate variability in the quantile functional. Quantile regression forests thus satisfy the same structural assumptions as standard regression forests, with only the target functional of the weighted local data modified.

\subsection{Survival forests}
\label{subsec:survival}

Random survival forests \citep{ishwaran2008} target time-to-event outcomes $(T_i, \delta_i)$, where $T_i$ is the observed time and $\delta_i$ is a censoring indicator, with trees grown using survival-specific splitting criteria such as the log-rank statistic. Within each terminal node, the prediction at $x$ is a summary of the local survival distribution -- typically the Nelson-Aalen cumulative hazard estimator or the Kaplan-Meier survival function evaluated at a fixed time horizon.

The key structural property is that, conditional on the tree structure, the prediction at $x$ is a deterministic function of the outcomes $\{(T_i, \delta_i)\}_{i \in A_\theta(x)}$ in the averaging set. Although the within-node estimator involves the risk-set ordering and is not a simple unweighted average of responses, the tree-level prediction remains a fixed functional of the local data once the partition is realized. Theorem~\ref{thm:finite_var} requires only exchangeability of the tree-generating mechanism, which is preserved regardless of the loss function or within-node estimator. The covariance floor $C_T(x) = \Cov(T_\theta(x), T_{\theta'}(x) \mid X)$ is well-defined and strictly positive under the same conditions: whenever an observation with positive conditional outcome variance receives nonzero weight at $x$ with positive probability.

Empirical estimation via synthetic resampling as detailed in this paper requires specifying an imposed law for $(T_i, \delta_i \mid X)$ rather than $(Y_i \mid X)$. This could be accomplished through a conditional survival function estimated from the observed data, with synthetic event times and censoring indicators drawn from the fitted conditional distributions at the realized design points. The covariance form of the estimator~\eqref{eq:ct_hat} applies without modification once the synthetic outcome replicates are generated under the appropriate survival model. Construction and evaluation of this estimator is deferred to future work. 

\subsection{Honest forests}
\label{subsec:honest}

Honest forests \citep{athey2019, wager2018} partition the training sample into a structure-building subset $\mathcal S_1$ and an estimation subset $\mathcal S_2$. The tree partition $R_\theta(x)$ is determined entirely by $\mathcal S_1$, while the prediction at $x$ uses only observations from $\mathcal S_2$ that fall in the terminal region containing $x$:
\[
T_\theta(x) = \frac{1}{|\mathcal S_2 \cap A_\theta(x)|} \sum_{i \in \mathcal S_2 \cap A_\theta(x)} Y_i.
\]
The prediction retains the weighted-average form with nonnegative weights summing to one, so Theorem~\ref{thm:finite_var} applies directly.

The sample-splitting mechanism has a specific consequence for the covariance decomposition. When two trees use independent estimation subsets $\mathcal S_2^{(\theta)}$ and $\mathcal S_2^{(\theta')}$ drawn without replacement from the same pool, the probability that any observation contributes to the prediction at $x$ in both trees is reduced relative to standard forests. In the limiting case where $\mathcal S_2^{(\theta)} \cap \mathcal S_2^{(\theta')} = \emptyset$, which occurs with positive probability when the estimation fraction is small, the observation-overlap component of~\eqref{eq:lotc} is goes to zero. No shared outcome receives positive weight in both trees, eliminating reuse-induced dependence by construction.

The alignment component, however, persists. Independently generated honest trees discover similar partitions near $x$ whenever the signal structure in $\mathcal S_1$ produces stable split decisions, because the data-generating process -- not the specific structure-building observations-- determines which splits are favorable. Two trees with disjoint estimation sets but similar partitions will average different observations drawn from the same local subpopulation, producing dependence through aligned local averaging rules exactly as described in Section~\ref{subsec:partition_alignment}. The covariance floor is therefore strictly positive even under complete separation of estimation sets, provided the alignment mechanism is active.

Honest forests thus provide a natural setting in which the two mechanisms identified in this paper are empirically separable: the sample-splitting design eliminates or attenuates observation overlap while leaving alignment intact. This separation confirms that alignment, not observation reuse, is the more fundamental source of irreducible dependence in tree-based ensembles.

\subsection{Causal forests}
\label{subsec:causal}

Causal forests \citep{athey2019, wager2018} estimate heterogeneous treatment effects $\tau(x) = \mathbb E[Y(1) - Y(0) \mid X = x]$ by constructing trees that target the conditional average treatment effect within terminal nodes. In the generalized random forest framework \citep{athey2019}, the prediction at $x$ takes the form
\[
\hat\tau_\theta(x) = \sum_{i=1}^n \alpha_i(x;\theta)\,Y_i,
\]
where the weights $\alpha_i(x;\theta)$ depend on terminal-node membership, the treatment indicator $D_i$, and the local treatment proportion within the node. Unlike regression forests, these weights can be negative: treated observations in $A_\theta(x)$ receive weights of one sign while control observations receive weights of the opposite sign, reflecting the contrast structure of the estimand.

The variance identity in Theorem~\ref{thm:finite_var} is unaffected by the sign of the weights. It requires only exchangeability of the tree-generating mechanism, which is preserved in causal forests. The covariance floor $C_T(x) = \Cov(\hat\tau_\theta(x), \hat\tau_{\theta'}(x) \mid X)$ is therefore well-defined.

Strict positivity of $C_T(x)$ requires a brief additional argument. In the alignment component~\eqref{eq:calign_explicit}, each summand involves $\mathbb E[\bar W_i(x;\mathcal I_\theta)\,\bar W_i(x;\mathcal I_{\theta'}) \mid X]$, where $\bar W_i$ denotes the resampling-conditional mean weight. Because $\mathcal I_\theta$ and $\mathcal I_{\theta'}$ are independent, this expectation factors:
\[
\mathbb E\!\left[\bar W_i(x;\mathcal I_\theta)\,\bar W_i(x;\mathcal I_{\theta'}) \;\middle|\; X\right]
=
\left(\mathbb E\!\left[\bar W_i(x;\mathcal I_\theta) \;\middle|\; X\right]\right)^2
\ge 0.
\]
Every term in the alignment component is therefore nonneg\-ative regardless of the sign of individual weights, and strict positivity follows whenever there exists an observation whose expected resampling-conditional weight is nonzero and whose conditional outcome variance is positive.

For the observation-overlap component, the product $\alpha_i(x;\theta)\,\alpha_i(x;\theta')$ is nonneg\-ative for each observation $i$ in practice, because the sign of $\alpha_i$ is determined by observation $i$'s treatment status $D_i$, which is the same in both trees. Two independent trees that include observation $i$ in their terminal regions at $x$ assign it a weight whose sign depends on $D_i$ and the local treatment proportion, and since $D_i$ is fixed, the signs are consistent across trees. The overlap component is therefore also nonneg\-ative under standard implementations.

Estimation via synthetic resampling requires specifying an imposed law for $(Y_i, D_i \mid X)$. When treatment assignment is unconfounded conditional on $X$, this can be decomposed as $D_i \mid X \sim \mathrm{Bernoulli}(\hat e(X_i))$ and $Y_i \mid X, D_i \sim \hat F_{D_i}(\cdot \mid X_i)$, where $\hat e$ is an estimated propensity score and $\hat F_d$ are estimated conditional outcome distributions under each treatment arm. Synthetic replicates are generated by drawing treatment assignments and potential outcomes from the fitted conditional models at the realized design points. The covariance estimator~\eqref{eq:ct_hat} then targets the covariance floor of the causal forest under the imposed law.